\pdfoutput=1
\documentclass[twoside,11pt]{article}
\usepackage{jmlr2e}
\usepackage{natbib}

\usepackage[utf8]{inputenc} 
\usepackage[T1]{fontenc}    
\usepackage{hyperref}       
\usepackage{url}            
\usepackage{booktabs}       
\usepackage{amsfonts}       
\usepackage{nicefrac}       
\usepackage{microtype}      
\usepackage{xcolor}         
\usepackage{tcolorbox}
\usepackage{dsfont}
\usepackage{lipsum,booktabs}
\usepackage{amsmath,mathrsfs,amssymb,amsfonts,bm,enumitem}
\usepackage{rotating}
\usepackage{pdflscape}
\usepackage{tcolorbox}
\usepackage{multirow, makecell}
\usepackage{hyperref,url,dsfont,nicefrac}
\hypersetup{
    colorlinks,
    breaklinks,
    linkcolor = blue,
    citecolor = blue,
    urlcolor  = blue,
}
\allowdisplaybreaks
\usepackage{appendix}
\usepackage{multirow,makecell,tabularx,tabulary}

\usepackage{algorithmic,algorithm}

\renewcommand{\tilde}{\widetilde}
\renewcommand{\hat}{\widehat}
\def \Rcal {\psi}
\newcommand \Div[2]{\D_{\Rcal}(#1, #2)}

\def \A {\mathcal{A}}
\def \B {\mathbb{B}}
\def \B {\mathcal{B}}

\def \D {\mathcal{D}}
\def \E {\mathbb{E}}

\def \H {\mathcal{H}}
\def \I {\mathbf{I}}

\def \M {\mathcal{M}}

\def \O {\mathcal{O}}
\def \P {\mathbf{P}}
\def \Pcal {\mathcal{P}}

\def \R {\mathbb{R}}

\def \Vb {\mathbf{V}}
\def \W {\mathcal{W}}

\def \X {\mathcal{X}}
\def \Xb {\mathbf{X}}
\def \Y {\mathcal{Y}}

\def \a {\mathbf{a}}

\def \g {\mathbf{g}}

\def \p {\mathbf{p}}

\def \u {\mathbf{u}}
\def \v {\mathbf{v}}
\def \w {\mathbf{w}}
\def \x {\mathbf{x}}
\def \y {\mathbf{y}}

\def \ft{\tilde{f}}

\def \Ecal {\mathcal{E}}
\def \Ot {\tilde{\O}}

\def \ellb {\boldsymbol{\ell}}

\def \base {\mathtt{base}\mbox{-}\mathtt{regret}}
\def \meta {\mathtt{meta}\mbox{-}\mathtt{regret}}

\def \tr {\operatorname{tr}}

\newcommand \Fnorm[1] {\norm{#1}_{\operatorname{F}}}

\def \gradg {\nabla g}
\def \gradf {\nabla f}

\usepackage{mathtools}
\let\norm\undefined 
\DeclarePairedDelimiter\norm{\lVert}{\rVert}
\DeclarePairedDelimiter\abs{\lvert}{\rvert}
\newcommand\inner[2]{\langle #1, #2 \rangle}

\DeclareMathOperator*{\argmin}{arg\,min}

\let\proof\relax

\usepackage{amsthm}

\newtheorem{myThm}{Theorem}

\newtheorem{myLemma}{Lemma}

\theoremstyle{definition}
\newtheorem{myAssumption}{Assumption}

\newtheorem{myRemark}{Remark}

\usepackage{graphicx,subfigure,color}

\definecolor{wine_red}{RGB}{228,48,64}
\definecolor{DSgray}{cmyk}{0,1,0,0}

\usepackage{prettyref}
\newcommand{\pref}[1]{\prettyref{#1}}

\newcommand{\savehyperref}[2]{\texorpdfstring{\hyperref[#1]{#2}}{#2}}
\newrefformat{eq}{\savehyperref{#1}{Eq.~\textup{(\ref*{#1})}}}
\newrefformat{eqn}{\savehyperref{#1}{Eq.~\textup{(\ref*{#1})}}}
\newrefformat{lem}{\savehyperref{#1}{Lemma~\ref*{#1}}}
\newrefformat{lemma}{\savehyperref{#1}{Lemma~\ref*{#1}}}
\newrefformat{def}{\savehyperref{#1}{Definition~\ref*{#1}}}
\newrefformat{line}{\savehyperref{#1}{Line~\ref*{#1}}}
\newrefformat{thm}{\savehyperref{#1}{Theorem~\ref*{#1}}}
\newrefformat{theorem}{\savehyperref{#1}{Theorem~\ref*{#1}}}
\newrefformat{corr}{\savehyperref{#1}{Corollary~\ref*{#1}}}
\newrefformat{cor}{\savehyperref{#1}{Corollary~\ref*{#1}}}
\newrefformat{col}{\savehyperref{#1}{Corollary~\ref*{#1}}}
\newrefformat{sec}{\savehyperref{#1}{Section~\ref*{#1}}}
\newrefformat{subsec}{\savehyperref{#1}{Section~\ref*{#1}}}
\newrefformat{appendix}{\savehyperref{#1}{Appendix~\ref*{#1}}}
\newrefformat{app}{\savehyperref{#1}{Appendix~\ref*{#1}}}

\newrefformat{assum}{\savehyperref{#1}{Definition~\ref*{#1}}}

\newrefformat{assumption}{\savehyperref{#1}{Definition~\ref*{#1}}}
\newrefformat{ex}{\savehyperref{#1}{Example~\ref*{#1}}}
\newrefformat{fig}{\savehyperref{#1}{Figure~\ref*{#1}}}
\newrefformat{alg}{\savehyperref{#1}{Algorithm~\ref*{#1}}}
\newrefformat{algo}{\savehyperref{#1}{Algorithm~\ref*{#1}}}
\newrefformat{rem}{\savehyperref{#1}{Remark~\ref*{#1}}}
\newrefformat{conj}{\savehyperref{#1}{Conjecture~\ref*{#1}}}
\newrefformat{prop}{\savehyperref{#1}{Proposition~\ref*{#1}}}
\newrefformat{proto}{\savehyperref{#1}{Protocol~\ref*{#1}}}
\newrefformat{prob}{\savehyperref{#1}{Problem~\ref*{#1}}}
\newrefformat{claim}{\savehyperref{#1}{Claim~\ref*{#1}}}
\newrefformat{que}{\savehyperref{#1}{Question~\ref*{#1}}}
\newrefformat{op}{\savehyperref{#1}{Open Problem~\ref*{#1}}}
\newrefformat{fn}{\savehyperref{#1}{Footnote~\ref*{#1}}}
\newrefformat{tab}{\savehyperref{#1}{Table~\ref*{#1}}}
\newrefformat{fig}{\savehyperref{#1}{Figure~\ref*{#1}}}

\def \epsilon {\varepsilon}
\def \Dreg {\textsc{D-Reg}}
\def \Areg {\textsc{A-Reg}}
\def \WAreg {\textsc{WA-Reg}}
\def \p {\boldsymbol{p}}
\newcommand{\LineComment}[1]{\hfill$\rhd\ $\text{#1}}
\newcommand*\ind[1]{\mathds{1}_{\{#1\}}}

\usepackage{lastpage}
\jmlrheading{26}{2025}{1-\pageref{LastPage}}{9/23; Revised 12/24}{8/25}{23-1188}{Peng Zhao, Yan-Feng Xie, Lijun Zhang, and Zhi-Hua Zhou}
\ShortHeadings{Efficient Methods for Non-stationary Online Learning}{Zhao, Xie, Zhang, and Zhou}

\firstpageno{1}

\begin{document}
\title{Efficient Methods for Non-stationary Online Learning}

\author{\name Peng Zhao \email zhaop@lamda.nju.edu.cn \\
   \name Yan-Feng Xie \email xieyf@lamda.nju.edu.cn \\
   \name Lijun Zhang \email zhanglj@lamda.nju.edu.cn \\
   \name Zhi-Hua Zhou \email zhouzh@lamda.nju.edu.cn \\
  \addr National Key Laboratory for Novel Software Technology, Nanjing University, China\\
   School of Artificial Intelligence, Nanjing University, China}

\editor{Alekh Agarwal}
   
\maketitle

\begin{abstract}
   Non-stationary online learning has drawn much attention in recent years. In particular, \emph{dynamic regret} and \emph{adaptive regret} are proposed as two principled performance measures for online convex optimization in non-stationary environments. To optimize them, a two-layer online ensemble is usually deployed due to the inherent uncertainty of non-stationarity, in which multiple base-learners are maintained and a meta-algorithm is employed to track the best one on the fly. However, the two-layer structure raises concerns about computational complexity --- such methods typically maintain $\O(\log T)$ base-learners simultaneously for a $T$-round online game and thus perform multiple projections onto the feasible domain per round, which becomes the computational bottleneck when the domain is complicated. In this paper, we present efficient methods for optimizing dynamic regret and adaptive regret that reduce the number of projections per round from $\O(\log T)$ to $1$. The proposed algorithms require only one gradient query and one function evaluation at each round. Our technique hinges on the reduction mechanism developed in parameter-free online learning and requires non-trivial modifications for non-stationary online methods. Furthermore, we study an even stronger measure, namely ``interval dynamic regret'', and reduce the number of projections per round from $\O(\log^2 T)$ to $1$ for minimizing it. Our reduction demonstrates broad generality and applies to two important applications: online stochastic control and online principal component analysis, resulting in methods that are both efficient and optimal. Finally, empirical studies verify our theoretical findings.
\end{abstract}

\section{Introduction}
\label{sec:intro}
Classic online learning focuses on minimizing the static regret, which evaluates the online learner's performance against the best fixed decision in hindsight~\citep{book'16:Hazan-OCO}. However, in many real-world applications, the environments are often non-stationary. In such scenarios, minimizing static regret becomes less attractive, since it would be unrealistic to assume the existence of a single decision behaved satisfactorily throughout the entire time horizon.

To address the limitation, in recent years, researchers have studied more strengthened performance measures to facilitate  online algorithms with the capability of handling non-stationarity. In particular, dynamic regret~\citep{ICML'03:zinkvich,NIPS'18:Zhang-Ader} and adaptive regret~\citep{ICML'09:Hazan-adaptive,ICML'15:Daniely-adaptive} are proposed as two principled metrics to guide the algorithm design. We focus on the online convex optimization (OCO) setting~\citep{book'16:Hazan-OCO}, which can be deemed as a game between the learner and the environments. At each round $t \in [T]$, the learner submits her decision $\x_t \in \X$ from a convex feasible domain $\X \subseteq \R^d$ and simultaneously environments choose a convex function $f_t: \X \mapsto \R$, and subsequently the learner suffers an instantaneous loss $f_t(\x_t)$. 

\subsection{Dynamic Regret and Adaptive Regret}
\label{sec:intro-measure}
Dynamic regret is proposed by~\citet{ICML'03:zinkvich} to compare the online learner's performance against a sequence of \emph{any} feasible comparators $\u_1,\ldots,\u_T \in \X$. Formally, it is defined as
\begin{equation}
  \label{eq:dynamic-regret-def}
  \Dreg_T(\u_1,\ldots,\u_T) = \sum_{t=1}^T f_t(\x_t) - \sum_{t=1}^T f_t(\u_t).
\end{equation}
Dynamic regret minimization enables the learner to track changing comparators. A favorable dynamic regret bound should scale with a certain non-stationarity measure dependent on the comparators such as the path length $P_T = \sum_{t=2}\norm{\u_t - \u_{t-1}}_2$.  Notably, the classic static regret can be treated as a special case of dynamic regret by specifying the comparators as  the best fixed decision in hindsight.

Adaptive regret is proposed by~\citet{ICML'09:Hazan-adaptive} and further strengthened by~\citet{ICML'15:Daniely-adaptive}, 
which measures the regret over \emph{any} interval $I = [r,s] \subseteq [T]$ with a length of $\tau = \abs{I}$ and hence is also referred to as the \emph{interval regret}. The specific definition is
\begin{equation}
  \label{eq:adaptive-regret-def}
  \Areg_T(\abs{I}) = \max_{[r,r+\tau-1] \subseteq [T]} \bigg\{ \sum_{t=r}^{r+\tau-1} f_t(\x_t) - \min_{\u \in \X} \sum_{t=r}^{r+\tau-1} f_t(\u) \bigg\}.
\end{equation}
Since the minimizers of different intervals can be different, adaptive regret minimization also ensures the capability of competing with changing comparators. A desired adaptive regret bound should be as close as the minimax static regret of this interval. Algorithms with adaptive regret matching static regret of this interval up to logarithmic terms in $T$ are referred to as strongly adaptive~\citep{ICML'15:Daniely-adaptive}. Moreover, it can be observed that adaptive regret can include the static regret when choosing $I = [T]$.

Substantial research has been devoted to optimizing these two measures, including algorithms for dynamic regret~\citep{ICML'03:zinkvich,ICML'13:dynamic-model,ICML'16:GyorgyS-shiftregret,NIPS'18:Zhang-Ader,NIPS'20:sword,JMLR'21:BCO,JMLR:sword++,COLT'21:baby-strong-convex,jacobsen2022parameter,ICML'22:mdp,NeurIPS'22:label_shift,NeurIPS'23:covariate_shift,NeurIPS'24:dynamicMDP,ICML'25:dynamic-regret-curved} and adaptive regret~\citep{ICML'09:Hazan-adaptive,ICML'15:Daniely-adaptive,JMLR'16:closer-adaptive-regret,AISTATS'17:coin-betting-adaptive,ICML'18:zhang-dynamic-adaptive,ICML19:Zhang-Adaptive-Smooth,ICML'20:Ashok,NIPS'21:dual-adaptive,ICML'24:small-loss-adaptive}. 
It is worth noting that the relationship between dynamic regret and adaptive regret in online convex optimization (OCO) remains largely unclear~\citep[Section 5]{IJCAI:2020:Zhang}, even though a black-box reduction from dynamic regret to adaptive regret has been established for the simpler prediction-with-expert-advice setting (i.e., online linear optimization over simplex)~\citep[Theorem 4]{COLT'15:Luo-AdaNormalHedge}. As a result, the two measures are typically developed independently. 

To achieve the best of both worlds, several studies~\citep{AISTATS'20:Zhang,ICML'20:Ashok} have proposed optimizing both measures simultaneously by optimizing an even strengthened metric called \emph{interval dynamic regret}, defined as 
\begin{equation}
  \label{eq:interval-dynamic-regret-def}
  \textsc{Interval-D-Reg}_T(\abs{I};\u_1,\ldots,\u_T) = \max_{[r,r+\tau-1] \subseteq [T]} \bigg\{ \sum_{t=r}^{r+\tau-1} f_t(\x_t) - \sum_{t=r}^{r+\tau-1} f_t(\u_t) \bigg\}.
\end{equation}
This measure requires the online algorithm to compete with a sequence of time-varying comparators over any interval.

\subsection{Two-layer Online Ensemble and Projection Complexity Issue}
\label{sec:intro-two-layer}
The fundamental challenge of optimizing the above non-stationary regret measures is the uncertainty of environmental non-stationarity. Concretely, to ensure robustness in unknown environments, dynamic regret aims to compete with \emph{any} feasible comparator sequence, while adaptive regret examines the local performance over \emph{any} intervals; with interval dynamic regret is even more ambitious. The unknown comparators or/and unknown intervals bring considerable uncertainty to online optimization. To address the issue, a two-layer online ensemble structure is usually deployed to optimize the measures~\citep{JMLR:sword++}, where a set of base-learners are maintained to handle the different possibilities of online environments and a meta-algorithm is employed to combine them all and track the unknown best one. Such a meta-base framework successfully achieves many state-of-the-art results, including the $\O(\sqrt{T(1+P_T)})$ dynamic regret~\citep{NIPS'18:Zhang-Ader} and the $\O(\sqrt{(F^\u_T+ P_T)(1+P_T)})$ small-loss dynamic regret for smooth functions~\citep{NIPS'20:sword}, where $P_T = \sum_{t=2}^T \norm{\u_t - \u_{t-1}}_2$ is the path length and $F^\u_T = \sum_{t=1}^T f_t(\u_t)$ is the cumulative loss of comparators; as well as the $\O(\sqrt{\abs{I} \log T})$ adaptive regret~\citep{AISTATS'17:coin-betting-adaptive} and the $\O(\sqrt{F_I \log F_I \log F_T})$ small-loss adaptive regret for smooth functions~\citep{ICML19:Zhang-Adaptive-Smooth} for any interval $I = [r,s] \subseteq [T]$, where $F_I = \min_{\x \in \X} \sum_{t=r}^{s} f_t(\x)$ and $F_T = \min_{\x \in \X} \sum_{t=1}^{T} f_t(\x)$. Besides, an $\O(\sqrt{\abs{I}(\log T + P_I)})$ interval dynamic regret is also achieved by a two-layer (or even three-layer) structure~\citep{AISTATS'20:Zhang}, where $P_I = \sum_{t=r}^{s} \norm{\u_t - \u_{t-1}}_2$ is the path length over the interval.

The two-layer online ensemble methods have demonstrated great effectiveness in tackling non-stationary online environments, whereas the gain is at the price of heavier computations than the methods for minimizing static regret. While it is believed that additional computations are necessary for more robustness, we are wondering whether it is possible to pay for a ``minimal'' computation overhead for adapting to the non-stationarity. To this end, we focus on the popular first-order online methods and aim to streamline unnecessary computations while retaining the same regret guarantees. Arguably, the most computationally expensive step of each round is the projection onto the convex feasible domain, namely, the projection operation $\Pi_{\X}[\y] = \argmin_{\x \in \X} \norm{\x - \y}_2$ for a convex set $\X \subseteq \R^d$. Typical two-layer non-stationary online algorithms require maintaining $N = \O(\log T)$ base-learners simultaneously to cover the possibility of unknown environments. Define the \emph{projection complexity} of online methods as the number of projections onto the feasible domain per round. Then, those non-stationary methods suffer an $\O(\log T)$ projection complexity, whereas standard online methods for static regret minimization require only $1$ projection onto the feasible domain per round such as online gradient descent~\citep{ICML'03:zinkvich}.

\subsection{Our Contributions and Techniques}
\label{sec:intro-ours}
In this paper, we design a generic mechanism to reduce the projection complexity of many existing non-stationary methods from $\O(\log T)$ to $1$ \emph{without sacrificing regret optimality}, thereby matching the projection complexity of stationary methods. Our reduction is inspired by recent advances in parameter-free online learning~\citep{COLT'18:black-box-reduction,COLT'19:Lipschitz-MetaGrad}. The idea is simple: we reduce the original problem learned in the feasible domain $\X$ to an alternative one learned in a \emph{surrogate domain} $\Y \supseteq \X$ where projection is much cheaper, e.g., simply choosing $\Y$ as a properly scaled Euclidean ball; and moreover, a carefully designed \emph{surrogate loss} is necessary for the alternative problem to preserve the regret optimality. We reveal that a necessary condition for our reduction mechanism to deploy and reduce the projection complexity is that the non-stationary online algorithm shall \emph{query the gradient once and evaluate the function value once per round}. Several existing algorithms for obtaining the worst-case (problem-independent) dynamic regret or adaptive regret already satisfy the requirements, and our reduction immediately yields efficient counterparts with the same regret guarantees but only one projection per round.
However, many non-stationary algorithms, particularly those designed for small-loss bounds, do not satisfy this condition, requiring non-trivial efforts to make them compatible. 
To this end, we develop a family of algorithms that achieve worst-case/small-loss dynamic regret and adaptive regret with with one projection, one gradient query, and one function evaluation per round.

Despite that the reduction mechanism of this paper has been studied in parameter-free online learning, applying it to non-stationary online learning requires new ideas and non-trivial modifications. Here we highlight the technical innovation. The main challenge comes from the reduction condition mentioned earlier: since the surrogate loss involves the projection operation, our reduction requires the algorithm query one gradient and evaluate one function value at each round. However, many non-stationary algorithms do not satisfy the requirement, which is to be contrasted to the parameter-free algorithms such as MetaGrad~\citep{NIPS'16:MetaGrad} or its variants~\citep{UAI'19:Maler-Wang,COLT'19:Lipschitz-MetaGrad} that naturally satisfy the condition. For example, the SACS algorithm~\citep{ICML19:Zhang-Adaptive-Smooth} enjoys the best known small-loss adaptive regret, yet the method requires $N$ gradient queries and $N+1$ function evaluations at each round, where $N = \O(\log T)$ is the number of base-learners. Thus, we have to dig into the algorithm and modify it to fit our reduction. First, we replace their meta-algorithm with Adapt-ML-Prod~\citep{COLT'14:second-order-Hedge}, an expert-tracking algorithm with a \emph{second-order} regret with excess losses to accommodate the linearized loss that is used to ensure one gradient query per round. Second, we introduce a sequence of \emph{time-varying} thresholds to adaptively determine the problem-dependent geometric covers in contrast to a fixed threshold used in their method. In particular, we register the cumulative loss of the final decisions rather than the base-learner's one to compare it with the changing thresholds, which renders the design of one function value evaluation per round and also turns out to be crucial for achieving an improved small-loss bound that can recover the best known worst-case adaptive adaptive regret (by contrast, SACS cannot obtain optimal worst-case adaptive regret). Our algorithm only requires one projection/gradient query/function evaluation at each round, substantially improving the efficiency of SACS algorithm that requires $N$ projections/gradient queries/function evaluations per round. 

In this paper, extending our preliminary conference version~\citep{NeurIPS'22:efficient}, we further study a stronger measure, \emph{interval dynamic regret}, and develop an efficient method that reduces its projection complexity from $\O(\log^2 T)$ to $1$ using our reduction mechanism. Analyzing interval dynamic regret requires techniques similar to those for adaptive regret, along with additional refinements in the geometric covering intervals and base-algorithm to handle the more complex structure. Table~\ref{table:main} summarizes our main results in comparison with existing ones, in which we focus on the small-loss bounds for convex and smooth functions.
Our reduction mechanism demonstrates great generality and can be applied to two important applications: (i) online non-stochastic control~\citep{ICML'19:online-control}, and (ii) online principal component analysis~\citep{jmlr'08:Warmuth-pca}. In both applications, the projection operation is complicated and time-consuming, which becomes even worse when leveraging these methods to handle non-stationary environments, since typically $\O(\log T)$ projections per round are required by following previous methods~\citep{ICML19:Zhang-Adaptive-Smooth, JMLR'23:memory}. We explore the structure of these two problems and propose efficient methods with $1$ projection per round to solve them under non-stationary environments, while ensuring optimal theoretical guarantees. 

\begin{table}[!t]
  \centering
  \caption{
  Summary of different \emph{small-loss} regret results for smooth functions. The third column lists the number of projections required per iteration. The fourth and fifth columns indicate the number of gradient and function value queries per round, respectively, with ``--'' denoting that function values are not required. The sixth column presents the regret bounds for different algorithms, where $F_I=\min _{\x \in \X} \sum_{t=r}^s f_t(\x)$, $F_I^\u= \sum_{t=r}^s f_t(\u_t)$, and $P_I=\sum_{t=r}^s \norm{\u_t -\u_{t-1}}_2$ are problem-dependent quantities over the interval $I \subseteq [1, T]$. For simplicity, we use $F_T$ to denote $F_{[T]}$ (similarly defined for $P_T$ and $F_T^\u$).  }
  \vspace{-2mm}
  \label{table:main}
  \renewcommand*{\arraystretch}{1.6}
  \resizebox{0.99\textwidth}{!}{
    \begin{tabular}{c|c|c|c|c|c}
      \hline

      \hline
      \textbf{Measure} & \textbf{Reference} & \textbf{\# Proj} & \textbf{\# Grad} & \textbf{\# Value} & \textbf{Regret Bound} \\
      \hline  
       
      \hline
      \multirow{2}{*}{\textbf{D-Reg}} & \citet{JMLR:sword++} & $\mathcal{O}\left(\log T\right)$ & $1$ & --  &$\O(\sqrt{(F^{\u}_T + P_T)(1+P_T)})$ \\ \cline{2-6} 
       & Theorem~\ref{thm:dynamic-regret-project-smooth} (ours) & $1$ & $1$& -- &$\O(\sqrt{(F^{\u}_T + P_T)(1+P_T)})$ \\ 
       \hline  
       
       \hline
       \multirow{2}{*}{\textbf{A-Reg}} &\citet{ICML19:Zhang-Adaptive-Smooth} & $\O(\log T)$ & $\O(\log T)$&$\O(\log T)$ &$\O(\sqrt{F_I\log F_I\log F_T})$ \\ \cline{2-6} 
       &  Theorem~\ref{thm:small-loss-adaptive} (ours) & $1$ & $1$&$1$ &$\O(\min\{\sqrt{F_I\log F_I\log F_T}, \sqrt{\abs{I}\log T} \})$ \\  \hline  
       
      \hline
       \multirow{2}{*}{\makecell{\textbf{Interval} \\ \textbf{D-Reg}}} &\citet{AISTATS'20:Zhang} & $\O(\log^2 T)$ & $\O(\log^2 T)$&$\O(\log^2 T)$ &$\O(\sqrt{\abs{I}(\log T + P_I)})$ \\ \cline{2-6} 
       & Theorem~\ref{thm:small-loss-interval-dynamic-regret} (ours) & $1$ & $1$&$1$ &$\O(\sqrt{(F^{\u}_I+P_I)(\log (F_I^\u+P_I)\cdot \log(F_T^\u+P_T) + P_I)})$ \\ \hline 
       
       \hline
      \end{tabular}
  }
\end{table}

\subsection{Assumptions}
\label{sec:assumptions}
In this part, we list several standard assumptions used in OCO~\citep{book'12:Shai-OCO,book'16:Hazan-OCO}. Notably, not all these assumptions are always required. We will explicitly state the requirements in the theorem.
\begin{myAssumption}[bounded gradient]
\label{assumption:bounded-gradient}
The norm of the gradients of online functions over the domain $\X$ is bounded by $G$, i.e., $\norm{\nabla f_t(\x)}_2 \leq G$, for all $\x \in \X$ and $t \in [T]$.
\end{myAssumption}

\begin{myAssumption}[bounded domain]
\label{assumption:bounded-domain}
The domain $\X \subseteq \R^d$ contains the origin $\mathbf{0}$, and the diameter of the domain $\X$ is at most $D$, i.e., $\norm{\x -\x'}_2 \leq D$ for any $\x, \x' \in \X$.
\end{myAssumption}

\begin{myAssumption}[non-negativity and smoothness]
\label{assumption:smoothness}
All the online functions are non-negative and $L$-smooth, i.e., for $t \in [T]$, the online function $f_t:\R^d \mapsto \R$ satisfies $\norm{\nabla f_t(\x)-\nabla f_t(\x')}_2 \leq L \norm{\x-\x'}_2$ for any $\x, \x' \in \R^d$.
\end{myAssumption}
Note that in Assumption~\ref{assumption:smoothness} we assume that the online functions are smooth over the \emph{entire space} $\R^d$. This is primarily for technical reasons, as it ensures the self-bounding property (described in Lemma~\ref{lemma:self-bounded}) holds when deriving small-loss regret bounds~\citep{NIPS'10:smooth, aistats'12:exp-concave-smooth, ICML19:Zhang-Adaptive-Smooth, NIPS'20:sword}. The constraint on the feasible domain remains valid, as it can be understood by noting that, while the online functions may be defined over $\R^d$, the decisions are only feasible within a subset $\X \subseteq \R^d$, which is also the case in our two applications as discussed in Section~\ref{sec:application}.

In fact, it is possible to relax Assumption~\ref{assumption:smoothness} by requiring it to hold only over an extended domain $\X^\prime$, which is slightly larger than the original domain $\X$. Specifically, this extended domain is defined as $\X^\prime = \{\x + \y \mid \x \in \X, \y \in \R^d, \norm{\y}_2 \leq G/L\}$; in other words, $\X^\prime$ is the Minkowski sum of the original domain $\X$ and a small Euclidean ball.

\subsection{Paper Outline} 
The remainder of the paper is structured as follows. In Section~\ref{sec:reduction}, we delineate the reduction mechanism and illustrate its application to dynamic regret minimization. Section~\ref{sec:adaptive-regret} offers efficient methods for optimizing adaptive regret. Subsequently, Section~\ref{sec:interval-dynamic-regret} gives the results for optimizing an even stronger performance measure -- interval dynamic regret. In Section~\ref{sec:application} we provide two applications of our proposed reduction mechanism for non-stationary online learning. Experimental validations are reported in Section~\ref{sec:experiment}. Finally, we conclude the paper and make several discussions in Section~\ref{sec:conclusion}. All the proofs and omitted details for algorithms are deferred to the appendices.
\section{The Reduction Mechanism and Dynamic Regret Minimization}
\label{sec:reduction}
We start from the dynamic regret minimization. First, we briefly review existing methods in Section~\ref{sec:review-dynamic-regret}, and then present our reduction mechanism and illustrate how to apply it to reducing the projection complexity of dynamic regret methods in Section~\ref{sec:reduction-apply-dynamic}.

\subsection{A Brief Review of Dynamic Regret Minimization}
\label{sec:review-dynamic-regret}
\citet{NIPS'18:Zhang-Ader} propose a two-layer online algorithm called Ader with an $\O(\sqrt{T(1+P_T)})$ dynamic regret, which is proved to be minimax optimal for convex functions. Ader maintains a group of base-learners, each performing online gradient descent (OGD)~\citep{ICML'03:zinkvich} with a customized step size specified by the pool $\H = \{\eta_1,\ldots,\eta_N\}$, and then uses a meta-algorithm to combine them all. Denoted by $\B_1,\ldots,\B_N$ the $N$ base-learners. For each $i\in [N]$, the base-learner $\B_i$ updates by 
\begin{equation}
  \label{eq:Ader-base}
  \x_{t+1,i} = \Pi_{\X}[\x_{t,i} - \eta_i \nabla f_t(\x_t)],
\end{equation}
where $\eta_i \in \H$ is the associated step size and $\Pi_{\X}[\cdot]$ denotes the projection onto the feasible domain $\X$ with $\Pi_{\X}[\y] = \argmin_{\x\in \X}\norm{\y - \x}_2$. Notably, all the base-learners share the same gradient $\nabla f_t(\x_t)$ rather than using their individual one $\nabla f_t(\x_{t,i})$. This is because Ader optimizes the linearized loss $\ell_t(\x) = \inner{\nabla f_t(\x_t)}{\x}$, which enjoys the benign property of $\nabla \ell_t(\x_{t,i}) = \nabla f_t(\x_t)$ for all $i \in [N]$.\footnote{We refer to the improved approach presented in Section 3.4 of~\citet{NIPS'18:Zhang-Ader}.}

Furthermore, the meta-algorithm evaluates each base-learner by the linearized loss $\ell_t(\x_{t,i}) = \inner{\nabla f_t(\x_t)}{\x_{t,i}}$ and updates the weight vector $\p_{t+1} \in \Delta_N$ by the Hedge algorithm~\citep{JCSS'97:boosting}, namely,
\begin{equation}
  \label{eq:Ader-meta}
  p_{t+1,i} = \frac{p_{t,i} \exp(-\epsilon \inner{\nabla f_t(\x_t)}{\x_{t,i}})}{\sum_{j=1}^N p_{t,j} \exp(-\epsilon \inner{\nabla f_t(\x_t)}{\x_{t,j}})},~~ \forall i \in [N],
\end{equation}
where $\epsilon > 0$ is meta-algorithm's learning rate. The final prediction is obtained by $\x_{t+1} = \sum_{i=1}^N p_{t+1,i} \x_{t+1,i}$. The learner submits  $\x_{t+1}$ and receives the loss $f_{t+1}(\x_{t+1})$ and the gradient $\nabla f_{t+1}(\x_{t+1})$ as the feedback of this round. With a suitable step size pool $\H$ and learning rate $\epsilon = \Theta(\sqrt{(\ln N)/T})$ with $N = \O(\log T)$, Ader enjoys an $\O(\sqrt{T(1+P_T)})$ dynamic regret using $1$ gradient query per round~\citep[Theorem 4]{NIPS'18:Zhang-Ader}.

For smooth functions,~\citet{JMLR:sword++} demonstrate that an $\O(\sqrt{(F^{\u}_T + P_T)(1+P_T)})$ small-loss dynamic regret can be achieved where $F^{\u}_T = \sum_{t=1}^T f_t(\u_t)$ is the cumulative loss of the dynamic comparators. This bound safeguards the minimax rate in the worst case, while it can be much smaller than $\O(\sqrt{T(1+P_T)})$ bound in the benign environments. 

\subsection{The Reduction Mechanism for Reducing Projection Complexity}
\label{sec:reduction-apply-dynamic}
As demonstrated in the update~\eqref{eq:Ader-base}, all the base-learners require projecting the intermediate solution onto the domain $\X$ to ensure feasibility. As a result, $\O(\log T)$ projections are required at each round, which is generally time-consuming particularly when the domain $\X$ is complicated. To address so, we present a generic reduction mechanism for reducing the projection complexity and apply it to dynamic regret methods. Our reduction builds upon the seminal work~\citep{COLT'18:black-box-reduction} and a further refined result~\citep{ICML'20:Ashok}, who propose a black-box reduction from constrained online learning to the unconstrained setting (or another constrained problem with a larger domain).

\paragraph{Reduction mechanism.} Given an algorithm for non-stationary online learning \textsf{Algo} whose projection complexity is $\O(\log T)$, our reduction mechanism builds on it to yield an algorithm \mbox{\textsf{Efficient-Algo}} with $1$ projection onto $\X$ per round and retaining the same order of regret. The central idea is to replace expensive projections onto the original domain $\X$ with other much cheaper projections. To this end, we introduce a \emph{surrogate domain} $\Y$ defined as the minimum Euclidean ball containing the feasible domain $\X$, i.e., $\Y = \{\x \mid \norm{\x}_2 \leq D\} \supseteq \X$. Then, the reduced algorithm \textsf{Algo} works on $\Y$ whose projection can be realized by a simple rescaling. More importantly, to avoid regret degeneration, it is necessary to carefully construct the surrogate loss $g_t: \Y \mapsto \R$ as
\begin{equation}
  \label{eq:surrogate-loss}
  g_t(\y) = \inner{\nabla f_t(\x_t)}{\y} - \ind{\inner{\nabla f_t(\x_t)}{\v_t} < 0} \cdot \inner{\nabla f_t(\x_t)}{\v_t} \cdot S_{\X}(\y),
\end{equation}
where $S_\X(\y) = \inf_{\x \in \X} \norm{\y - \x}_2$ is the distance function to $\X$ and $\v_t = (\y_t - \x_t)/\norm{\y_t - \x_t}_2$ is the vector indicating the projection direction. 

The main protocol of our reduction is presented as follows. The input includes original functions $\{f_t\}_{t=1}^T$, the feasible domain $\X$, and the reduced algorithm~\textsf{Algo}. 

\begin{algorithmic}[1]
  \FOR{$t = 1, \ldots, T$}
    \STATE receive the gradient information $\nabla f_t(\x_t)$;
    \STATE construct the surrogate loss $g_t: \Y \mapsto \R$ according to~\pref{eq:surrogate-loss};
    \STATE obtain the intermediate prediction $\y_{t+1}  = \textsf{Algo}(g_t(\cdot),\y_t,\Y)$;
    \STATE submit the final prediction $\x_{t+1} = \Pi_{\X}[\y_{t+1}]$;
  \ENDFOR
\end{algorithmic}

Our reduction enjoys the regret safeness due to the benign properties of surrogate loss.
\begin{myThm}[{Theorem 2 of~\citet{ICML'20:Ashok}}]
\label{thm:surrogate-loss}
The surrogate loss $g_t: \Y \mapsto \R$ defined in~\eqref{eq:surrogate-loss} is convex. Moreover, we have $\norm{\nabla g_t(\y_t)}_2 \leq \norm{\nabla f_t(\x_t)}_2$ and for any $\u_t \in \X$
\begin{equation}
  \label{eq:surrogate-upper-bound}
  \inner{\nabla f_t(\x_t)}{\x_t - \u_t} \leq g_t(\y_t) - g_t(\u_t) \leq \inner{\nabla g_t(\y_t)}{\y_t - \u_t}.
\end{equation}
\end{myThm}
The theorem shows the convexity of the surrogate loss $g_t(\y)$ and we thus have $f_t(\x_t) - f_t(\u_t) \leq \inner{\nabla g_t(\y_t)}{\y_t - \u_t}$, which implies that it suffices to optimize the linearized loss $\ell_t(\y) = \inner{\nabla g_t(\y_t)}{\y}$. Moreover, the property $\norm{\nabla g_t(\y_t)}_2 \leq \norm{\nabla f_t(\x_t)}_2$ established in Theorem~\ref{thm:surrogate-loss}, is crucial for achieving the small-loss bound, as will be demonstrated later in Theorem~\ref{thm:dynamic-regret-project-smooth}. Furthermore, we have the following lemma that specifies the gradient calculation.
\begin{myLemma}
\label{lemma:gradient-compute}
For any $\y \in \Y$, $\nabla g_t(\y) = \nabla f_t(\x_t)$ when $\inner{\nabla f_t(\x_t)}{\v_t} \geq 0$; and $\nabla g_t(\y) = \nabla f_t(\x_t) - \inner{\nabla f_t(\x_t)}{\v_t} \cdot (\y - \Pi_{\X}[\y])/\norm{\y - \Pi_{\X}[\y]}_2$ when $\inner{\nabla f_t(\x_t)}{\v_t} < 0$. Here $\v_t = (\y_t - \x_t)/\norm{\y_t - \x_t}_2$. In particular, $\nabla g_t(\y_t) =  \nabla f_t(\x_t) - \inner{\nabla f_t(\x_t)}{\v_t} \cdot \v_t$ when $\inner{\nabla f_t(\x_t)}{\v_t} < 0$.
\end{myLemma}

\paragraph{Reduction requirements.} An important necessary condition for the reduction is to require the reduced algorithm satisfying \emph{one gradient query} and \emph{one function evaluation} at each round. Indeed, the reduction essentially updates according to the surrogate loss $\{g_t\}_{t=1}^T$. Note that the definition of surrogate loss involves the distance function $S_\X(\y)$, see~\pref{eq:surrogate-loss}. Thus, each evaluation of $g_t(\y)$ leads to one projection onto $\X$ due to the calculation of $S_\X(\y)$. Similarly, each gradient query of $\nabla g_t(\y)$ also contributes to one projection, see Lemma~\ref{lemma:gradient-compute} for details. To summarize, we can use the reduction to ensure a $1$ projection complexity, only when the reduced algorithm satisfies the requirements of one gradient query and one function evaluation per round. Below, we demonstrate the usage of our reduction mechanism for two methods of dynamic regret minimization that satisfy the conditions, including the worst-case method~\citep{NIPS'18:Zhang-Ader} and the small-loss method~\citep{JMLR:sword++}.

\begin{algorithm}[!t]
\caption{Efficient Algorithm for Minimizing Dynamic Regret}
\label{alg:efficient-dynamic}
\begin{algorithmic}[1]
\REQUIRE step size pool $\H = \{\eta_1,\ldots,\eta_N\}$, learning rate of meta-algorithm $\epsilon_t$ (or simply a fixed one $\epsilon_t = \epsilon$).
\STATE{Initialization: let $\x_1$ and $\{\y_{1,i}\}_{i=1}^N$ be any point in $\X$; $\forall i\in [N], p_{1,i} = 1/N$.}
\FOR{$t=1$ {\bfseries to} $T$}
  \STATE Receive the gradient information $\nabla f_t(\x_t)$.
  \STATE Construct the surrogate loss $g_t: \Y \mapsto \R$ according to~\pref{eq:surrogate-loss}. \label{line:surrogate-loss-dynamic}
  \STATE Compute the gradient $\gradg_t(\y_t)$ according to Lemma~\ref{lemma:gradient-compute}.\label{line:surrogate-gradient-dynamic}
  \STATE For each $i \in [N]$, the base-learner $\B_i$ produces the local decision by 
    \begin{equation*}
      \label{eq:base-update-specific}
      \hat{\y}_{t+1,i} = \y_{t,i} - \eta_i \nabla g_t(\y_t),~~ \y_{t+1,i} = \hat{\y}_{t+1,i} \Big(\ind{\norm{\hat{\y}_{t+1,i}}_2 \leq D} + \frac{D}{\norm{\hat{\y}_{t+1,i}}_2}\cdot \ind{\norm{\hat{\y}_{t+1,i}}_2 \geq D}\Big).
    \end{equation*} \label{line:base-update}
  \STATE Meta-algorithm updates weight by $p_{t+1,i} \propto \exp(-\epsilon_{t+1} \sum_{s=1}^t \inner{\nabla g_s(\y_s)}{\y_{s,i}})$, $i \in [N]$. \label{line:meta-update-dynamic}
  \STATE Compute $\y_{t+1} = \sum_{i=1}^N p_{t+1,i} \y_{t+1,i}$. \label{line:combine-y-dynamic}
  \STATE Submit $\x_{t+1} = \Pi_{\X}[\y_{t+1}]$. \label{line:project-y-dynamic} \LineComment{the only projection onto feasible domain $\X$ per round} \label{line:project-back-dynamic}
\ENDFOR
\end{algorithmic}
\end{algorithm}

\paragraph{Application to dynamic regret minimization.} Algorithm~\ref{alg:efficient-dynamic} summarizes the main procedures of our efficient methods for optimizing dynamic regret, which is an instance of the reduction mechanism by picking \textsf{Algo} as Ader~\citep{NIPS'18:Zhang-Ader}. More specifically, {\color{blue}Lines}~\ref{line:base-update} --~\ref{line:combine-y-dynamic} are essentially performing Ader algorithm using the surrogate loss $\{g_t\}_{t=1}^T$ over the surrogate domain $\Y$. Note that the base update in~\pref{line:base-update} is essentially performing OGD with projection onto $\Y$, a scaled Euclidean ball, and thus the projection admits a simple closed form. The overall algorithm requires projecting onto $\X$ only once per round, see~\pref{line:project-back-dynamic}. Our method provably retains the same dynamic regret. 
\begin{myThm}
\label{thm:dynamic-regret-project}
Set the step size pool as $\H = \big\{ \eta_i = 2^{i-1} (D/G)\sqrt{5/(2T)} \mid i \in [N] \big\}$ with $N = \lceil 2^{-1}\log_2(1 + 2T/5) \rceil + 1$ and the learning rate as $\epsilon = \sqrt{(\ln N) / (1 + G^2D^2 T)}$. Under Assumptions~\ref{assumption:bounded-gradient} and~\ref{assumption:bounded-domain}, our algorithm requires one projection onto $\X$ per round and enjoys 
\begin{equation}
  \label{eq:dynamic-mimax}
  \sum_{t=1}^T f_t(\x_t) - \sum_{t=1}^T f_t(\u_t) \leq \O\left(\sqrt{T (1+P_T)}\right).
\end{equation}
\end{myThm}

For smooth and non-negative functions, the Sword++ algorithm~\citep{JMLR:sword++} achieves an $\O(\sqrt{(F^{\u}_T + P_T)(1+P_T)})$ small-loss dynamic regret, which requires one gradient and one function value per iteration.\footnote{The Sword++ algorithm is primarily designed to attain gradient-variation dynamic regret, incorporating advanced components like a correction term and optimism into its algorithmic design~\citep{JMLR:sword++}. Nonetheless, it can be verified that when seeking a small-loss bound only, the algorithm's complexity can be reduced by omitting both the correction term and optimism.} However, notice that the surrogate loss $g_t(\cdot)$ in~\pref{eq:surrogate-loss} is neither smooth nor non-negative, which hinders the application of our reduction to their method. Fortunately, owing to the benign property of $\norm{\nabla g_t(\y_t)}_2 \leq \norm{\nabla f_t(\x_t)}_2$ (see Theorem~\ref{thm:surrogate-loss}), we can still deploy the reduction via an improved analysis and obtain a projection-efficient algorithm with the same small-loss bound. 
\begin{myThm}
\label{thm:dynamic-regret-project-smooth}
Set the step size pool as $\H = \big\{\eta_i = 2^{i-1} \sqrt{5D^2 / (1+8LGDT)} \mid i \in [N]\big\}$ with $N = \lceil 2^{-1} \log_2((5D^2 + 2D^2T)(1+8LGDT)/(5D^2) ) \rceil + 1$ and the learning rate of the meta-algorithm as $\epsilon_t =  \sqrt{(\ln N) /(1 + D^2 \sum_{s=1}^{t-1} \|\nabla g_s(\y_s)\|_2^2)}$. Under Assumptions~\ref{assumption:bounded-gradient},~\ref{assumption:bounded-domain}, and~\ref{assumption:smoothness}, our algorithm requires one projection onto $\X$ per round and enjoys the following dynamic regret:
\begin{equation*}
  \sum_{t=1}^T f_t(\x_t) -  \sum_{t=1}^T f_t(\u_t) \leq \O\left(\sqrt{(F^{\u}_T + P_T)(1+P_T)}\right),
\end{equation*}
where $F^{\u}_T = \sum_{t=1}^T f_t(\u_t)$ is the cumulative loss of the dynamic comparators.
\end{myThm}
\section{Adaptive Regret Minimization}
\label{sec:adaptive-regret}
In this section, we present our efficient methods to minimize adaptive regret. First, we briefly review existing methods in Section~\ref{sec:review-adaptive-regret}, and then present our efficient methods to reduce the projection complexity of adaptive regret methods in Section~\ref{sec:efficient-adaptive-regret}.

\subsection{A Brief Review of Adaptive Regret Minimization}
\label{sec:review-adaptive-regret}
Adaptive regret minimization ensures that the online learner is competitive with a fixed decision across every contiguous interval $I \subseteq [T]$. Typically, an online algorithm to achieve this consists of three components:
\begin{enumerate}[leftmargin=1cm]
	\setlength\itemsep{0em}
	\item[(i)] base-algorithm: an online algorithm attaining low (static) regret in a given interval;
	\item[(ii)] scheduling: a series of intervals that can cover the entire time horizon $[T]$, which might overlap. Each interval is associated with a base-learner whose goal is to minimize static regret over the duration of that interval (from its start to end);
	\item[(iii)] meta-algorithm: a combining algorithm that can track the best base-learner on the fly.
\end{enumerate}
By dividing the entire algorithm into these three main components, it becomes more convenient to compare various algorithms and highlight the effectiveness of individual components.

For the worst-case bound, the best known result is the $\O(\sqrt{\abs{I} \log T})$ adaptive regret achieved by the CBCE algorithm~\citep{AISTATS'17:coin-betting-adaptive}. We omit its details but mention that CBCE requires multiple gradients at each round. \citet{IJCAI:2018:Wang} improve CBCE by employing the linearized loss for updating both meta-algorithm and base-algorithm. This revision allows it to require only one gradient per iteration while maintaining the same adaptive regret. Furthermore, this improved CBCE algorithm evaluates the function value only once per iteration. As a result, our reduction can be directly applied, yielding a projection-efficient variant with the same adaptive regret.

Now, we focus on the more challenging case of small-loss adaptive regret. The best known result is the $\O(\sqrt{F_I \log F_I \log F_T})$ adaptive regret for any interval $I = [r,s] \subseteq [T]$ obtained by the SACS algorithm~\citep{ICML19:Zhang-Adaptive-Smooth}, where $F_I = \min_{\x \in \X} \sum_{t=r}^{s} f_t(\x)$ and $F_T = \min_{\x \in \X} \sum_{t=1}^{T} f_t(\x)$. However, SACS does not satisfy our reduction requirements. This is because it requires $N$ gradient queries (specifically, $\nabla f_t(\x_{t,i})$ for $i \in [N]$) and $N+1$ function evaluations (specifically, $f_t(\x_{t,i})$ for $i \in [N]$, and $f_t(\x_t)$) at round $t \in [T]$. Here, $N$ denotes the number of active base-learners, and $\x_{t,i}$ denotes the local decision returned by the $i$-th base-learner. Therefore, we have to modify the algorithm to fit our purpose.

To this end, we need to review the construction of the SACS algorithm.  First, SACS uses the scale-free online gradient descent (SOGD)~\citep{TCS'18:SOGD} as its base-algorithm, which ensures a small-loss regret in a given interval. Second, SACS employs AdaNormalHedge~\citep{COLT'15:Luo-AdaNormalHedge} as the meta-algorithm, which supports the sleeping expert setup and also the small-loss regret. Finally, SACS introduces a novel scheduling strategy called the \emph{problem-dependent geometric covering intervals}. This ensures that the number of maintained base-learners also depends on small-loss quantities. Owing to these designs, SACS can achieve a fully problem-dependent adaptive regret of order $\O(\sqrt{F_I \log F_I \log F_T})$, which scales according to the cumulative loss of comparators. 

However, there are also some pitfalls. SACS also suffers from an $\O(\log T)$ projection complexity in the worst case due to its two-layer structure. Further, it can be observed that SACS only attains an $\O(\sqrt{\abs{I} \log {\abs{I}} \log T})$ bound in the worst case, which exhibits an $\O(\sqrt{\log \abs{I}})$ gap compared with the best known result of $\O(\sqrt{\abs{I} \log T})$~\citep{AISTATS'17:coin-betting-adaptive}. 

In the next subsection, we will present an efficient algorithm for small-loss adaptive regret minimization, which resolves the above two issues simultaneously.

\subsection{Efficient Algorithms for Adaptive Regret}
\label{sec:efficient-adaptive-regret}
Since SACS involves multiple gradient and function queries in all its three components, we need to make modifications to achieve an algorithm that attains the same small-loss adaptive regret while demanding merely one gradient query and function evaluation per iteration. Once equipped with such an algorithm, we can deploy our reduction scheme to obtain an efficient method with $1$ projection complexity.

The overall procedures of our proposed algorithm are summarized in Algorithm~\ref{alg:problem-dependent-adaptive}. Below, we will present the details of the three components. In particular, we will elucidate the scheduling design, i.e., the construction of covering intervals, which is paramount in achieving a fully problem-dependent bound (even improving upon the previously best known result of~\citet{ICML19:Zhang-Adaptive-Smooth}) and reducing the number of gradient and function evaluations needed. 

By the reduction mechanism, it is noticeable that we only need to consider the input online functions as surrogate loss $\{g_t\}_{t=1}^T$, where $g_t:\Y \mapsto \R$ is defined in~\pref{eq:surrogate-loss}.

\paragraph{Scheduling.} Adaptive regret examines every contiguous interval $I \subseteq [T]$, which demands a rapid adaptation to potential environment changes. A natural way to construct the scheduling is to initiate a base-learner at each round and enable her to make predictions till the end of the whole time horizon~\citep{journal'07:Hazan-adaptive}. While this approach can effectively maintain many base-learners to handle the non-stationarity, it is computationally expensive due to maintaining $\O(t)$ base-learners at round $t$. For enhanced efficiency, alternative scheduling is proposed to set the length of each learner's active time in a geometric manner~\citep{journal'07:Hazan-adaptive,ICML'15:Daniely-adaptive}. To elucidate the design, we introduce two concepts: the \emph{unit interval} and the \emph{marker}. The unit intervals partition the time horizon $[T]$, of which the adaptive algorithm chooses geometric many to construct the active time as illustrated later. The markers denote the starting and ending time stamps of the unit intervals. Formally, the $i$-th unit interval is represented by $[s_i, s_{i+1} - 1]$ with the time stamps ${s_1, \dots, s_M}$ determining the intervals (referred to as markers). Notice that, these unit intervals are disjoint and consecutive. The second and third rows of Figure~\ref{fig:gc} provide an illustrative example of unit intervals and markers.
\begin{figure}
	\includegraphics[clip, trim=3cm 13.8cm 1cm 12.2cm,width=\textwidth]{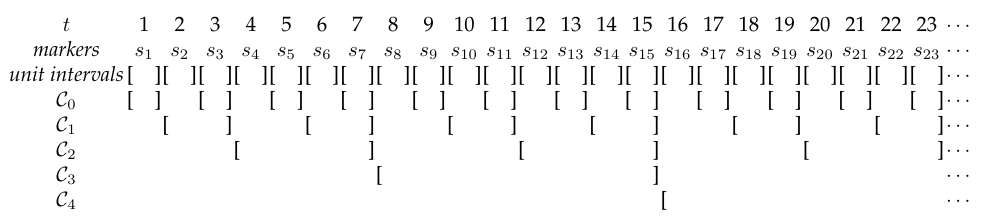}
	\caption{Geometric Covering Intervals~\citep[Figure 2]{ICML19:Zhang-Adaptive-Smooth}.}
	\label{fig:gc}
\end{figure}
Based on the two concepts, we can then illustrate how to generate \emph{geometric covering intervals}, also referred to as the scheduling above. The distinction between unit intervals and geometric intervals merits emphasis. The geometric covering intervals are the active time of each base-learner, and our adaptive algorithm manages the prediction period of the base-learners according to these intervals. The unit intervals partition the time horizon, and the covering intervals are generated based on these unit intervals by selecting a geometric number of them. The adaptive algorithm initializes a base-learner once at the beginning time of each unit interval, say at time stamp $s_m$, and determines the active time based on the indexes of the markers, 
\begin{align}
	\label{eq:geometric-generation}
	[s_{i \cdot 2^k}, s_{(i+1) \cdot 2^k} - 1],
\end{align}
where $m = i\cdot 2^k, k \in \{0\}\cup \mathbb{N}$ and $k$ is the largest number made the factorization. The setting indicates that the base-learner initialized at time $s_m$ remains active across $2^k - 1$ unit intervals, leading to an exponentially number of unit intervals. This scheduling diversifies multiple base-learners to capture the non-stationarity across different time durations, and meanwhile ensures that at most $\O(\log T)$ base-learners are maintained per round.

Conventionally, previous studies set $s_t = t$~\citep{journal'07:Hazan-adaptive,ICML'15:Daniely-adaptive}, where the increasing of markers coincidences with one of the time stamps. This leads to the standard geometric covering intervals $\mathcal{C}$ as shown in Figure~\ref{fig:gc}, formally defined below,
\begin{align}
	\label{eq:standard-gc}
	\mathcal{C}= \bigcup\nolimits_{k \in \mathbb{N} \cup \{0\}} \mathcal{C}_k, \text{ where }~\mathcal{C}_k = \left\{[{i \cdot 2^k}, {(i+1) \cdot 2^k} -1] \mid i\text{ is odd} \right\} \text{ for } k \in \mathbb{N} \cup \{0\},
\end{align}
which are derived from~\pref{eq:geometric-generation} by registering the marker at each round. However, this manner of initializing base-learners is \emph{problem-independent}, which introduces a $\sqrt{\log T}$ factor to the adaptive regret bound of order $\O(\sqrt{\abs{I}\log T})$, given that the meta-algorithm requires combining $\O(T)$ (sleeping-)experts in total. To address the issue, \citet{ICML19:Zhang-Adaptive-Smooth} propose the \emph{problem-dependent} scheduling aiming to achieve a small-loss adaptive regret. In this paper, we further refine their design in order to reduce the gradient and function value queries at each round, which also helps improve their small-loss bound slightly.

Specifically, we aim to develop a fully data-dependent adaptive regret bound of order $\O(\sqrt{F_I \log F_I \log F_T})$ with $1$ projection complexity. This bound replaces the $\O(\sqrt{\log T})$ factor by $\O(\sqrt{\log F_T \log F_I})$, which under specific beneficial conditions is considerably smaller. We postpone the discussion of the additional factor $\O(\sqrt{\log F_I})$ in Remark~\ref{remark:match-adaptive}. To achieve this result, the SACS algorithm~\citep{ICML19:Zhang-Adaptive-Smooth} modifies the generation mechanism of the markers, which \emph{registers a marker only when the cumulative loss exceeds a pre-defined fixed threshold}, instead of setting $s_t = t$. As a result, the number of active base-learners relates to the small-loss quantity, leading the overall algorithm to achieve a fully problem-dependent adaptive regret. The underlying rationale for this design stems from the potential inefficiency of the prior scheduling mechanism (which sets $s_t =t$ and initializes a base-learner at each round), since it is unnecessary to initialize a fresh new base-learner when the environment is relatively stable, or more precisely, when the cumulative loss is not large enough. Figure~\ref{fig:pgc} provides an illustrative example of problem-dependent geometric covering intervals.
\begin{figure}
	\includegraphics[clip, trim=3cm 13.8cm 1cm 12.2cm,width=\textwidth]{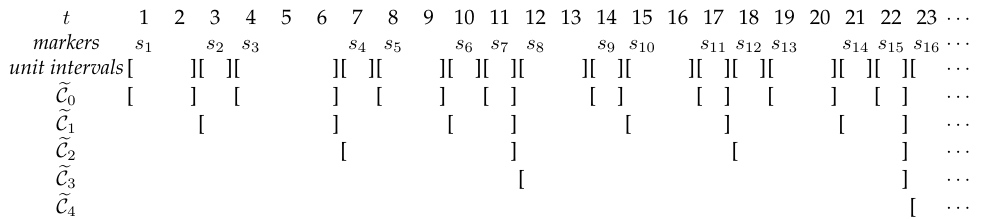}
	\caption{Problem-dependent Geometric Covering Intervals~\citep[Figure 4]{ICML19:Zhang-Adaptive-Smooth}.}
	\label{fig:pgc}
\end{figure}
Therefore, we are in the position to determine the threshold. SACS sets the threshold by monitoring the cumulative loss of the latest base-learner $f_t(\x_{t,i^\dagger})$ with $i^\dagger$ being the latest base-learner's index, but this will introduce an additional function evaluation in addition to $f_t(\x_t)$ at each round. To avoid the limitation, we design the following two important improvements:
\begin{itemize}
	\item we register markers and start a new base-learner according to the cumulative loss of \emph{final decisions}, i.e., $\{f_t(\x_t)\}_{t=1}^T$, bypassing the additional function evaluation;
	\item we introduce a sequence of \emph{time-varying thresholds} with a careful design, instead of using a fixed threshold over the time horizon.
\end{itemize}
This configuration of covering intervals realizes the condition of one function evaluation per round. Additionally, the new design of the thresholds mechanism is crucial to ensure that the small-loss bound can simultaneously recover the best known worst-case guarantee, which cannot be achieved by prior best small-loss adaptive regret bound~\citep{ICML19:Zhang-Adaptive-Smooth}. More discussions are presented in Remark~\ref{remark:match-adaptive}. 

We are now ready to introduce our efficient  algorithm for optimizing the small-loss adaptive regret. Let $C_1,C_2,C_3,\ldots$ denote the sequence of thresholds, determined by a threshold generating function $\mathcal{G}(\cdot): \mathbb{N} \mapsto \mathbb{R}_+$, which we will specify later. Our problem-dependent geometric covering intervals are defined as follows. We initialize the setting by $s_1 = 1$. We set $s_2$ as the round when the cumulative loss of the overall algorithm (namely, $\sum_{s=1}^{t} f_s(\x_s)$) exceeds the threshold $C_1$ and then initialize a new instance of SOGD starting at this round. The process is repeated until the end of the online learning process. We thus generate a sequence of markers $\{s_1,s_2,\ldots\}$. See the condition in~\pref{line:condition}, registration of markers in~\pref{line:pdGC-2}, and the overall updates in {\color{blue}Lines}~\ref{line:condition} --~\ref{line:pdGC-4} of Algorithm~\ref{alg:problem-dependent-adaptive}. Those markers specify the starting time (and the ending time) of base-learners and further we can construct the problem-dependent covering intervals as
\begin{equation}
\label{eq:covering}
	\tilde{\mathcal{C}} = \bigcup\nolimits_{k \in \mathbb{N} \cup \{0\}} \tilde{\mathcal{C}}_k, \text{ where }~ \tilde{\mathcal{C}}_k = \left\{[s_{i \cdot 2^k}, s_{(i+1) \cdot 2^k} -1] \mid i\text{ is odd} \right\} \text{ for all } k \in \mathbb{N} \cup \{0\}.
\end{equation}
This construction matches exactly the method of selecting geometric many of the unit intervals as described in~\pref{eq:geometric-generation}. It is worth noting that the problem-dependent covering intervals, unlike the standard ones which set $s_t = t$, are constructed using markers that cannot be specified with exact time stamps in advance. Instead, they are determined according to the learner's performance on the fly. However, this does not hinder the practice of our algorithm, as it activates or deactivates base-learners based on the marker indices it maintains.

\paragraph{Base-algorithm.} We employ SOGD as the base-algorithm, running with a \emph{linearized loss} $\inner{\nabla g_t(\y_t)}{\y}$ over the surrogate domain $\Y$. Denote by $A_t$ the set of active base-learners' indices, then the base-learner $\B_i$ updates by 
\begin{equation*}
	\label{eq:adaptive-base}
	\y_{t+1,i} = \Pi_{\Y}[\y_{t,i} - \eta_{t,i} \nabla g_t(\y_t)],
\end{equation*}
with $\eta_{t,i} = D/\sqrt{(\delta + \sum_{s=\tau_i}^t \norm{\nabla g_s(\y_s)}_2^2)}$, where $\tau_i$ denotes the starting time of the base-learner $i \in A_t$. The projection onto $\Y$ can be easily calculated by a simple rescaling if needed. Notably, owing to the convexity of the surrogate loss $g_t$, we can use the \emph{same} gradient $\nabla g_t(\y_t)$ for all the base-learners at each round, ensuring one gradient query of $\nabla f_t(\x_t)$ at each round.

\begin{algorithm}[!t]
\caption{Efficient Algorithm for Problem-dependent Adaptive Regret}
\label{alg:problem-dependent-adaptive}
\begin{algorithmic}[1]
\REQUIRE threshold generating function $\mathcal{G}(\cdot): \mathbb{N} \mapsto \mathbb{R}_+$.
\STATE{Initialize total intervals $m = 1$, marker $s_1 = 1$, threshold $C_1 = \mathcal{G}(1)$; let $\x_1$ be any point in $\X$; let $A_t$ denote the set of indexes for the active base-learners at time $t$.}
\FOR{$t=1$ {\bfseries to} $T$}
	\STATE Receive the gradient information $\nabla f_t(\x_t)$.
	\STATE Construct the surrogate loss $g_t: \Y \mapsto \R$ according to~\pref{eq:surrogate-loss}. \label{line:surrogate-loss}
	\STATE Compute the (sub-)gradient $\gradg_t(\y_t)$ according to Lemma~\ref{lemma:gradient-compute}.
	\STATE {Compute $L_t = L_{t-1} + f_t(\x_t)$.\\
	\vspace{1mm}
	\texttt{\% constructing Problem-dependent Geometric Covers(PGC)} \vspace{1mm}}

	\IF{$L_t > C_{m}$ \label{line:condition}}
			\STATE  Set $L_t = 0$, remove base-learners $\B_k$ whose deactivating time stamp is before the registration of $(m+1)$-th marker. \label{line:pdGC-1}
			\STATE  Set $m \leftarrow m +1$, register marker $s_m \leftarrow t$, update threshold $C_m = \mathcal{G}(m)$. \label{line:pdGC-2}
			\STATE  Initialize a new base-learner whose active span is $[s_{n \cdot 2^k}, s_{(n+1) \cdot 2^k} - 1]$ where $m = n\cdot 2^k, k \in \{0\}\cup \mathbb{N}$ and $k$ is the largest number made the factorization. \label{line:pdGC-3}
			\STATE  Set $\gamma_m = \ln(1 + 2m)$, $w_{t, m} = 1$, $\eta_{t, m} = \min\{1/2, \sqrt{\gamma_m} \}$ for the meta-algorithm. \label{line:pdGC-4}
	\ENDIF
  	\STATE Send $\gradg_t(\y_t)$ to all base-learners and obtain local predictions $\y_{t+1,i}$ for $i \in A_t$.
	\STATE Meta-algorithm updates weight $\p_{t+1} \in \Delta_{\abs{A_{t+1}}}$ by~\pref{eq:meta-AdaMLProd-inner},~\pref{eq:feedback-loss}, and~\pref{eq:meta-AdaMLProd-weight}
  	\STATE Compute $\y_{t+1} = \sum_{i \in A_{t+1}} p_{t+1,i} \y_{t+1,i}$. \label{line:combine-y-adaptive}
  	\STATE Submit $\x_{t+1} = \Pi_{\X}[\y_{t+1}]$. \LineComment{the only projection onto feasible domain $\X$ per round} \label{line:project-adaptive}
\ENDFOR
\end{algorithmic}
\end{algorithm}

\paragraph{Meta-algorithm.} SACS uses  AdaNormalHedge~\citep{COLT'15:Luo-AdaNormalHedge} as its meta-algorithm, but this choice is not suitable for us. To ensure only one projection per iteration, we cannot rely on multiple function values $\{g_t(\y_{t,i})\}_{i=1}^{\abs{A_t}}$ for the meta-algorithm to evaluate the loss. Instead, we can only use the \emph{linearized} loss value, namely, $\{ \inner{\nabla g_t(\y_t)}{\y_{t,i}}\}_{i=1}^{\abs{A_t}}$ in the meta update. The small-loss bound in SACS crucially depends on access to the original function values, which are unavailable here. Technically, with only linearized losses, it is difficult to establish a \emph{squared} gradient-norm bound and then convert it to the small-loss bound due to the \emph{first-order} regret bound of AdaNormalHedge. This key observation motivates us to instead use the Adapt-ML-Prod algorithm~\citep{COLT'14:second-order-Hedge} as our meta-algorithm. The main advantage is that it enjoys a \emph{second-order regret} while also supporting the sleeping expert setup. Adapt-ML-Prod maintains multiple learning rates $\bm{\eta}_{t+1} \in [0,1]^{\abs{A_{t+1}}}$ and an intermediate weight vector $\w_{t+1} \in \R^{\abs{A_{t+1}}}$, , updated as follows: for any active base-learner $i\in A_{t+1}$,
\begin{equation}
	\label{eq:meta-AdaMLProd-inner}
	\eta_{t+1, i} = \min \Bigg\{\frac{1}{2}, \sqrt{\frac{\gamma_i}{1 + \sum_{k=s_i}^{t} (\hat{\ell}_k - \ell_{k, i})^2}} \Bigg\},~~ w_{t+1, i} = \Big(w_{t,i} \big(1 + \eta_{t,i} (\hat{\ell}_{t} - \ell_{t, i})\big) \Big)^{\frac{\eta_{t+1, i}}{\eta_{t, i}}},
\end{equation}
where $\gamma_i = \ln(1 + 2i)$ is a certain scaling factor and the feedback loss is constructed in the following way, for $i \in A_t$, set
\begin{equation}
	\label{eq:feedback-loss}
	\hat{\ell}_t = \inner{\nabla g_t(\y_t)}{\y_t} / (2GD), \mbox{ and } \ell_{t,i} = \inner{\nabla g_t(\y_t)}{\y_{t,i}}/(2GD).
\end{equation}
The final weight vector $\p_{t+1} \in \Delta_{\abs{A_{t+1}}}$ is obtained by 
\begin{equation}
	\label{eq:meta-AdaMLProd-weight}
	p_{t+1, i} = \frac{w_{t+1, i} \cdot \eta_{t+1, i}}{ \sum_{j \in A_{t+1}} w_{t+1, j} \cdot \eta_{t+1, j}}.
\end{equation}
Notably, the meta update only uses one gradient at round $t$, namely, $\nabla g_t(\y_t)$.

Finally, we compute $\y_{t+1} = \sum_{i \in A_{t+1}} p_{t+1,i} \y_{t+1,i}$ as the overall prediction in the surrogate domain $\Y$ and calculate $\x_{t+1} = \Pi_{\X}[\y_{t+1}]$ to ensure the feasibility. This is the only projection onto $\X$ at each round. Algorithm~\ref{alg:problem-dependent-adaptive} summarizes the main procedures of our efficient method for small-loss adaptive regret. Albeit with a similar two-layer structure as SACS, our algorithm exhibits salient differences in meta-algorithm and covering intervals. As a benefit, we can successfully deploy our reduction mechanism and make the overall algorithm project onto the feasible domain $\X$ once per round, see~\pref{line:project-adaptive}. Our method retains the same small-loss adaptive regret as~\citep{ICML19:Zhang-Adaptive-Smooth}.

\begin{myThm}
\label{thm:small-loss-adaptive}
Under Assumptions~\ref{assumption:bounded-gradient}--\ref{assumption:smoothness}, setting the threshold generating function $\mathcal{G}:\mathbb{N}\mapsto \R$,
\begin{align}
	\label{eq:generate-function-adaptive}
	\mathcal{G}(m) = (54GD+168D^2L)\ln(1+2m) + 168D^2L\mu_T^2 + 18GD\mu_T + 6D\sqrt{\delta} + 672D^2L,
\end{align}
where $\mu_T =  \ln(1 + (1 + \ln(1+T))/(2e))$ and thus $\mathcal{G}(m) = \O(\log m)$, Algorithm~\ref{alg:problem-dependent-adaptive} requires only one projection onto the feasible domain $\X$ per round and enjoys the following small-loss adaptive regret for any interval $I = [r,s] \subseteq [T]$:
\begin{equation}
\label{eq:small-loss-adaptive}
		\sum_{t=r}^{s} f_t(\x_t) - \sum_{t=r}^{s} f_t(\u) \leq \O\Big(\min\big\{\sqrt{F_I \log F_I \log F_T}, \sqrt{\abs{I} \log T}\big\}\Big),
\end{equation}
 where $F_I = \min_{\x \in \X} \sum_{t=r}^{s} f_t(\x)$ and $F_T = \min_{\x \in \X} \sum_{t=1}^{T} f_t(\x)$.
\end{myThm}

\begin{myRemark}
\label{remark:match-adaptive}
Note that the $\O(\sqrt{F_I \log F_I \log F_T})$ small-loss bound of~\citet{ICML19:Zhang-Adaptive-Smooth} becomes $\O(\sqrt{\abs{I} \log \abs{I} \log T})$ in the worst case, looser than the $\O(\sqrt{\abs{I} \log T})$  bound~\citep{AISTATS'17:coin-betting-adaptive} by a factor of $\sqrt{\log \abs{I}}$. We show that this limitation can be actually avoided by the new design of the threshold mechanism and a refined analysis,  both of which are crucial for obtaining the additional $\O(\sqrt{\abs{I} \log T})$ worst-case regret guarantee. Indeed, our result in~\eqref{eq:small-loss-adaptive} can \emph{strictly} match the best known problem-independent result in the worst case. 
\end{myRemark}

\section{Interval Dynamic Regret Minimization}
\label{sec:interval-dynamic-regret}
As mentioned in Section~\ref{sec:intro-measure}, adaptive regret and dynamic regret are not directly comparable under the general OCO settings. Therefore, researchers consider optimizing them \emph{simultaneously} using a more stringent measure called \emph{interval dynamic regret}, which competes with \emph{any} changing comparator sequence over \emph{any} interval~\citep{AISTATS'20:Zhang, ICML'20:Ashok}.

\subsection{A Brief Review of Interval Dynamic Regret Minimization}
\label{sec:review-interval-dynamic}
\citet{AISTATS'20:Zhang} propose the AOA algorithm to optimize the interval dynamic regret and obtain a worst-case bound of $\sum_{t=r}^{s} f_t(\x_t) - \sum_{t=r}^{s} f_t(\u_t) \leq \O(\sqrt{\abs{I}(\log T + P_I)})$ for any interval $I = [r,s] \subseteq [T]$, where $P_I = \sum_{t=r+1}^s \norm{\u_t - \u_{t-1}}_2$ denotes the path length. The overall structure of AOA is similar to the adaptive algorithm framework described in Section~\ref{sec:efficient-adaptive-regret}. The major difference is that AOA uses Ader~\citep{NIPS'18:Zhang-Ader}, a two-layer structure algorithm, as the base algorithm to minimize dynamic regret over a given interval. As a result, the projection complexity of their method is $\O(\log^2 T)$, since each base-learner requires $\O(\log T)$ projections, and AOA employs $\O(\log T)$ such base-learners.

\citet{ICML'20:Ashok} achieves an $\Ot(\sqrt{\abs{I}(1 + P_I)})$ interval dynamic regret through a fundamentally different framework based on parameter-free online learning, where $\Ot(\cdot)$ hides the polynomial dependence on $\log T$ factors. In fact, \citet{ICML'20:Ashok} derives a \emph{gradient-norm bound} (scaling with $\sum_{t \in I} \norm{\nabla f_t(\x_t)}^2$) for interval dynamic regret, which can be easily converted into a small-loss bound (scaling with $\sum_{t \in I} f_t(\u_t)$) when the online functions are non-negative and smooth. However, we note that, even in this case, their small-loss bound is not favorable enough, as the regret guarantee is not \emph{fully} problem-dependent --- it still exhibits polynomial dependence on $\log T$ terms.

\subsection{Efficient Algorithms for Interval Dynamic Regret}
\label{sec:efficient-interval-dynamic}
In this part, we present efficient algorithms for optimizing interval dynamic regret. We begin by discussing the worst-case interval dynamic regret, which is relatively straightforward to achieve. Next, we focus on the more challenging small-loss regret bound.

\paragraph{Worst-case interval dynamic regret.} Although AOA itself does not satisfy the condition of one gradient query and one function evaluation per round, it can be verified that simply using a linearized loss for AOA achieves the same bound while meeting the reduction requirement. Consequently, applying our reduction mechanism yields an efficient algorithm for optimizing the interval dynamic regret. This algorithm retains the same worst-case bound as AOA but significantly reduces the projection complexity from $\O( \log^2 T)$ to $1$.

\paragraph{Small-loss interval dynamic regret.} Regarding the small-loss type bound, the only relevant result in the literature is the gradient-norm bound presented by~\citet{ICML'20:Ashok}, which can be further converted into a small-loss bound of order $\Ot(\sqrt{(F_I^{\u} + P_I)(1 + P_I)})$ when the online functions are non-negative and smooth, where $F_I^{\u} = \sum_{t \in I} f_t(\u_t)$ and $P_I = \sum_{t \in I} \norm{\u_t - \u_{t-1}}_2$. However, it is only a \emph{partially} problem-dependent bound, as the $\Ot(\cdot)$ notation still hides polynomial dependence on $\log T$ factors.

In the following, we present the first \emph{fully} problem-dependent small-loss bound for optimizing the interval dynamic regret, building on the AOA algorithm~\citep{AISTATS'20:Zhang}. Importantly, our algorithm requires only a single projection onto the feasible domain $\X$ per round.\footnote{Our reduction mechanism can also be applied to reduce the projection complexity of the method proposed by~\citet{ICML'20:Ashok}, but the variant still exhibits a less favorable regret  and higher computational complexity compared to the one presented here. Further details can be found in Remark~\ref{remark:small-loss-bound-comparison} and Remark~\ref{remark:projection-onto-lifted-domain}.} To achieve this, we employ the same structure as AOA, but with several modifications. Specifically, we use a linearized loss for the base-learners and Adapt-ML-Prod as the meta-algorithm. Additional effort is required to construct the problem-dependent geometric covers. Algorithm~\ref{alg:problem-dependent-interval-dynamic} summarizes the main procedures. The structure of this algorithm is similar to our algorithm for small-loss adaptive regret (see Algorithm~\ref{alg:problem-dependent-adaptive}), with the key difference that we replace the base-algorithm from SOGD with our efficient method for dynamic regret (i.e., Algorithm~\ref{alg:efficient-dynamic}). This change is necessary to manage the \emph{dual uncertainties} inherent in interval dynamic regret minimization, arising from both the unknown interval and the unknown comparator sequence. Intuitively, the design of this interval dynamic algorithm can be thought of as a \emph{three-layer structure}, because, in addition to using Adapt-ML-Prod to combine the base-learners' decisions, each base-learner also generates her own decision by combining several maintained subroutines (OGD) with a meta-algorithm (Hedge with time-varying learning rates). The details are as follows.

\begin{algorithm}[!t]
\caption{Efficient Algorithm for Problem-dependent Interval Dynamic Regret}
\label{alg:problem-dependent-interval-dynamic}
\begin{algorithmic}[1]
\REQUIRE step size pool $\H = \{\eta_1,\ldots,\eta_N\}$, learning rate setting $\epsilon_{i,t}$ for the $i$-th base-learner, and the threshold generating function $\mathcal{G}(\cdot): \mathbb{N} \mapsto \mathbb{R}_+$.
\STATE{Set total intervals $m = 1$, marker $s_1 = 1$, threshold $C_1 = \mathcal{G}(1)$; let $\x_1 \in \X$ be any point; let $A_t$ denote the set of indexes for the active base-learners at time $t$.}
\FOR{$t=1$ {\bfseries to} $T$}
    \STATE Receive the gradient information $\nabla f_t(\x_t)$.
    \STATE Construct the surrogate loss $g_t: \Y \mapsto \R$ according to~\pref{eq:surrogate-loss}. \label{line:surrogate-loss-interval-dynamic}
    \STATE Compute the (sub-)gradient $\gradg_t(\y_t)$ according to Lemma~\ref{lemma:gradient-compute}.
    \STATE {Compute $L_t = L_{t-1} + f_t(\x_t)$.\\
    \vspace{1mm}
	\texttt{\% constructing Problem-dependent Geometric Covers(PGC)} \vspace{1mm}}

    \IF{$L_t > C_{m}$ \label{line:condition-interval-dynamic}}
            \STATE Set $L_t = 0$, remove base-learners $\B_k$ whose deactivating time stamp is before the registration of $(m+1)$-th marker. \label{line:pdGC-1-interval-dynamic}
			\STATE Set $m \leftarrow m +1$, register marker $s_m \leftarrow t$, update threshold $C_m = \mathcal{G}(m)$. \label{line:pdGC-2-interval-dynamic}
            \STATE  Initialize a new base-learner according to~\pref{eq:simul-base-algo-1} and~\pref{eq:simul-base-algo-2}, whose active span is $[s_{n \cdot 2^k}, s_{(n+1) \cdot 2^k} - 1]$ where $m = n\cdot 2^k, k \in \{0\}\cup \mathbb{N}$ and $k$ is the largest number made the factorization.  
            \STATE Set the required inputs with step size pool $\H$, learning rate for Hedge used in base-algorithm  as $\epsilon_{i,t}$ with $i = m$. \label{line:pdGC-3-interval-dynamic}
            \STATE  Set $\gamma_m = \ln(1 + 2m)$, $w_{t, m} = 1$, $\eta_{t, m} = \min\{1/2, \sqrt{\gamma_m} \}$ for the meta-algorithm. \label{line:pdGC-4-interval-dynamic}
    \ENDIF
    \STATE Send $\gradg_t(\y_t)$ to all base-learners and obtain local predictions $\y_{t+1,i}$ for $i \in A_t$.
    \STATE Meta-algorithm updates weight $\p_{t+1} \in \Delta_{\abs{A_{t+1}}}$ by~\pref{eq:meta-AdaMLProd-inner},~\pref{eq:feedback-loss}, and~\pref{eq:meta-AdaMLProd-weight}. \label{line:interval-weight}

  \STATE Compute $\y_{t+1} = \sum_{i \in A_{t+1}} p_{t+1,i} \y_{t+1,i}$. \label{line:combine-y-interval-dynamic}
    \STATE Submit $\x_{t+1} = \Pi_{\X}[\y_{t+1}]$. \LineComment{the only projection onto feasible domain $\X$ per round} \label{line:project-adaptive-interval-dynamic}
\ENDFOR
\end{algorithmic}
\end{algorithm}
\paragraph{Scheduling and Meta-algorithm.} To attain the small-loss interval dynamic regret, we employ the same scheduling mechanism and meta-algorithm as Algorithm~\ref{alg:problem-dependent-adaptive} (which is designed for small-loss adaptive regret), namely, we use the problem-dependent geometric covering intervals determined by a sequence of time-varying thresholds and use the Adapt-ML-Prod as the meta-algorithm. One can refer to Section~\ref{sec:efficient-adaptive-regret} for detailed elaborations.

\paragraph{Base-algorithm.} As mentioned, we now employ our efficient method for dynamic regret minimization  (i.e., Algorithm~\ref{alg:efficient-dynamic}) as the base-algorithm, which consists two layers. We specify the update procedure of the $i$-th base-learner. At round $t+1$, the base-learner will submit the decision as $\y_{t+1,i}= \sum_{j = 1}^N p^\prime_{t+1, i,  j}\y_{t+1, i, j}$, where
\begin{align}
    \texttt{base:base-level:} \quad \y_{t+1, i, j} &= \Pi_{\Y}\left[\y_{t+1, i, j} - \eta_j \nabla g_t(\y_t) \right], j \in [N],\label{eq:simul-base-algo-1}\\
    \texttt{base:meta-level:} \quad  p^\prime_{t+1, i, j} &\propto \exp\left(-\epsilon_{i,t+1}\sum_{s=i_{\text{start}}}^t \inner{\nabla g_s(\y_s)}{\y_{s, i, j}}\right), j \in [N],\label{eq:simul-base-algo-2}
\end{align}
We use the prefix ``\texttt{base:}'' to indicate updates performed within the base-algorithm utilized by Algorithm~\ref{alg:problem-dependent-interval-dynamic}. The local decision $\y_{t+1, i, j}$ is returned from the $j$-th base:base-learner in the base-algorithm, $N$ denotes the number of base:base-learners, $i_{\text{start}}$ denotes the starting time stamp for the $i$-th base-learner and $p_{t+1, i, j}$ denotes its corresponding combination weight. Note that these base-level updates are conducted within the domain $\Y$, where the algorithm can benefit from a rapid projection. Moreover, it is also worth mentioning the choice of step size $\eta_j$. Ideally, it should be scaled with the length of the active interval $\abs{I}$ of the associated base-learner, nevertheless, we actually set it based on $T$ in Theorem~\ref{thm:small-loss-interval-dynamic-regret} due to the problem-dependent covering intervals. Technically, since the active time of each base-learner is determined on the fly, we can only set potentially over-estimated learning rates to ensure that the dynamic regret is guaranteed whenever the base-learner is deactivated. As shown in Lemma~\ref{lemma:anytime-small-loss-dynamic-regret}, this setup ensures that each base-learner enjoys an anytime dynamic regret bound.

With these components in place, the following theorem establishes the small-loss interval dynamic regret of Algorithm~\ref{alg:problem-dependent-interval-dynamic}. 
\begin{myThm}
\label{thm:small-loss-interval-dynamic-regret}
Under Assumptions~\ref{assumption:bounded-gradient}--\ref{assumption:smoothness}, employing the step size pool for base:base-algorithm (defined at~\pref{eq:simul-base-algo-1}) as $\H = \{\eta_j = 2^{j-1} \sqrt{5D^2 / (1+8LGDT)} \mid j \in [N]\}$ with $N = \lceil 2^{-1} \log_2((5D^2 + 2D^2T)(1+8LGDT)/(5D^2) ) \rceil + 1$, setting the meta-algorithm's learning rate of the $i$-th base-learner (defined at~\pref{eq:simul-base-algo-2}) as $\epsilon_{i,t} =  \sqrt{(\ln N) /(1 + D^2 \sum_{s=i_{\text{start}}}^{t-1} \|\nabla g_s(\y_s)\|_2^2)}$ with $i_{\text{start}}$ denoting its starting time stamp, setting the threshold generating function 
\begin{equation}
    \label{eq:threshold-generator-interval-dynamic}
    \begin{split}    
        \mathcal{G}(m) {}&=  7L\left(12D\sqrt{\ln(1+2m)} + 4D\mu_T + 6D\sqrt{\ln N}\right)^2  \\
         &{} + 54GD\ln(1+2m) + 18GD\mu_T + \frac{3(6+G^2D^2)\sqrt{\ln N}}{2} +(630L+23)D^2 + 9,
    \end{split}
\end{equation} 
where $\mu_T =  \ln(1 + (1 + \ln(1+T))/(2e))$ and thus $\mathcal{G}(m) = \O(\log m)$, Algorithm~\ref{alg:problem-dependent-interval-dynamic} requires only $1$ projection onto $\X$ per round and ensures that the interval dynamic regret $\sum_{t=r}^{s} f_t(\x_t) - \sum_{t=r}^{s} f_t(\u_t)$ is upper bounded by
\begin{equation}
\label{eq:small-loss-interval-dynamic}
	\O\bigg(\sqrt{(F^{\u}_I + P_I)  \cdot \big(P_I + \log \big(\min_{\u_{1:T}}\{ F^{\u}_T + P_T\}\big)\big) \cdot \log \big(\min_{\u_{r:s}}\{ F^{\u}_I + P_I\}\big)}\bigg),
\end{equation}
which holds for any interval $I = [r,s] \subseteq [T]$ and any comparator $\u_1,\cdots,\u_T \in \X$. Here, $P_I = \sum_{t=r+1}^s \norm{\u_t - \u_{t-1}}_2$ denotes the path length, $F^{\u}_I = \sum_{t=r}^{s} f_t(\u_t)$ denotes the comparators' cumulative loss, $\min_{\u_{r:s}}\{ F^{\u}_I + P_I\} = \min_{\u_{r}, \dots, \u_s}\left\{\sum_{t=r}^{s} f_t(\u_t) +  \sum_{t=r+1}^s \norm{\u_t - \u_{t-1}}_2\right\}$, and we abbreviate $P_{[1:T]}$ and $F_{[1:T]}$ by $P_T$ and $F_T$ respectively.
\end{myThm}
The proof can be found in Appendix~\ref{sec:appendix-proof-simultaneous}. Given Theorem~\ref{thm:dynamic-regret-project-smooth} and Theorem~\ref{thm:small-loss-adaptive}, a natural conjecture of interval dynamic regret bound might be $\O(\sqrt{(F_I^\u + P_I) \cdot (P_I + \log F_T^\u) \cdot \log F_I^\u})$. Nonetheless, our analysis reveals that in the context of interval dynamic regret, $F_I^\u + P_I$ should actually be treated as a combined quantity and estimated together when designing the threshold mechanism. A precise analysis can be found in Lemma~\ref{lemma:threshold-lower-bound-ada-dynamic} and Lemma~\ref{lemma:simul-upper-bound-m}.

The obtained result is versatile. When competing with a fixed comparator, the path length $P_I = 0$ and interval dynamic regret~\eqref{eq:small-loss-interval-dynamic} becomes $\O(\sqrt{F_I\log F_I \log F_T})$, recovering the small-loss adaptive regret exhibits in Theorem~\ref{thm:small-loss-adaptive}. Moreover, when considering the entire time horizon with $I = [1, T]$, the regret bound becomes $\Ot(\sqrt{(F^{\u}_T + P_T)(1+P_T)})$, which nearly matches the $\O(\sqrt{(F^{\u}_T + P_T)(1+P_T)})$ dynamic regret in Theorem~\ref{thm:dynamic-regret-project-smooth} up to logarithmic factors in $\min_{\u_{1:T}}\{ F^\u_T + P_T\}$. We note that a similar logarithmic gap of $\O(\log T)$ also exists in the study of the problem-independent interval dynamic regret~\citep[Theorem 5]{AISTATS'20:Zhang}.


Algorithm~\ref{alg:problem-dependent-interval-dynamic} requires only one gradient query, one function value, and one projection operation, improving the oracle query complexity over previous method~\citep{AISTATS'20:Zhang} for interval dynamic regret minimization. Moreover, Algorithm~\ref{alg:problem-dependent-interval-dynamic} needs to maintain $\O\big(\log T \cdot \log (\min_{\u_{1:T}} \{F^\u_T + P_T\})\big)$ instances of online gradient descent, where $\log (\min_{\u_{1:T}} \{F^\u_T + P_T\})$ can be a constant if the environments are favorable. This approach enhances the efficiency over that of~\citet{AISTATS'20:Zhang}, which requires $\O(\log^2 T)$ instances. Compared to Algorithm~\ref{alg:problem-dependent-adaptive}, which is designed for minimizing adaptive regret and requires $\O(\log F_T)$ instances, we demonstrate that $\min_{\u_{1:T}} \{F^\u_T + P_T\} \leq F_T$, since the choices of $\u_1 = \dots = \u_T = \u^\star \in \argmin_{\x \in \X} \sum_{t=1}^T f_t(\x)$ will lead to $F^\u_T + P_T = F_T$, which may not be the optimal solution to $\min_{\u_{1:T}} \{F^\u_T + P_T\}$. Therefore, Algorithm~\ref{alg:problem-dependent-interval-dynamic} can even have the computational advantage over the two-layer structure algorithm if $F_T \gg \min_{\u_{1:T}} \{F^\u_T + P_T\}$, which may occur in non-stationary environments, e.g., where the best decision in each round $\x_t^\star$ is drifting slowly.

\citet{ICML'20:Ashok} achieves an interval dynamic regret bound with techniques from parameter-free online learning. It is worthwhile to examine the subtle differences between our findings and those of~\citet{ICML'20:Ashok}; the following remarks provide this comparison.

\begin{myRemark}
\label{remark:interval-gradient-norm}
Algorithm~\ref{alg:problem-dependent-interval-dynamic} actually enjoys the same gradient-norm bound as~\citet[Theorem 7]{ICML'20:Ashok} (in fact even stronger by logarithmic factors in $T$) under Assumptions~\ref{assumption:bounded-gradient} and~\ref{assumption:bounded-domain} (without requiring the smoothness assumption). More specifically, our analysis is conducted in terms of $\norm{\gradg_t(\y_t)}_2$ (see~\pref{eq:simul-dynamic-meta-regret} and~\pref{eq:simul-dynamic-base-regret}), which can then be related to $\norm{\nabla f_t(\x_t)}_2$ by Lemma~\ref{lemma:gradient-compute}. Following the same convention of \citet{ICML'20:Ashok}, who uses the $\Ot(\cdot)$-notation to hide the logarithmic dependence in $T$, Algorithm~\ref{alg:problem-dependent-interval-dynamic} can obtain an $\tilde{\O}(\sqrt{(G_I + P_I)(1+P_I)})$ interval dynamic regret for any interval $I = [r,s] \subseteq [T]$ without any modification on the learning rates, where $P_I = \sum_{t=r+1}^{s} \|\u_{t} - \u_{t-1}\|_2$ and $G_I = \sum_{t=r}^{s} \|\gradf_t(\x_t) \|_2^2$. 
\end{myRemark}
\begin{myRemark}
\label{remark:small-loss-bound-comparison}
The interval dynamic regret bound obtained by~\citet{ICML'20:Ashok} essentially implies a small-loss bound $\Ot(\sqrt{( F^{\u}_I + P_I)(1+P_I)})$ when assuming the smoothness of loss functions. However, the $\Ot(\cdot)$ notation hides problem-independent factors such as $\log T$. In contrast, our result in Theorem~\ref{thm:small-loss-interval-dynamic-regret} is fully problem-dependent, replacing the quantity such as $\log T$ by $\log(\min_{\u_{1:T}}\{F^{\u}_T + P_T\})$, through our techniques of constructing problem-dependent covering intervals and corresponding threshold designs. These problem-dependent quantities can be much smaller under nice environments, and it remains unclear how to adapt our techniques to the method of~\citet{ICML'20:Ashok} due to a different framework. 
\end{myRemark}

\begin{myRemark}
\label{remark:projection-onto-lifted-domain}
The proposed Algorithm~\ref{alg:problem-dependent-interval-dynamic} requires only one projection onto the original domain, whereas the method by~\citet{ICML'20:Ashok} requires multiple projections onto a more complex, lifted domain per round, raising concerns about computational efficiency.  In Appendix~\ref{subsec:discussion-ashok-method}, we present how to apply our reduction mechanism to the method of~\citet{ICML'20:Ashok}, replacing most of the potentially costly projections onto the lifted domains by projections onto cylinders. However, one projection onto the complex domain remains necessary.
\end{myRemark}

\section{Applications}
\label{sec:application}
In this section, we provide two applications of our proposed reduction mechanism for non-stationary online learning, including  minimizing the \emph{dynamic regret of online non-stochastic control} and minimizing the \emph{adaptive regret of online principal component analysis}.

It is important to note that both problems are adaptations of the standard online learning settings with essential modifications. Specifically, there are three key points worth mentioning: (i) both problems operate over the matrix space instead of the vector space studied in previous sections; (ii) for online non-stochastic control, our reduction mechanism is further enhanced to account for the \emph{switching cost} of algorithmic decisions~\citep{NIPS'15:OCOmemory,ICML'19:online-control}, a crucial characteristic of this decision-theoretic problem; (iii) for online principal component analysis, we actually present the first \emph{strongly} adaptive regret result (still with one projection per round), improving the regret bound upon the previously best known result~\citep{ICML'19:adaptive-PCA} that only achieves weakly adaptive regret.

\subsection{Online Non-Stochastic Control}
In this part, we apply our reduction mechanism to an important online decision-making problem, online non-stochastic control, which attracts much attention these years in online learning and control theory community~\citep{NIPS'19:control,ICML'20:log-control,ALT'20:control-Hazan,NIPS'20:bandit-control,JMLR'23:memory,hazan2022introduction}.

\subsubsection{Problem Formulation}
We focus on the online control of linear dynamical system (LDS) defined as $x_{t+1} = Ax_t + Bu_t + w_t$, where $x_t$ is the state, $u_t$ is the control, $w_t$ is a disturbance to the system. The controller suffers cost $c_t(x_t, u_t)$ with convex function $c_t:\R^{d_x} \times \R^{d_u}\mapsto \R$. Throughout this subsection, we follow the convention of notations in the non-stochastic control community to use unbold fonts to represent vectors and matrices. In online non-stochastic control, since there are no statistical assumptions imposed on system disturbance $w_t$ and additionally the cost function can be chosen adversarially. 
The adversarial nature of the control setting hinders us from precomputing the optimal policy, as is possible in classical control theory~\citep{kalman1960contributions}, and therefore requires modern online learning techniques to tackle adversarial environments.

We adopt \emph{dynamic policy regret}~\citep{JMLR'23:memory} to benchmark the performance of the designed controller with a sequence of arbitrary \emph{time-varying} controllers $\pi_1, \dots, \pi_T \in \Pi$, 
\begin{align}
\label{eq:dynamic-policy-regret}
  \Dreg(\pi_{1},\ldots,\pi_T) = \sum_{t=1}^T c_t(x_t, u_t) - \sum_{t=1}^T c_t(x_t^{\pi_t}, u_t^{\pi_t}).
\end{align}
For this problem, the pioneering work~\citep{ICML'19:online-control} investigates the static regret  as~\eqref{eq:dynamic-policy-regret} with $\pi_{1},\ldots,\pi_T \in \argmin_{\pi \in \Pi} \sum_{t=1}^T c_t(x_t^{\pi}, u_t^{\pi})$. The authors propose a gradient-based controller with $\Ot(\sqrt{T})$ static regret. Specifically, they propose to employ the Disturbance-Action Controller (DAC) policy class $\pi(K, M)$, which is parametrized by a fixed matrix $K\in \R^{d_u \times d_x}$ and parameters matrix tuple $M = (M^{[1]}, \dots, M^{[H]}) \in (\R^{d_u \times d_x})^H$ with a memory length $H$. At each round, DAC makes the decision as a linear map of the past disturbances with an offset linear controller $u_t = -Kx_t + \sum_{i=1}^H M^{[i]}w_{t-i}$. This parametrization makes the action as a linear function of the past disturbances and further can reduce the online non-stochastic control to online convex optimization with memory (OCO with Memory)~\citep{NIPS'15:OCOmemory} with truncated loss, and thus one can further apply techniques developed for OCO with memory to handle the non-stochastic control.

For online non-stochastic control,~\citet{JMLR'23:memory} propose an online control approach with an $\Ot(\sqrt{T(1+P_T)})$ dynamic policy regret, where $P_T = \sum_{t=2}^T \norm{M_{t-1}^* - M_t^*}_{\operatorname{F}}$ denotes the cumulative variation of comparators. The algorithm leverage an online ensemble structure equipped with $\O(\log T)$ base-learners, which leads to an $\O(\log T)$ projection complexity. In the sequel, we will investigate the computational complexity of the projection operation. Our aim is to refine the method through our efficient reduction mechanism, obtaining an algorithm that retains the \emph{same} regret guarantee while requiring \emph{one} projection per round.

\subsubsection{Projection Computational Complexity}
Previous studies project the parameters matrix tuple $M$ onto the following domain,
\begin{align}
  \label{eq:control-domain}
  \M = \Big\{M \triangleq (M^{[1]},\dots, M^{[H]}) \in \left(\R^{d_u \times d_x}\right)^H \mid \norm{M^{[i]}}_{\operatorname{op}} \leq c_i \Big\},
\end{align}
where $\norm{\cdot}_{\operatorname{op}}$ denotes the operator norm and $c_i$ is some fixed constant. The projection can be done by projecting each matrix $M^{[i]}, i \in [H]$, onto spectral norm ball with radius $c_i$ sequentially. For each $M^{[i]}$, the projection involves the dominant process of diagonalization and project the singular values onto $\ell_{\infty}$ ball. Thus, the projection computational complexity onto $\M$ is in order of $\O\left(H\min\{d_u^2 d_x, d_ud_x^2\}\right)$, which dominates the computational expense in the base-learners, given that the gradient descent step requires $\O(H\cdot d_ud_x)$ time only.

\subsubsection{Efficient Reduction}
\label{sec:efficient-reduction-control}
In this part, we apply our efficient reduction mechanism to Scream.Control~\citep{JMLR'23:memory}, which can improve the projection complexity and maintain the theoretical guarantee.
The algorithm is summarized in Algorithm~\ref{alg:control} and we introduce the main ingredients below. 

To facilitate the efficient reduction, we design the following surrogate domain,
\begin{align}
  \label{eq:control-surr-domain}
  \M^\prime = \left\{M = (M^{[1]},\dots, M^{[H]}) \in \left(\R^{d_u \times d_x}\right)^H \mid \norm{M^{[i]}}_{\operatorname{F}} \leq c_i\sqrt{d} \right\},
\end{align}
as the replacement of the original one defined at~\pref{eq:control-domain}, where we denote by $d = \min\{d_x, d_u\}$. Notice that Scream.Control already satisfies the reduction requirements of querying only one function value and one gradient value per round~\citep{JMLR'23:memory}. Also, this algorithm utilizes linearized cost function to perform update for meta-learner and base-learners, enabling the extension of surrogate loss defined at~\pref{eq:surrogate-loss} to matrix version $g_t(\cdot): \M^{H+2}\mapsto \R$:
\begin{align}
  \label{eq:surrogate-loss-control}
 g_t(M) = \inner{\nabla \ft(M_t)}{M}  - \ind{\inner{\nabla \ft(M_t)}{V_t}<0}\cdot \inner{\nabla \ft(M_t)}{V_t} \cdot S_{\M}(M),
\end{align}
where $S_{\M}(M) = \inf_{A \in \M}\norm{A - M}_{\operatorname{F}}$ is the distance function to $\M$ and we denote by $V_t = {(M^\prime_t - M_t)}/{\norm{M^\prime_t - M_t}_{\operatorname{F}}}$ the projection direction with $M_t, M_t^\prime$ defined in Algorithm~\ref{alg:control}. 
We denote by $\ft_t(\cdot)$ the unary truncated loss function constructed from cost $c_t(\cdot, \cdot)$ and refer the interested readers to Section 5.3 of~\citet{JMLR'23:memory} for more details, as in this paper we mainly focus on the projection issues. It is worth emphasizing that the truncated loss circumvents the growing of memory length with time, enabling the application of techniques from OCO with memory, while the gap between $\tilde{f}_t(\cdot)$ and $c_t(\cdot, \cdot)$ will not be too large. 

The caveat to apply efficient method to Scream.Control remains that we should ensure the adoption of surrogate loss and surrogate domain will not ruin the transformation from non-stochastic control to OCO with memory, which requires the algorithm to account for the \emph{switching cost} $\norm{M_{t-1} - M_{t}}_{\operatorname{F}}$ between parameters as well. 
Inspecting the algorithm derived from the efficient reduction closely, for two parameters $M^\prime_{t-1}, M^\prime_{t}$ in the surrogate domain and the submitted parameters $M_{t-1} = \Pi_{\M}[M^\prime_{t-1}], M_{t} = \Pi_{\M}[M^\prime_{t}]$, the nonexpanding property of the projection operator in the Hilbert space implies that $\norm{\Pi_{\M}[M^\prime_{t-1}] -\Pi_{\M}[M^\prime_t]}_{\operatorname{F}} \leq \norm{M^\prime_{t-1} -M^\prime_{t}}_{\operatorname{F}}$~\citep{SIAM'09:nemirovski-robust-stochastic}, meaning that the switching cost $\norm{M_{t-1} - M_{t}}_{\operatorname{F}}$ can be controlled when applying efficient reduction mechanism. Therefore we can safely apply the efficient reduction to Scream.Control and improve the projection efficiency. Algorithm~\ref{alg:control} enjoys the following theoretical guarantee with the proof sketch presented in Appendix~\ref{sec:appendix-control}.

\begin{algorithm}[!t]
  \caption{Efficient Control Algorithm for Dynamic Policy Regret}
  \label{alg:control}
  \begin{algorithmic}[1]
  \REQUIRE Scream.Control algorithm~$\A$.
  \STATE Initialization: let $\A$ project onto domain $\M^\prime$; submit $M_{H} \in \mathcal{M}$ in $[1, H]$ rounds.
  \FOR{$t=H+1$ {\bfseries to} $T$}
    \STATE Observe $c_t(\cdot, \cdot)$ and calculate the gradient of truncated loss function $\nabla \ft_t(M_t)$. \label{line:control-truncated}
    \STATE Construct the surrogate loss $g_t(\cdot)$ according to \pref{eq:surrogate-loss-control}.
    \STATE Compute the (sub-)gradient $\nabla g_t(M_t^\prime)$ by Lemma~\ref{lemma:gradient-compute} with extension to matrix.
    \STATE Send linearized loss $h_t(M) = \operatorname{tr}\left(\nabla g_t(M_t^\prime) \cdot M\right)$ to $\A$ for update. \label{line:control-switching}
    \STATE Obtain decision $M_{t+1}^\prime$ from $\A$ and submit $M_{t+1} = \Pi_{\M}[M_{t+1}^\prime]$.
  \ENDFOR
  \end{algorithmic}
  \end{algorithm}

\begin{myThm}
  \label{thm:control}
  Under Assumptions~\ref{assump:control-1}-\ref{assump:control-3}, by choosing $H = \Theta(\log T)$, Algorithm~\ref{alg:control} enjoys the dynamic regret of $\sum_{t=1}^T c_t(x_t, u_t) - \sum_{t=1}^T c_t(x_t^{\pi_t}, u_t^{\pi_t}) \leq \Ot\big(\sqrt{T(1 + P_T)}\big)$, with \emph{one} projection per round. Here, the comparators can be any feasible policies in $\Pi=\left\{\pi(K, M)\mid M \in \M\right\}$ with $\pi_t = \pi(K, M_t^*)$ for $t\in [T]$, and $P_T = \sum_{t=2}^T \Fnorm{M^*_{t-1} - M^*_t}$ measures the variation.
\end{myThm}
\begin{myRemark}
  The original algorithm of~\citet{JMLR'23:memory} requires maintaining $\O\left(\log T\right)$ base-learners, resulting in a computational cost of $\O\left(H\cdot\log T\cdot\min\{d_u^2 d_x, d_ud_x^2\}\right)$ per round. In our variant algorithm, it permits the base-learners to project onto $\mathcal{M}^\prime$, which can be achieved by simply rescaling the matrix. Consequently, the computational complexity can be significantly reduced by a factor of $\O(\log T)$ as depicted in Algorithm~\ref{alg:control}.
\end{myRemark}

\begin{myRemark}
  \label{remark:measure-note-control}
  Note that this part focuses on optimizing dynamic regret of online non-stochastic control, rather than interval dynamic regret. We now explain why. Due to the switching cost, optimizing non-stationary regret is significantly more challenging than in standard OCO. The meta-base structure studied in this paper can be used to optimize the \emph{dynamic regret} with injected correction terms, as demonstrated in~\citep{JMLR'23:memory}.
  On the other hand, optimizing \emph{adaptive regret} of OCO with switching cost requires a different approach. The current method~\citep{NeurIPS'22:adaptive-switching-cost} builds on parameter-free online learning~\citep{ICML'20:Ashok} to optimize the adaptive regret and tolerant the switching cost (so it may be possible to apply our reduction mechanism to~\citep{ICML'20:Ashok}, as discussed in Appendix~\ref{subsec:discussion-ashok-method}, to reduce certain projection complexities). However, it remains unclear whether the meta-base approach can effectively optimize adaptive regret.
  Therefore, we focus on dynamic regret to streamline the projection complexity of the algorithm proposed in~\citep{JMLR'23:memory}. The result demonstrates the effectiveness and practicality of reducing projection complexity, especially given the complications of the feasible domains in non-stochastic control.
\end{myRemark}

\subsection{Online Principal Component Analysis}
Principal Component Analysis (PCA) is a crucial dimensionality reduction technique, widely used in data processing, machine learning, and many more. Unlike the conventional offline PCA, online PCA is designed for scenarios where data arrive sequentially, thereby necessitating to conduct dimensionality reduction in an online manner. 

\subsubsection{Problem Formulation}
Online (uncentered-)PCA problem requires algorithms to forecast the optimal projection subspace 
upon receiving a series of streaming data on the fly~\citep{jmlr'08:Warmuth-pca,NIPS'13:PCA-projection,jmlr'16:Warmuth-pca}. Concretely, at each time $t$ the algorithm receives an instance $\x_t \in \R^d$ (or in a more general setting, the algorithm receives an instance matrix $\mathbf{X}_t \in \R^{d\times d}$) and needs to project it onto a
$k$-dimensional subspace ($k < d$) represented by a rank-$k$ \emph{projection matrix} $\P_t \in \Pcal_k$. The domain of rank-$k$ projection matrices is defined as
\begin{align}
  \label{eq:pca-domain}
  \Pcal_k = \left\{\P \in \mathbb{S}^{d}\mid \sigma_i(\P) \in \{0, 1\}, \operatorname{rank}(\P) = k\right\},
\end{align}
where $\mathbb{S}^d$ denotes the set of real-valued $d\times d$ symmetric matrices and $\sigma_i(\cdot)$ denotes the $i$-th eigenvalue of the given matrix. 

Online PCA uses \emph{compression loss} $f_t(\P) = \|\P\x_t - \x_t\|_2^2$ to measure the reconstruction error at round $t$. Many prior works considers minimizing static regret for the online PCA problem, which benchmarks the cumulative compression loss of the learner against the fixed projection matrix in hindsight. However, in many environments the online data distributions can change over time, it is crucial to consider the non-stationarity issue in the algorithmic design. To this end, we investigate the adaptive regret for online PCA, which requires the algorithm to perform well for any interval $I \subseteq [T]$ with length $\tau = |I|$, defined as
\begin{align}
\label{eq:strong-adaptive-PCA}
  \Areg(|I|) = \max_{[r, r + \tau - 1] \subseteq [T]} \left\{\sum_{t=r}^{r + \tau -1}f_t(\P_t) - \min_{\P \in \Pcal_k}\sum_{t=r}^{r + \tau -1}f_t(\P) \right\}.
\end{align}

\citet{ICML'19:adaptive-PCA} examine a regret notion similar to but weaker than~\pref{eq:strong-adaptive-PCA}, defined as $\WAreg(T) = \max_{[r,q] \subseteq [T]} \left\{\sum_{t=r}^{q}f_t(\P_t) - \min_{\P \in \Pcal_k}\sum_{t=r}^{q}f_t(\P) \right\}$. This variant is usually termed as the \emph{weakly} adaptive regret~\citep{journal'07:Hazan-adaptive,ICML'18:zhang-dynamic-adaptive}, which lacks the guarantee for intervals with $|I| \leq \O(\sqrt{T})$. It should be noted that~\citet{ICML'19:adaptive-PCA} propose an algorithm with an $\Ot(\sqrt{T})$ weakly adaptive regret for online PCA, and to the best of our knowledge, current literature lacks algorithms with an $\Ot(\sqrt{|I|})$ strongly adaptive regret for online PCA. We not only design the first algorithm with such a strongly adaptive regret guarantee, but also implement our reduction mechanism to ensure that it enjoys a projection complexity of $1$.

\subsubsection{Projection Computational Complexity}
Before delving into the efficient projection mechanism concerned in our paper, we initiate with a brief overview on the projection challenge (due to the non-convexity issue) and its solution in the existing online PCA literature. 

Notice that the feasible domain $\Pcal_k$ defined in~\pref{eq:pca-domain} is inherently \emph{non-convex}, making it hard to apply OCO techniques. To remedy this, the convex hull of $\Pcal_k$ is usually employed as a surrogate during the update, defined as $\hat{\Pcal}_k = \big\{\P \in \mathbb{S}^{d} \mid \boldsymbol{0} \preceq \P \preceq \I, \operatorname{tr}(\P) = k\big\}$. Nonetheless, the online PCA protocol requires the algorithm provide the decision within $\Pcal_k$. To this end, the pioneering study of~\citet{jmlr'08:Warmuth-pca} decomposes $\hat{\P}$ into a convex combination of, at most, $d$ rank-$k$ projection matrices represented as $\hat{\P} = \sum_{i=1}^d \lambda_i \P_i$, where $\lambda_i \in [0, 1]$ constitutes a distribution, and each $\P_i \in \Pcal_k$ is a rank-$k$ projection matrix. Following this decomposition, one can leverage the composite coefficients $\lambda_i$ to sample a projection matrix as the submitted decision.

The gradient descent algorithm~\citep{jmlr'16:Warmuth-pca} can obtain $\O(\sqrt{kT})$ static regret for online PCA, which mainly consists of the following steps:
\begin{align*}
  \hat{\P}_{t+1}^\prime = \hat{\P}_t - \eta \nabla f_t(\P_t), \quad \hat{\P}_{t+1} = \argmin\nolimits_{\P \in \hat{\Pcal}_k} \norm{\P - \hat{\P}_{t+1}^\prime}_{\operatorname{F}},
\end{align*}
where $\eta > 0$ represents the step size and $\nabla f_t(\P_t)$ denotes the gradient with respect to $\P_t$.
Then the algorithm samples a rank-$k$ projection matrix based on $\hat{\P}_{t+1}$ to submit. The gradient descent step requires $\O(d^2)$ time expense, while the primary bottleneck is the projection step onto $\hat{\Pcal}_k$, which typically demands $\O(d^3)$ computational complexity, owing to the matrix diagonalization process, as illustrated by Lemma~\ref{lemma:pca-projection} in Appendix~\ref{sec:appendix-pca}.

\subsubsection{Efficient Reduction}
In this part, we provide our efficient algorithm for online PCA with the strongly adaptive regret guarantee. The algorithm is presented in Algorithm~\ref{alg:PCA} and we introduce the necessary components in the below paragraphs.

By inspecting the convex surrogate domain $\hat{\mathcal{P}}_k$ carefully, onto which the algorithm projects, we propose the surrogate domain $\hat{\Pcal}_k^s$ defined by the Frobenius norm as,
\begin{align}
  \label{eq:surrogate-domain-pca}
  \hat{\Pcal}_k^s = \left\{\P \in \mathbb{S}^{d}\mid \norm{\P}_{\operatorname{F}} \leq \sqrt{k}\right\},
\end{align}
which admits a fast projection by a simple rescaling. 

At first glance, the compression loss $f_t(\P) = \|\P\x_t - \x_t\|_2^2$ seems to be quadratic, but indeed it is coordinate-wise linear with parameter $\P$ as shown below,
\begin{align}
  \label{eq:pca-linear-loss-function}
 f_t(\P) = \|\P\x_t - \x_t\|_2^2 = \operatorname{tr}\big((\I - \P)^2\x_t\x_t^\top\big) = \operatorname{tr}\big((\I - \P)\x_t\x_t^\top\big),
\end{align}
where the second equality is by the property of projection matrix $\P$. To ensure our algorithm's adaptability across varied scenarios, we consider a general setting, where the algorithm receives any semi-positive matrix $\Xb_t \in \R^{d\times d}$ as input instance rather than vector. Therefore, the loss function considered at~\pref{eq:pca-linear-loss-function} becomes $f_t(\P) = \operatorname{tr}\left((\I - \P)\mathbf{X}_t\right)$. We extend the surrogate loss defined at~\pref{eq:surrogate-loss} to online PCA as
\begin{align}
  \label{eq:surrogate-loss-PCA}
  g_t(\P) = \operatorname{tr}(\nabla f_t(\hat{\P}_t)\cdot\P) - \ind{\operatorname{tr}(\nabla f_t(\hat{\P}_t) \cdot \Vb_t)< 0}\cdot \operatorname{tr}(\nabla f_t(\hat{\P}_t) \cdot \Vb_t) \cdot S_{\hat{\Pcal}_k}(\P),
\end{align}
where $S_{\hat{\Pcal}_k}(\P) = \inf_{\mathbf{Q} \in \hat{\Pcal}_k} \|\P-\mathbf{Q}\|_{\operatorname{F}}$ is the minimum distance from $\P$ to the domain $\hat{\Pcal}_k$ and $\Vb_t = (\hat{\P}_{t} -\hat{\P}_{t}^s)/{\|\hat{\P}_{t} -\hat{\P}_{t}^s\|_{\operatorname{F}}}$ denotes the matrix indicating the projection direction with $\hat{\P}_{t},\hat{\P}_{t}^s$ defined in Algorithm~\ref{alg:PCA}. This surrogate loss enjoys the benign properties, as illustrated by Theorem~\ref{thm:surrogate-loss} and Lemma~\ref{lemma:gradient-compute}. These two theoretical results are indeed consistent with the nearest-point projection in the Hilbert space. As for online PCA, the loss function is defined by the trace operator, and we employ the Frobenius norm as the distance metric for projection, which implies our optimization operates within a Hilbert space, making the aforementioned results directly applicable.

\begin{algorithm}[!t]
  \caption{Efficient Algorithm for Adaptive Regret under PCA Setting}
  \label{alg:PCA}
  \begin{algorithmic}[1]
  \STATE Initialization: let $\P_1 = \hat{\P}_1$ be any point in $\Pcal_k$; let $A_t$ denote the set of indexes for the active base-learners at time $t$.
  \FOR{$t=1$ {\bfseries to} $T$}
    \STATE Receive instance matrix $\Xb_t \in \mathbb{S}^{d}$. 
    \STATE Construct the surrogate loss $g_t(\cdot)$ according to~\eqref{eq:surrogate-loss-PCA}. \label{line:surrogate-loss-pca}
    \STATE Compute the (sub-)gradient $\nabla g_t(\hat{\P}^s_t)$ by Lemma~\ref{lemma:gradient-compute} with extension to this problem.
    \STATE Remove base-learners whose deactivating time  is $t$ according to $\mathcal{C}$ defined at~\pref{eq:standard-gc}.
    \STATE Initialize base-learner whose start time is $t+1$ and set the learning rate for her $\eta_{t+1} = \frac{k(d-k)}{d\abs{I_{t+1}}}$ where $I_{t+1}$ is the active span according to $\mathcal{C}$ defined at~\pref{eq:standard-gc}.
    \STATE Send $\nabla g_t(\hat{\P}^s_t)$ to each base-learner for update.
    \STATE For each $i \in A_t$, the base-learner updates the decision within $\hat{\mathcal{P}}_k^s$ defined at~\pref{eq:surrogate-domain-pca},
    \begin{equation*}
      \hat{\P}^{s, \prime}_{t+1,i}= \hat{\P}^s_{t,i} - \eta_i \nabla g_t(\hat{\P}^s_t),~~ \hat{\P}^s_{t+1,i} = \hat{\P}^{s, \prime}_{t+1,i} \Big(\ind{\Fnorm{\hat{\P}^{s, \prime}_{t+1,i}} \leq \sqrt{k}} + \frac{\sqrt{k}}{\Fnorm{\widehat{\P}_{t+1}}}\ind{\Fnorm{\hat{\P}^{s, \prime}_{t+1,i}} > \sqrt{k}}\Big).
    \end{equation*} 
    \STATE Meta-algorithm updates weight $\p_{t+1} \in \Delta_{\abs{A_{t+1}}}$ according to~\pref{eq:meta-AdaMLProd-inner},~\pref{eq:feedback-loss}, and~\pref{eq:meta-AdaMLProd-weight} with $\hat{\ell}_t = {\operatorname{tr}\left(\nabla g_t(\hat{\P}^s_t) \cdot \hat{\P}_t^s \right)}/{2\sqrt{k}}$ and $\ell_{t,i} = {\operatorname{tr}\left(\nabla g_t(\hat{\P}^s_t) \cdot \hat{\P}_{t,i}^s \right)}/{2\sqrt{k}}$.
    \STATE Compute $\hat{\P}^s_{t+1} = \sum_{i\in A_{t+1}}p_{t+1, i}\cdot \hat{\P}^{s}_{t+1,i}$, and $\hat{\P}_{t+1} = \Pi_{\hat{\Pcal}_k}[\hat{\P}^s_{t+1}]$.
    \STATE Sample a projection matrix $\P_{t+1} \sim \hat{\P}_{t+1}$ to submit.
  \ENDFOR
  \end{algorithmic}
\end{algorithm}

Given the absence of a strongly adaptive algorithm for the online PCA problem in the literature, we offer a detailed description of our algorithm (as opposed to the black-box reduction style for online non-stochastic control). We propose the algorithm by incorporating the gradient descent method~\citep{jmlr'16:Warmuth-pca} as the base-algorithm, Adapt-ML-Prod as the meta-algorithm, and the covering intervals defined at~\pref{eq:standard-gc}. Our efficient online PCA algorithm satisfies the following theorem, and we provide a proof sketch in Appendix~\ref{sec:appendix-pca}:
\begin{myThm}
  \label{thm:pca-adaptive}
  Assuming that $\Fnorm{\Xb_t} \leq 1$ for any $t\in [T]$ and $k \leq \frac{d}{2}$, then Algorithm~\ref{alg:PCA} requires only one projection onto the domain $\hat{\Pcal}_k$ per round and enjoys the following adaptive regret for any interval $I = [r, s]\subseteq [T]$: $\mathbb{E}\big[\sum_{t=r}^{s}f_t(\P_t)\big] - \min_{\P \in \Pcal_k}\sum_{t=r}^{s}f_t(\P) \leq \Ot\big(\sqrt{k\cdot |I|}\big)$, where we adopt the general setting by choosing $f_t(\P) = \operatorname{tr}\left((\mathbf{I} - \P)\Xb_t\right)$ and  the expectation is due to the randomness introduced by the sampling of the algorithm.
\end{myThm}
\begin{myRemark}
  The projection operation onto $\hat{\Pcal}_k$ is dominated by the matrix diagonalization which is of $\O(d^3)$ under general instances assumption. The vanilla adaptive PCA algorithm incurs $\O(d^3 \log T)$ computational cost by maintaining $\O(\log T)$ base-learners. Our efficient algorithm requires one projection and improves the computational cost to $\O(d^3)$ per round.
\end{myRemark}

\begin{myRemark}
  \label{remark:measure-note-PCA}
  For online PCA, we focus on strongly adaptive regret to provide a strict improvement over~\citet{ICML'19:adaptive-PCA} --- we achieve a reduction in projection complexity from $\O(\log T)$ to $1$ per round and meanwhile improving the regret guarantee (previous result only considers weakly adaptive regret), significantly enhancing both computational efficiency and regret performance. While extending the analysis to \emph{interval dynamic regret} for online PCA is feasible, it would complicate the process. Instead, we focus on adaptive regret to highlight the primary contributions of our method, particularly in terms of projection efficiency, which is crucial given the complexity of the feasible domain in online PCA.
\end{myRemark}
\section{Experiment}
\label{sec:experiment}
In this section, we provide empirical studies to evaluate our proposed methods.

\paragraph{General Setup.} We conduct experiments on the synthetic data. We consider the following online regression problem. Let $T$ denote the number of total rounds. At each round $t \in [T]$ the learner outputs the model parameter $\w_t \in \W \subseteq \mathbb{R}^d$ and simultaneously receives a data sample $(x_t, y_t)$ with $x_t \in \X \subseteq \mathbb{R}^d$ being the feature and $y_t \in \mathbb{R}$ being the corresponding label.\footnote{With a slight abuse of notations, we here use $\w$ to denote the model parameter and $\W$ to denote the feasible domain, and meanwhile we reserve the notations of $x$ and $\X$ to denote the feature and feature space following the conventional notations of machine learning terminologies.} The learner can then evaluate her model by the online loss $f_t(\w_t) = \frac{1}{2} (x_t^\top \w_t - y_t)^2$ which uses a square loss to evaluate the difference between the predictive value $x_t^\top \w_t$ and the ground-truth label $y_t$, and then use the feedback information to update the model. In the simulations, we set $T=20000$, the domain diameter as $D=6$, and the dimension of the domain as $d=8$. The feasible domain $\W$ is set as an ellipsoid $\W = \left\{\w \in \mathbb{R}^d \mid \w^\top \mathbf{E}\w \leq \lambda_{\min}(\mathbf{E}) \cdot (D/2)^2 \right\}$, where $\mathbf{E}$ is a certain diagonal matrix and $\lambda_{\min}(\mathbf{E})$ denotes its minimum eigenvalue. Then, a projection onto $\W$ requires solving a convex programming that is generally expensive. In the experiment, we use \texttt{scipy.optimize.NonlinearConstraint} to solve it to perform the projection onto the feasible domain.

To simulate the non-stationary online environments, we control the way to generate the date samples $\{(x_t, y_t)\}_{t=1}^T$. Specifically, for $t \in [T]$, the feature $x_t$ is randomly sampled in an Euclidean ball with a diameter $D$ same as the feasible domain of model parameters; and the corresponding label is set as $y_t = x_t^\top \w_t^* + \epsilon_t$, where $\epsilon_t$ is the random noise drawn from $[0, 0.1]$ and $\w_t^*$ is the underlying ground-truth model from the feasible domain $\W$ generated according to a certain strategy specified below.  For dynamic regret minimization, we simulate \emph{piecewise-stationary} model drifts, as dynamic regret will be linear in $T$ and thus vacuous when the model drift happens every round due to a linear path length measure. Concretely, we split the time horizon evenly into $25$ stages and restrict the underlying model parameter $\w_t^*$ to be stationary within a stage. For adaptive regret minimization, we simulate \emph{gradually evolving} model drifts, where the underlying model parameter $\w_{t+1}^*$ is generated based on the last-round model parameter $\w_t^*$ with an additional random walk in the feasible domain $\W$. The step size of random walk is set proportional to $D/T$ to ensure a smooth model change.

\paragraph{Contenders.} For both dynamic regret and adaptive regret, we directly work on the small-loss online methods. We choose the Sword algorithm~\citep{JMLR:sword++} as the contender of our efficient method for dynamic regret (Algorithm~\ref{alg:efficient-dynamic}) and  choose the SACS algorithm~\citep{ICML19:Zhang-Adaptive-Smooth} as the contender of our efficient method for adaptive regret (Algorithm~\ref{alg:problem-dependent-adaptive}).

\begin{figure}[!t]
  \centering
  \subfigure[dynamic regret (loss)]{ \label{fig:loss-dynamic}
    \includegraphics[clip, trim=0.6cm 0.3cm 1.5cm 1.4cm,width=0.42\textwidth]{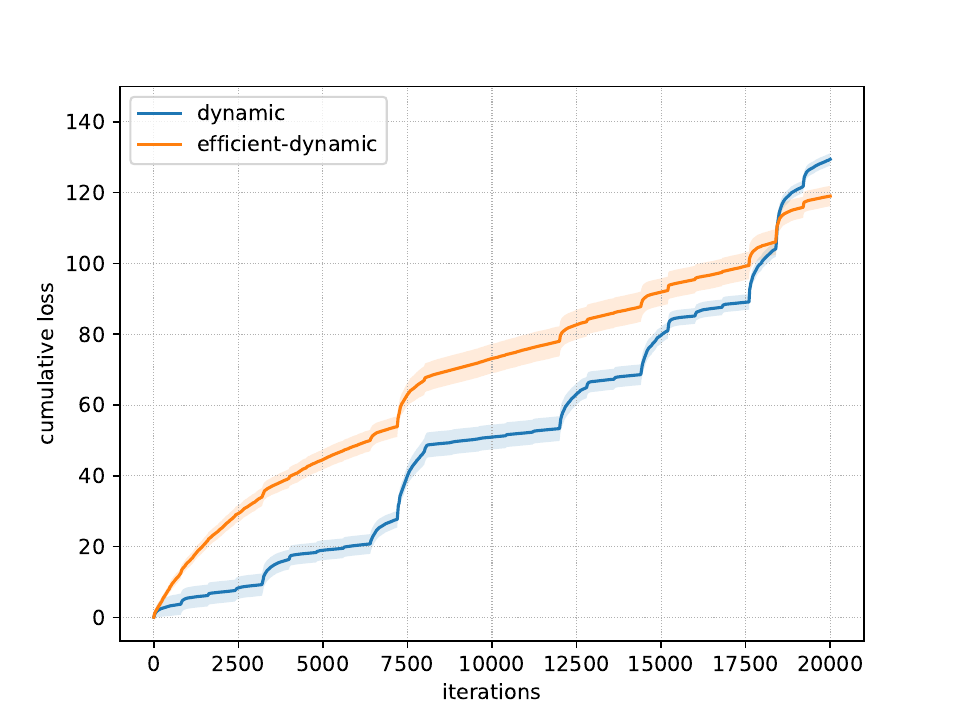}}\hspace{5mm}
  \subfigure[dynamic regret (time)]{ \label{fig:time-dynamic}
    \includegraphics[clip, trim=0.6cm 0.3cm 1.5cm 1.4cm,width=0.42\textwidth]{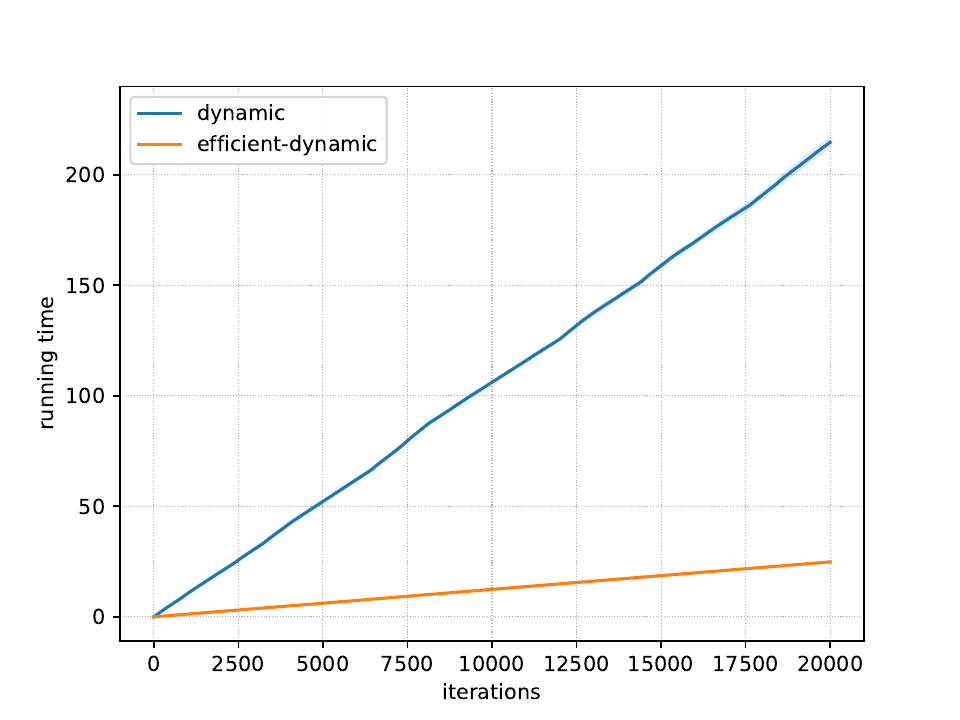}}\vspace{-1mm}
  \subfigure[adaptive regret (loss)]{ \label{fig:loss-adaptive}
    \includegraphics[clip, trim=0.6cm 0.3cm 1.5cm 1.4cm,width=0.42\textwidth]{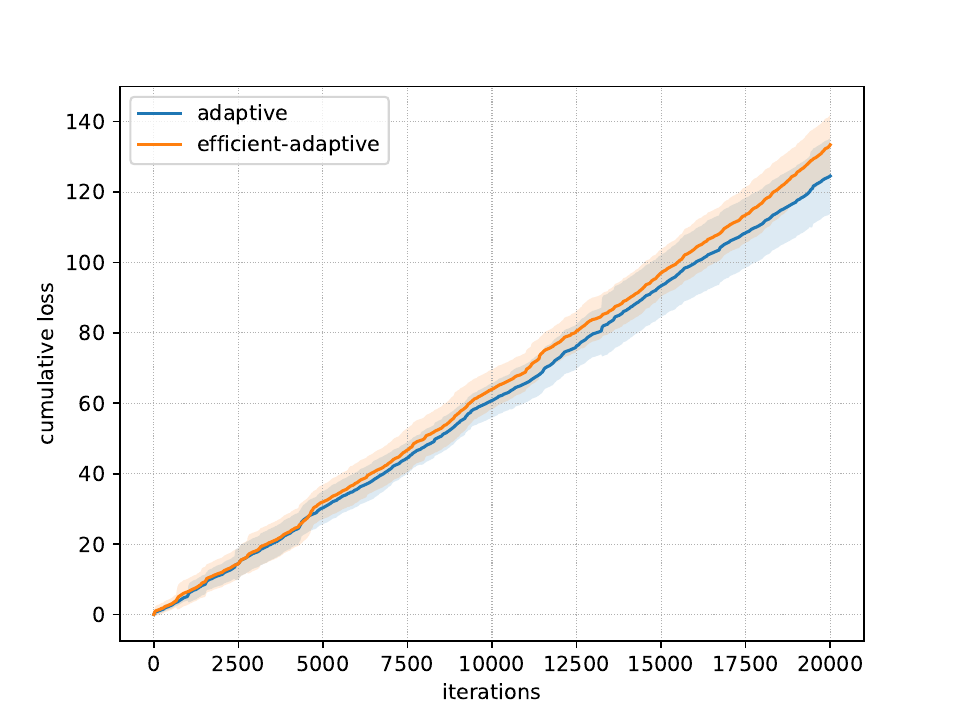}} \hspace{5mm}
  \subfigure[adaptive regret (time)]{ \label{fig:time-adaptive}
    \includegraphics[clip, trim=0.6cm 0.3cm 1.5cm 1.4cm,width=0.42\textwidth]{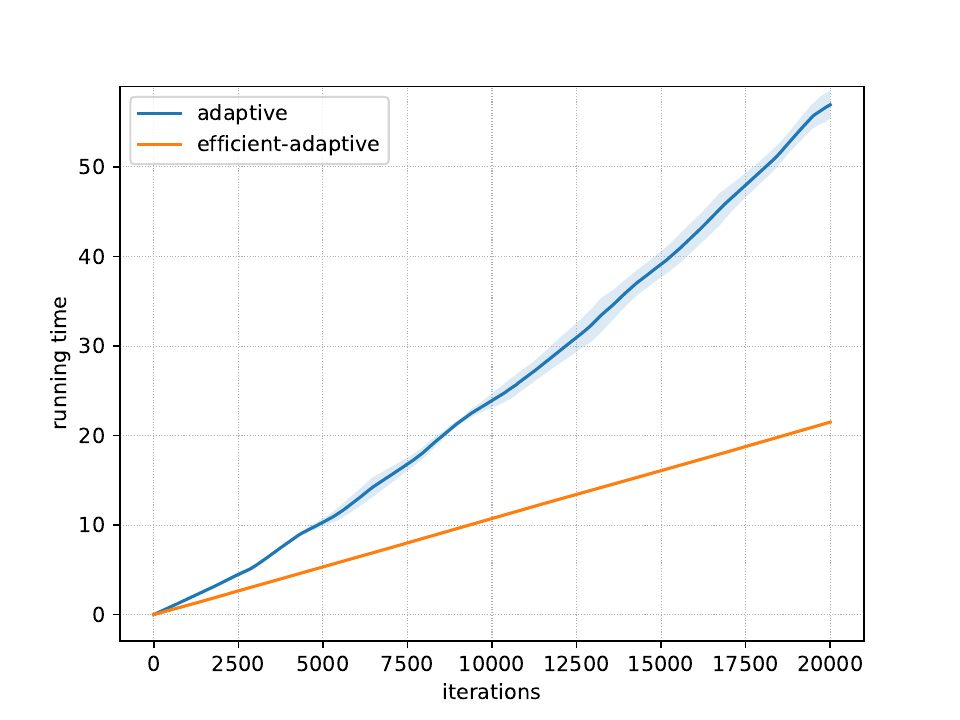}}
  \vspace{-3mm}
  \caption{Performance comparisons of existing methods and our methods (indicated by ``efficient'' prefix) in terms of cumulative loss and running time (in seconds). The first two figures plot the results of methods for dynamic regret minimization, while the latter ones are for adaptive regret.}
  \vspace{-4mm}
\label{fig:comparison}
\end{figure}

\paragraph{Results.} We repeat the experiments for five times with different random seeds and report the results (mean and standard deviation) in Figure~\ref{fig:comparison}. We use a machine with a single CPU (Intel(R) Core(TM) i9-10900K CPU @ 3.70GHz) and 32GB main memory to conduct the experiments. We plot both cumulative loss and running time (in seconds) for all the methods. We first examine the performance of dynamic regret minimization, see Figure~\ref{fig:loss-dynamic} for cumulative loss and see Figure~\ref{fig:time-dynamic} for running time. The empirical results show that our method has a comparable performance to Sword without much sacrifice of cumulative loss, while achieving about $10$ times speedup due to the improved projection complexity. Second, as shown in Figure~\ref{fig:loss-adaptive} and Figure~\ref{fig:time-adaptive}, a similar performance enhancement also appears in adaptive regret minimization, though the speedup is slightly smaller due to the fact that fewer learners are required to maintain for adaptive regret. To summarize, the empirical results show the effectiveness of our methods in retaining  regret performance and also the efficiency in terms of running time due to the reduced projection complexity.

\section{Conclusion}
\label{sec:conclusion}
In this paper, we design a generic reduction mechanism that can reduce the projection complexity of two-layer methods for non-stationary online learning, thereby approaching a clearer resolution of necessary computational overhead for robustness to non-stationarity. Building on the reduction mechanism, we develop a collection of online algorithms that optimize dynamic regret, adaptive regret, and interval dynamic regret. All the algorithms retain the best known regret guarantees, and more importantly, require a single projection onto the feasible domain per iteration. Notably, due to the requirement of our reduction, all our algorithms only perform one gradient query and one function evaluation at each round as well, making them particularly attractive in scenarios with limited feedback and a need for lightweight updates. Furthermore, we present two applications with light project complexity, including online non-stochastic control and online principal component analysis. Finally, our empirical studies also corroborate the theoretical findings.

One important open question remains regarding another type of problem-dependent bound that scales with gradient variation, defined as $V_T \triangleq \sum_{t=2}^T \sup_{\x \in \X} \norm{\nabla f_t(\x) - \nabla f_{t-1}(\x)}_2^2$~\citep{COLT'12:variation-Yang,ML'14:variation-Yang}. This bound plays a crucial role in achieving fast convergence in zero-sum games~\citep{NIPS'15:fast-rate-game,ICML'22:TV-game}. While~\citet{JMLR:sword++} introduced a two-layer method that attains gradient-variation dynamic regret with only one gradient and one function value query per iteration, developing a projection-efficient variant in this context remains an open challenge. As mentioned, the reduction requirement is necessary but not sufficient, and integrating optimistic online learning into our reduction mechanism is notably difficult due to the constrained feasible domain and the complexity of the two-layer structure. Another important problem is to understand the minimal computational overhead required for robustness to non-stationarity. For instance, is it possible to design a single-layer algorithm that achieves optimal dynamic regret, adaptive regret, or even interval dynamic regret, similar to the ensemble methods presented in this paper? Preliminary results suggest that a single-layer algorithm with a wavelet-based restarting scheme can achieve optimal dynamic regret in certain regimes, provided the online learner has access to noisy feedback about the comparators~\citep{ICML'24:wavelet}. However, its applicability to more general settings remains unclear. Extending this approach to broader scenarios would be valuable. Alternatively, information-theoretic arguments demonstrating the necessity of more complex structures would also be highly interesting.

\section*{Acknowledgment}
This work was supported by National Science and
Technology Major Project (2022ZD0114800) and NSFC
(U23A20382). Peng Zhao was supported in part by the Xiaomi Foundation. We are grateful for anonymous reviewers for their careful reviews and many constructive suggestions.

\bibliography{online_learning}
\appendix
\newpage
\section{Omitted Details for Reduction Mechanism}
\label{sec:appendix-property-surrogate}
In this section, we provide the proofs of Theorem~\ref{thm:surrogate-loss} and Lemma~\ref{lemma:gradient-compute}.

\subsection{Properties of Distance Function}
\label{sec:property-distance-function}
Before presenting the proofs, we here collect two useful lemmas regarding the distance function used in the surrogate loss, which will be useful in the following proofs. The proofs of the two lemmas can be found in the seminal paper of~\citet{COLT'18:black-box-reduction}.

\begin{myLemma}[{Proposition 1 of~\citet{COLT'18:black-box-reduction}}]
\label{lemma:distance-convex}
The distance function $S_\X(\y) = \inf_{\x \in \X} \norm{\y - \x}_2$ is convex and $1$-Lipschitz for any closed convex feasible domain $\X \subseteq \R^d$.
\end{myLemma}

\begin{myLemma}[{Theorem 4 of~\citet{COLT'18:black-box-reduction}}]
\label{lemma:distance-gradient}
Let $\X \subseteq \R^d$ be a closed convex set. Given $\y \in \R^d$ and $\y \notin \X$. Let $\x = \Pi_{\X}[\y]$. Then we have $\{\frac{\y - \x}{\norm{\y - \x}_2} \}= \partial S_{\X}(\y)$.
\end{myLemma}

\subsection{Proof of Theorem~\ref{thm:surrogate-loss}}
\label{sec:proof-surrogate-loss}
Theorem~\ref{thm:surrogate-loss} is originally due to~\citet{ICML'20:Ashok}, and for self-containedness we restate their proof using our notations.
\begin{proof}
When $\inner{\gradf_t(\x_t)}{\v_t} \geq 0$, by the definition of the surrogate loss defined in~\pref{eq:surrogate-loss}, we have $g_t(\y) = \inner{\gradf_t(\x_t)}{\y}$, which is linear in $\y$ and thus convex. It is clear that $\norm{\gradg_t(\y_t)}_2 = \norm{\gradf_t(\x_t)}_2$ and thus satisfies the claimed inequality of gradient norms in the statement. Moreover, the inequality~\eqref{eq:surrogate-upper-bound} holds evidently due to the linear surrogate loss in this case.

Let us focus on the case when $\inner{\gradf_t(\x_t)}{\v_t} < 0$. First, it can be verified that the surrogate loss $g_t(\y) = \inner{\gradf_t(\x_t)}{\y} - \inner{\gradf_t(\x_t)}{\v_t} \cdot S_{\X}(\y)$ is convex due to the convexity of $S_{\X}(\y)$ shown in Lemma~\ref{lemma:distance-convex} and the condition of $\inner{\gradf_t(\x_t)}{\v_t} < 0$ in this case. Next, the gradient of $g_t(\cdot)$ at the $\y_t$ point can be calculated according to Lemma~\ref{lemma:gradient-compute} as
\begin{equation*}
    \gradg_t(\y_t)  = \nabla f_t(\x_t) - \inner{\nabla f_t(\x_t)}{\v_t} \cdot \v_t
\end{equation*}
where $\v_t = (\y_t - \x_t)/\norm{\y_t - \x_t}_2$. Notice that $\norm{\v_t}_2 = 1$ and $\gradg_t(\y_t)$ is an orthogonal projection of $\gradf_t(\x_t)$ onto the subspace perpendicular to the vector $\v_t$, so we have $\norm{\gradg_t(\y_t)}_2 \leq \norm{\gradf_t(\x_t)}_2$. Finally, we proceed to prove the inequality~\eqref{eq:surrogate-upper-bound} in this case. Since the comparator $\u_t \in \X$ is in the feasible domain, we have $S_{\X}(\u_t) = \norm{\u_t - \u_t}_2 = 0$ and get
\begin{align*}
    &\inner{\gradf_t(\x_t)}{\x_t - \u_t}\\
    &= \inner{\gradf_t(\x_t)}{\y_t} + \inner{\gradf_t(\x_t)}{\x_t - \y_t} - \inner{\gradf_t(\x_t)}{\u_t}\\
    &= \inner{\gradf_t(\x_t)}{\y_t} -  \inner{\gradf_t(\x_t)}{\frac{\y_t - \x_t}{\norm{\y_t - \x_t}_2}} \cdot \norm{\y_t - \x_t}_2  - \inner{\gradf_t(\x_t)}{\u_t}\\
    &=  \inner{\gradf_t(\x_t)}{\y_t} - \inner{\gradf_t(\x_t)}{\v_t} \cdot  S_{\X}(\y_t) - \inner{\gradf_t(\x_t)}{\u_t} + \inner{\gradf_t(\x_t)}{\v_t} \cdot S_{\X}(\u_t)\\
    &= g_t(\y_t) - g_t(\u_t)\\
    &\leq \inner{\gradg_t(\y_t)}{\y_t - \u_t},
\end{align*}
where the last inequality holds owing to the convexity of the surrogate loss proven earlier.

Combining the two cases finishes the proof.
\end{proof}

\subsection{Proof of Lemma~\ref{lemma:gradient-compute}}
\label{sec:proof-gradient}
Lemma~\ref{lemma:gradient-compute} is originally due to~\citet{COLT'18:black-box-reduction}, and for self-containedness we restate their proof using our notations.
\begin{proof}
With a slight abuse of notations, we simply use $\gradg_t(\y)$ to denote the \mbox{(sub-)gradient} of surrogate function $g_t(\cdot)$ at point $\y$, no matter whether the function is differentiable.

When $\inner{\gradf_t(\x_t)}{\v_t} \geq 0$, the surrogate loss is $g_t(\y) = \inner{\gradf_t(\x_t)}{\y}$ by definition in~\pref{eq:surrogate-loss}. Therefore, the gradient simply becomes $\gradg_t(\y_t) = \gradf_t(\x_t)$.

When $\inner{\gradf_t(\x_t)}{\v_t} < 0$, the surrogate loss becomes $g_t(\y) = \inner{\gradf_t(\x_t)}{\y} - \inner{\gradf_t(\x_t)}{\v_t} \cdot S_{\X}(\y)$ according to definition in~\pref{eq:surrogate-loss}. By Lemma~\ref{lemma:distance-gradient}, the gradient  can be calculated by
\begin{equation*}
    \gradg_t(\y) = \gradf_t(\x_t) - \inner{\gradf_t(\x_t)}{\v_t} \cdot \frac{\y - \Pi_{\X}[\y]}{\norm{\y - \Pi_{\X}[\y]}_2},
\end{equation*}
where the computation needs the projection onto domain $\X$. In particular, for $\y_t$, we have
\begin{align*}
    \gradg_t(\y_t) = \gradf_t(\x_t) - \inner{\gradf_t(\x_t)}{\v_t} \cdot \frac{\y_t - \x_t}{\norm{\y_t - \x_t}_2} =\gradf_t(\x_t) - \inner{\gradf_t(\x_t)}{\v_t} \cdot \v_t.
\end{align*}
This ends the proof.
\end{proof}

\section{Omitted Details for Dynamic Regret Minimization}
\label{sec:appendix-proof-dynamic}
In this section, we provide the proofs for the theorems presented in Section~\ref{sec:reduction}. Specifically, we first prove the worst-case bound (Theorem~\ref{thm:dynamic-regret-project}) and then the small-loss bound (Theorem~\ref{thm:dynamic-regret-project-smooth}).

\subsection{Proof of Theorem~\ref{thm:dynamic-regret-project}}
\begin{proof}
Notice that~\citet{NIPS'18:Zhang-Ader} propose the improved Ader algorithm (see Algorithm~3 and Algorithm~4 in their paper), which uses the linearized loss as the input which allows the online algorithm to require one gradient and one function value. So the algorithm satisfies the requirements of our reduction mechanism, and our algorithm (Algorithm~\ref{alg:efficient-dynamic} with specifications in Theorem~\ref{thm:surrogate-loss}) can be regarded as the improved Ader equipped with the projection-efficient reduction. As a consequence, we can directly obtain the same dynamic regret guarantee and meanwhile ensure $1$ projection complexity by following the same proof of the improved Ader as well as the reduction guarantee (Theorem~\ref{thm:surrogate-loss}).
\end{proof}

\subsection{Proof of Theorem~\ref{thm:dynamic-regret-project-smooth}}

\begin{proof}
By the reduction guarantee shown in Theorem~\ref{thm:surrogate-loss}, we have the following result that  decomposes the dynamic regret into the two terms.
\begin{align}
    \sum_{t=1}^T f_t(\x_t) - \sum_{t=1}^T f_t(\u_t) &{} \leq \sum_{t=1}^T g_t(\y_t) - \sum_{t=1}^T g_t(\u_t) \leq \sum_{t=1}^T \inner{\nabla g_t(\y_t)}{\y_t - \u_t} \notag \\
&{}=\underbrace{\sum_{t=1}^T \inner{\nabla g_t(\y_t)}{\y_t - \y_{t,i}}}_{\meta} + \underbrace{\sum_{t=1}^T \inner{\nabla g_t(\y_t)}{\y_{t,i} - \u_t}}_{\base}, \label{eq:Sword-before-plug-dynamic-regret}
\end{align}
where in~\eqref{eq:Sword-before-plug-dynamic-regret} the first term is called \emph{meta-regret} as it measures the regret overhead of the meta-algorithm to track the unknown best base-learner, and the second term is called the \emph{base-regret} to denote the dynamic regret of the base-learner $i$. Note that the above decomposition holds for any base-learner index $i \in [N]$.

\paragraph{Upper bound of meta-regret.} Since the meta-algorithm is essentially FTRL with time-varying learning rates and a negative entropy regularizer, we apply Lemma~\ref{lemma:self-tuning-hedge} to obtain the meta-regret upper bound by choosing $\ell_{t, i} = \inner{\gradg_t(\y_t)}{\y_{t, i}}$ and obtain that
\begin{align}
    \sum_{t=1}^T \inner{\gradg_t(\y_t)}{\y_{t} - \y_{t,i}} & \leq 3\sqrt{\ln N \left(1 + \sum_{t=1}^T D^2 \norm{\gradg_t(\y_t)}_2^2\right) } + \frac{G^2D^2\sqrt{\ln N}}{2} \notag\\
    & \leq 3D\sqrt{\ln N   \sum_{t=1}^T \norm{\gradg_t(\y_t)}_2^2 } + \frac{(6+G^2D^2)\sqrt{\ln N}}{2}\notag\\
    & \leq  3D\sqrt{\ln N   \sum_{t=1}^T \norm{\gradf_t(\x_t)}_2^2 } + \O(1) \notag \\
    & \leq 6D\sqrt{L\ln N  \sum_{t=1}^T f_t(\x_t)} + \O(1) \label{eq:dynamic-meta-regret},
\end{align}
where the first inequality holds because we have $\norm{\ellb_t}_{\infty}^2 = \max_{i \in [N]} (\inner{\gradg_t(\y_t)}{\y_{t, i}})^2 \leq D^2 \norm{\gradg_t(\y_t)}_2^2$ by Cauchy-Schwarz inequality and $\norm{\gradg_t(\y_t)}_2 \leq \norm{\gradf_t(\x_t)}_2 \leq G $ by Theorem~\ref{thm:surrogate-loss}, and the last inequality is due to the self-bounding properties of smooth functions (see Lemma~\ref{lemma:self-bounded}). Note that $\O(\ln N) = \O(\log \log T)$ can be treated as a constant following previous studies~\citep{COLT'15:Luo-AdaNormalHedge,COLT'14:second-order-Hedge,JMLR:sword++}

\paragraph{Upper bound of base-regret.} According to Lemma~\ref{thm:ogd-dynamic-regret} and noticing that the comparator sequence $\u_1,\ldots,\u_T \in \X \subseteq \Y$ and the diameter of $\Y$ equals to $2D$ by definition, with slight modifications, we have the following dynamic regret bound.
\begin{align}
    \sum_{t=1}^T \inner{\gradg_t(\y_t)}{\y_{t, i} - \u_t} &{}\leq \frac{5D^2}{2\eta_i} + \frac{D}{\eta_i}\sum_{t=2}^T \norm{\u_{t} - \u_{t-1}}_2 + \eta_i \sum_{t=1}^T \norm{\gradg_t(\y_t)}_2^2 \notag \\
& {} \leq \frac{5D^2}{2\eta_i} + \frac{D}{\eta_i}\sum_{t=2}^T \norm{\u_{t} - \u_{t-1}}_2 + 4\eta_iL\sum_{t=1}^T f_t(\x_t) \notag ,
\end{align}
where the last inequality uses the property of $\norm{\nabla g_t(\y_t)}_2 \leq \norm{\nabla f_t(\x_t)}_2$ in Theorem~\ref{thm:surrogate-loss} and the self-bounding property of smooth functions (see Lemma~\ref{lemma:self-bounded}).

\paragraph{Upper bound of dynamic regret.} Combining the upper bounds meta-regret and base-regret together  yields the following dynamic regret:
\begin{equation}
    \label{eq:Sword-dynamic-regret-f_t}
    \sum_{t=1}^T f_t(\x_t) - \sum_{t=1}^T f_t(\u_t) \leq 6D\sqrt{L\ln N \sum_{t=1}^T f_t(\x_t)} + \frac{5D^2 + 2DP_T}{2\eta_i} +  4\eta_iL \sum_{t=1}^T f_t(\x_t) + \O(1),
\end{equation}
which holds for any base-learner's index $i \in [N]$.

Next, we specify the base-learner $\Ecal_i$ to be compared with. Indeed, we aim at choosing the one with step size closest to the (near-)optimal step size $\eta^{\star} = \sqrt{\frac{5D^2 + 2DP_T}{1 + 8LF_T^\x}}$, where we denote by $F_T ^\x= \sum_{t=1}^T f_t(\x_t)$ the cumulative loss of the decisions. By Assumption~\ref{assumption:bounded-gradient} and Assumption~\ref{assumption:bounded-domain}, we have $F_T^\x \in [0, GDT]$ and then the possible minimum optimal and maximum step size are
\begin{equation*}
    \eta_{\min} = \sqrt{\frac{5D^2}{1 + 8LGDT}}, \mbox{ and } \eta_{\max} = \sqrt{5D^2 + 2D^2 T}.
\end{equation*}

The construction of step size pool is by discretizing the interval $[\eta_{\min}, \eta_{\max}]$ with intervals with exponentially increasing length. The step size of each base-learner is designed to be monotonically increasing with respect to the index. Consequently, it is evident to verify that there exists an index $i^{\star}\in[N]$ such that $\eta_{i^{\star}} \leq \eta^{\star} \leq \eta_{i^{\star}+1} = 2\eta_{i^{\star}}$. As the upper bounds of meta-regret and base-regret hold for any compared base-learner, we can choose the index as $i^{\star}$ in particular. Then the second and the third terms in the inequality~\eqref{eq:Sword-dynamic-regret-f_t} satisfy
\begin{align}
    &\frac{5D^2 + 2DP_T}{2\eta_{i^{\star}}} + 4\eta_{i^{\star}} L F_T^\x  \notag\\
    & \leq \frac{5D^2 + 2DP_T}{\eta^{\star}} + 4\eta^{\star}L F_T^\x \notag\\
    &\leq \sqrt{(5D^2 + 2DP_T)(1 + 8LF_T^\x)}+ \frac{1}{2}\sqrt{(5D^2 + 2DP_T)(1 + 8LF_T^\x)} \notag\\
    &\leq 3\sqrt{2(5D^2 + 2DP_T)(1 + LF_T^\x)}. \label{eq:Sword-base-tuning}
\end{align}

Substituting inequality~\eqref{eq:Sword-base-tuning} into inequality~\eqref{eq:Sword-dynamic-regret-f_t}, we have,
\begin{align*}
    & \sum_{t=1}^T f_t(\x_t) - \sum_{t=1}^T f_t(\u_t) \\
    &\leq 6D\sqrt{L\ln N F_T^\x}  +   3\sqrt{2(5D^2 + 2DP_T)(1 + LF_T^\x)} + \O(1)\\
    &\leq \left(6D\sqrt{L\ln N} + 3\sqrt{2L(5D^2 + 2DP_T)}\right)\sqrt{F_T^\x} + 3\sqrt{2(5D^2 + 2DP_T)} + \O(1)\\
    & \leq \O\Big(\sqrt{(1 + P_T) (F^{\u}_T + \sqrt{P_T} + \O(1))} + P_T + 1\Big)\\
    & = \O\Big(\sqrt{(F^{\u}_T+P_T)(1+P_T)}\Big),
\end{align*}
where the last inequality holds by Lemma~\ref{lemma:substitute-F_T}. Hence, we complete the proof of Theorem~\ref{thm:dynamic-regret-project-smooth}.
\end{proof}

\section{Omitted Details for Adaptive Regret Minimization}
\label{sec:appendix-fpd-adaptive}
In this section, we present omitted details for minimizing the worst-case and small-loss adaptive regret bounds with a focus on proving the main theorem of small-loss bound, i.e., Theorem~\ref{thm:small-loss-adaptive}. Appendix~\ref{sec:key-lemma-adaptive-regret} provides key lemmas, Appendix~\ref{sec:proof-small-loss-adaptive} presents the proof of Theorem~\ref{thm:small-loss-adaptive}, and Appendix~\ref{sec:appendix-adaptive-lemma1} -- Appendix~\ref{sec:appendix-adaptive-lemma3} give the proofs of these key lemmas.

\subsection{Key Lemmas}
\label{sec:key-lemma-adaptive-regret}
In this part, we present three key lemmas for proving Theorem~\ref{thm:small-loss-adaptive}, based on which we prove Theorem~\ref{thm:small-loss-adaptive} in Appendix~\ref{sec:proof-small-loss-adaptive}. We will prove those lemmas in the following several subsections.

The first lemma gives the second-order regret bound for the meta-algorithm (Adapt-ML-Prod)~\citep{COLT'14:second-order-Hedge}, which plays a crucial role in applying our reduction. Though Adapt-ML-Prod can be applied to the sleeping-expert setting directly, we need more careful analysis to obtain the \emph{fully} small-loss adaptive regret bound, otherwise the direct reduction of results from \citet{COLT'14:second-order-Hedge} will incur an undesired $\O(\log T)$ factor.
\begin{myLemma}
    \label{lemma:pd-adaptive-surrogate-meta-regret}
    Under Assumptions~\ref{assumption:bounded-domain} and~\ref{assumption:smoothness}, for any interval $I=[i, j] \in \tilde{\mathcal{C}}$ in the geometric covers defined in~\pref{eq:covering} at the beginning of which we suppose the $m$-th base-learner is initialized, Algorithm~\ref{alg:problem-dependent-adaptive}  ensures
    \begin{align*}
        &\sum_{\tau = i}^t g_\tau(\y_\tau) - g_\tau(\y_{\tau,m}) \leq \sum_{\tau = i}^t \inner{\gradg_\tau(\y_\tau)}{\y_\tau - \y_{\tau, m}}\\
        &\leq 4D\bigg(3\sqrt{\ln(1+2m)} + \frac{\mu_t}{\sqrt{\ln(1+2m)}}\bigg)\sqrt{L\sum_{\tau = i}^t f_\tau(\x_\tau) } + 18GD\ln(1+2m) + 6GD\mu_t\\
        &=  \O\left(\sqrt{\log(m) \sum_{\tau=i}^t f_\tau(\x_\tau)} + \log (m)\right),
    \end{align*}
    where we denote $\mu_t = \ln(1 + (1 + \ln(1+t))/(2e))$.
\end{myLemma}

Combining the above lemma with the regret bound for base-learners, we can obtain the adaptive regret for any interval in the geometric covering intervals $\tilde{\mathcal{C}}$ defined in~\pref{eq:covering}.
\begin{myLemma}
    \label{lemma:pd-adaptive-covering-regret}
    Under Assumptions~\ref{assumption:bounded-gradient},~\ref{assumption:bounded-domain}, and~\ref{assumption:smoothness}, for any interval $[i, j] \in \tilde{\mathcal{C}}$ in the geometric covering intervals defined in~\pref{eq:covering}, at the beginning of which we assume the $m$-th base-learner is initialized, Algorithm~\ref{alg:problem-dependent-adaptive}  ensures for any time $t \in [i, j]$ and any comparator $\u \in \X$,
    \begin{align}
        &\sum_{\tau=i}^t f_\tau(\x_\tau) - \sum_{\tau = i}^t f_\tau(\u)  \notag \\
        {}& \leq  4D\left(3\sqrt{\ln(1+2m)} + \mu_t + 2\right)\sqrt{L\sum_{\tau = i}^t f_\tau(\x_\tau)} + 18GD\ln(1+2m) + 6GD\mu_t + 4D\sqrt{\delta} \notag \\
        &=\O\left(\sqrt{\log(m) \sum_{\tau=i}^t f_{\tau}(\x_\tau) }+ \log (m)\right), \notag
    \end{align}
    where we denote $\mu_t = \ln(1 + (1 + \ln(1+t))/(2e))$.

    Also, Algorithm~\ref{alg:problem-dependent-adaptive} ensures the bound in terms of the cumulative loss of the comparator,
    \begin{align}
        &\sum_{\tau=i}^t f_\tau(\x_\tau) - \sum_{\tau = i}^t f_\tau(\u) \notag\\
        {} & \leq 4D\bigg(\sqrt{\ln(1+2m)} +\mu_t + 2\bigg)\sqrt{L\sum_{\tau=i}^t f_\tau(\u)} \notag\\
        &{} \qquad +(27GD+72D^2L)\ln(1+2m) + 72D^2L\mu^2_t + 9GD\mu_t + 6D\sqrt{\delta} + 288D^2L  \label{eq:f_tu-bound}\\
        &=\O\left(\sqrt{\log(m) \sum_{\tau=i}^t f_{\tau}(\u) }+ \log (m)\right). \notag
    \end{align}
\end{myLemma}
It is worth emphasizing that, the regret bound in terms of the cumulative loss of final decisions $\{\x_t\}$ plays a key role in proving the worst-case adaptive regret bound in Theorem~\ref{thm:small-loss-adaptive}.

The next lemma states that by the smooth and non-negative nature of loss functions, we can estimate the cumulative loss of any comparator $\u \in \X$ by the markers maintained by the problem-dependent scheduling.
\begin{myLemma}
    \label{lemma:threshold-lower-bound}
    Under Assumptions~\ref{assumption:bounded-gradient},~\ref{assumption:bounded-domain}, and~\ref{assumption:smoothness}, for any interval $[s_m, s_{m+1} -1]$ determined by two consecutive intervals $s_m$ and $s_{m+1}$, where we denote by $s_m$ the $m$-th marker, Algorithm~\ref{alg:problem-dependent-adaptive} ensures that for any comparator $\u \in \X$,
    \begin{align}
    \label{eq:estimate-lower-bound}
        \sum_{t=s_m}^{s_{m+1}-1} f_t(\u) \geq \frac{1}{4}C_m,
    \end{align}
    where $C_m = \mathcal{G}(m)$ is the $m$-th threshold with the threshold function $\mathcal{G}(\cdot)$ defined at~\pref{eq:generate-function-adaptive}.
\end{myLemma}

The above lemmas rely on the unknown variable of $m$, which represents the number of base-learners initialized till time stamp $t$. The following lemma shows that $m$ is of the same order with the cumulative loss $\sum_{\tau=1}^{t} f_\tau(\u)$ of any comparator $\u \in \X$, owing to the construction of the problem-dependent covering intervals.
\begin{myLemma}
    \label{lemma:pd-m-upper-bound}
    Under Assumptions~\ref{assumption:bounded-gradient},~\ref{assumption:bounded-domain}, and~\ref{assumption:smoothness}, for any interval $[i,j] \in \tilde{\mathcal{C}}$ and any $t \in [i, j]$, the variable $m$ specified in Lemma~\ref{lemma:pd-adaptive-surrogate-meta-regret} and Lemma~\ref{lemma:pd-adaptive-covering-regret} can be bounded by
    \begin{equation*}
        m \leq 1 + \frac{4}{C_1} \min_{\u \in \X} \sum_{\tau=1}^{t} f_\tau(\u) =  \O\left(F_{[1, t]}\right),
    \end{equation*}
    where $C_1$ is a constant calculated by the threshold function as $C_1 = \mathcal{G}(1)$ defined in~\pref{eq:generate-function-adaptive}.

    Moreover, for any $t \in [i, j]$ Algorithm~\ref{alg:problem-dependent-adaptive} ensures,
    \begin{equation*}
        \sum_{\tau=i}^t f_\tau(\x_\tau) - \mathop{\min}_{\u \in \X}\sum_{\tau = i}^t f_\tau(\u) \leq \alpha(t) + \beta(t)\sqrt{\mathop{\min}_{\u \in \X} \sum_{\tau=i}^t f_\tau(\u)} = \O\left(\sqrt{F_{[i, t]}\log F_{[1,t]}}\right),
    \end{equation*}
    where
    \begin{align*}
        \alpha(t) &{}= (27GD+72D^2L)\ln \left(3+ \frac{8}{C_1}\sum_{\tau=1}^t f_\tau(\u)\right) + 72D^2L\mu_t^2 + 9GD\mu_t + 6D\sqrt{\delta} + 288D^2L,\\
        \beta(t) &{} = 4D\sqrt{L}\left(\sqrt{\ln\left(3+ \frac{8}{C_1}\sum_{\tau=1}^t f_\tau(\u)\right)} +\mu_t + 2\right),
    \end{align*}
    $\mu_t = \ln(1 + (1 + \ln(1+t))/(2e))$, and $F_{[a, b]} = \mathop{\min}_{\u \in \X} \sum_{\tau = a}^b f_\tau (\u)$ denotes the cumulative loss of the comparator within the interval $[a,b] \subseteq [T]$.
\end{myLemma}

\subsection{Proof of Theorem~\ref{thm:small-loss-adaptive}}
\label{sec:proof-small-loss-adaptive}
\begin{proof}
The statement in Theorem~\ref{thm:small-loss-adaptive} consists of two parts, including a small-loss bound of $\O(\sqrt{F_I \log F_I \log F_T})$ and a worst-case bound of $\O(\sqrt{\abs{I} \log T})$. Below, we present the proofs of the two bounds respectively.

\paragraph{Small-loss regret bound.}
Given any interval $[r, s] \subseteq [T] $, we will identify a series of intervals $I_1, \dots, I_v $ in the schedule $\tilde{\mathcal{C}}$ that almost covers the entire interval $[r, s]$. Then, we can use Lemma~\ref{lemma:pd-m-upper-bound} to ensure low regret over these intervals. By further demonstrating that the regret on the uncovered interval can be well-controlled and that the number of intervals $v$ is not large, we can ultimately achieve the desired bound. Below, we provide formal proof.

Recall in the algorithmic procedures, the algorithm will register a series of markers $s_1, s_2, \ldots$. Let $s_p$ be the smallest marker that is larger than $r$, and let $s_q$ be the largest marker that is not larger than $s$. As a result, we have $s_{p-1} \leq r < s_p$, and $s_q \leq  s < s_{q+1}$.

We bound the regret over the interval $[r, s_p - 1]$ that is not covered by the schedule as
\begin{equation}
    \sum_{t = r}^{s_p - 1} f_t(\x_t) - \sum_{t = r}^{s_p - 1} f_t(\u) \leq \sum_{t = r}^{s_p - 1} f_t(\x_t) \leq \sum_{t = s_{p-1} }^{s_p - 1} f_t(\x_t) \leq C_{p-1} + GD. \label{eq:pd-adaptive-r-s_p}
\end{equation}
The first and the second inequalities are by the non-negative property of loss functions. The last inequality is due to $\sum_{t = s_{p-1} }^{s_p - 2} f_t(\x_t) \leq C_{p-1}$, which is determined by the threshold mechanism, and the fact that $ f_t(\x_t) \leq GD$ by Assumptions~\ref{assumption:bounded-gradient} to~\ref{assumption:smoothness}.

Next, we focus on the interval $[s_p, s]$, which can be covered by geometric covering intervals. By Lemma~\ref{lemma:pcgc-number}, we can find $v$ consecutive intervals
\begin{align}
    I_1=[s_{i_1}, s_{i_2} - 1],\  I_2 = [s_{i_2}, s_{i_3} - 1],\  \dots , I_v = [s_{i_v}, s_{i_{v+1}} - 1] \in \tilde{\mathcal{C}}, \label{eq:pd-adaptive-gc-intervals}
\end{align}
such that $i_1 = p,\  i_v\leq q < i_{v+1}, \text{and } v \leq \lceil \log_2(q-p+2) \rceil$. Also notice that,
\begin{align*}
    q < i_{v+1} \Rightarrow q + 1 \leq i_{v+1} \Rightarrow s_{q+1} - 1 \leq s_{i_{v+1}} - 1 \Rightarrow s \leq s_{i_{v+1}} - 1.
\end{align*}

By Lemma~\ref{lemma:pd-m-upper-bound}, our algorithm has anytime regret bounds on intervals $I_1$ to $I_v$, since they belong to the covering intervals $\tilde{\mathcal{C}}$,
\begin{align}
    \sum_{t = s_p}^{s} f_t(\x_t) - \sum_{t = s_p}^{s} f_t(\u)
    &= \sum_{k=1}^{v-1} \sum_{t \in I_k}\left(f_t(\x_t) - f_t(\u)\right) + \sum_{t \in [s_{i_v}, s]}\left(f_t(\x_t) - f_t(\u)\right)\notag \\
    &{}\leq \sum_{k=1}^{v-1} \left(\alpha(s_{i_{k+1}}-1) + \beta(s_{i_{k+1}}-1)\sqrt{F_{I_k}}\right) + \alpha(s) + \beta(s)\sqrt{F_{[s_{i_v}, s]}} \notag \\
     &{}\leq \sum_{k=1}^{v-1} \left(\alpha(s) + \beta(s)\sqrt{F_{I_k}}\right) + \alpha(s) + \beta(s)\sqrt{F_{[s_{i_v}, s]}} \label{eq:adaptive-small-loss-before-cauchy} \\
    &{} \leq v\alpha(s) + \beta(s)\sqrt{vF_{[s_p, s]}} \notag \\
    &{}\leq v\alpha(s) + \beta(s)\sqrt{vF_{I}}, \label{eq:pd-adaptive-s_p-s}
\end{align}
where the second inequality is because $\alpha(\cdot), \beta(\cdot)$ are monotonically increasing, and the last inequality is by the non-negativity of loss functions.

Combining~\eqref{eq:pd-adaptive-r-s_p} with~\eqref{eq:pd-adaptive-s_p-s}, the adaptive regret on any interval $[r,s]$ will be
\begin{align}
    \sum_{t = r}^{s} f_t(\x_t) - \sum_{t = r}^{s} f_t(\u) &= \sum_{t = r}^{s_p-1} \left(f_t(\x_t) -  f_t(\u)\right) + \sum_{t = s_p}^{s} \left(f_t(\x_t) - f_t(\u)\right)\notag \\
     &\leq C_{p-1} + GD + v\alpha(s) + \beta(s)\sqrt{vF_{I}} .\label{eq:small-loss-adaptive-last-step}
\end{align}

Furthermore, we show that $C_{p-1}$ and $v$ are of order $\O(\log F_T)$ and $\O(\log F_{[r,s]})$ respectively. By the definition of the time-varying threshold (see the threshold generating function~\pref{eq:generate-function-adaptive}) and the upper bound of $m$ in Lemma~\ref{lemma:pd-m-upper-bound}, the threshold can be bounded as,
\begin{equation*}
    C_{p-1} \leq (54GD+168D^2L)\ln\left(3+ \frac{8}{C_1}F_{[1, r]}\right) + 168D^2L\mu_T^2 + 18GD\mu_T + 6D\sqrt{\delta} + 672D^2L,
\end{equation*}
with $C_1$ and $\mu_T$ defined in Lemma~\ref{lemma:pd-m-upper-bound}. Notice that, we treat $\mu_T = \O(\log \log T)$ as a constant following previous studies~\citep{COLT'15:Luo-AdaNormalHedge,COLT'14:second-order-Hedge,JMLR:sword++}.

Through $[s_p, s_q - 1]$, the algorithm registers $q-p$ markers, i.e., $s_p, s_{p+1}, \dots, s_{q-1}$, then by Lemma~\ref{lemma:threshold-lower-bound} we can lower bound the cumulative loss of comparator $\u^\star \triangleq \argmin_{\u \in \X} \sum_{t=r}^s f_t(\u)$ by the corresponding thresholds,
\begin{align*}
    \sum_{t=s_p}^{s_q - 1}f_t(\u^\star) = \sum_{i=p}^{q-1} \sum_{t \in [s_i, s_{i+1} - 1]}f_t(\u^\star) \geq \frac{1}{4}\sum_{i=p}^{q-1} C_i \geq \frac{C_1}{4}(q - p),
\end{align*}
where the the last inequality is because the thresholds are monotonically increasing. The above inequality immediately implies that $ q - p \leq \frac{4}{C_1}\sum_{t=s_p}^{s_q - 1}f_t(\u^\star)\leq \frac{4}{C_1}\sum_{t=r}^{s}f_t(\u^\star)$ by the non-negativity of loss functions. We estimate the number of intervals $v$ with Lemma~\ref{lemma:pcgc-number},
\begin{align}
    v \leq \lceil \log_2(q-p+2) \rceil \leq \left\lceil \log_2\left( \frac{4}{C_1}F_{[r,s]}+2 \right) \right\rceil = \O(\log F_{[r,s]}). \label{eq:v-upper-bound}
\end{align}

Combining the upper bounds of $C_{p-1}$ and $v$, the adaptive regret bound in~\eqref{eq:small-loss-adaptive-last-step} as well as the definition of $\alpha(\cdot), \beta(\cdot)$ in Lemma~\ref{lemma:pd-m-upper-bound} yields,
\begin{align*}
    \sum_{t = r}^{s} f_t(\x_t) - \sum_{t = r}^{s} f_t(\u) &\leq C_{p-1} + v\alpha(s) + \beta(s)\sqrt{vF_{I}} + GD\\
    &\leq \O\left(\log F_T\right) + \O\left(\log F_I \log F_T\right) +  \O\left(\sqrt{F_I\log F_I \log F_T}\right) + \O\left(1\right)\\
    & = \O\left(\sqrt{F_I \log F_I \log F_T}\right),
\end{align*}
where the last step is true as we follow the same convention in~\citep{ICML19:Zhang-Adaptive-Smooth} to treat the $\O(\log F_I \log F_T)$ as the non-leading term. We finish the proof for small-loss regret bound.

\paragraph{Worst-case regret bound.} The above proof aims at obtaining small-loss type regret bound, and one of the key steps is to use Cauchy-Schwarz inequality to bound~\eqref{eq:adaptive-small-loss-before-cauchy}, which results in an additional $\O(\sqrt{\log F_{[r,s]}})$ term. Next, we show that asymptotically this extra term can be avoided thanks to the new-designed thresholds mechanism. Thus, our algorithm achieves the same worst-case adaptive regret as the best known result~\citep{AISTATS'17:coin-betting-adaptive}.

The primary insight for this proof lies in employing the summation for a geometric series instead of using the Cauchy-Schwarz inequality to combine the regret bounds on the intervals $I_1, \dots, I_v$. The crucial ingredient is to demonstrate that the worst-case regret bound on each interval scales with the number of markers within it, and that the number of these markers within each interval constitutes a geometric series.

From Lemma~\ref{lemma:pd-adaptive-covering-regret}, we have that for any interval $I = [i, j]$ in problem-dependent schedule defined in~\eqref{eq:covering}, the adaptive regret is at most
\begin{equation}
    \sum_{t=i}^j f_t(\x_t) - \sum_{t=i}^j f_t(\u) \leq \O\left(\sqrt{\log T \cdot F^{\x}_{[i, j]}} + \log T \right),\label{eq:worst-case-Fx}
\end{equation}
where we use the notation $F^{\x}_{[a, b]} = \sum_{t=a}^b f_t(\x_t)$ to denote the cumulative loss of the final decisions within the interval $[a,b] \subseteq [T]$, and we apply Lemma~\ref{lemma:pd-m-upper-bound} to upper bound $m \leq \O(T)$ as only the worst-case behavior matters now.

For the intervals $I_k = [s_{i_k}, s_{i_{k+1}}-1], k \in [v]$ defined in~\eqref{eq:pd-adaptive-gc-intervals}, we have the following facts:
\begin{align}
    i_{k+1} \leq 2 \cdot i_{k},\ \forall k \in [v], \text{ and } |i_{k+1} - i_{k}| \leq \frac{1}{2} |i_{k+2} - i_{k+1}|,\ \forall k \in [v-1].\label{eq:geometric-series}
\end{align}
The first inequality above, which can verified by the construction of cover defined in~\eqref{eq:covering}, is used to show that the time-varying thresholds do not grow too fast. The second inequality, which can be found in the proof of~\citep[Lemma 11]{ICML19:Zhang-Adaptive-Smooth}, indicates that the number of markers within each interval decreases exponentially from $I_{v}$ to $I_1$.

For any interval $I_k$ with $ k \in [v-1]$ in~\eqref{eq:pd-adaptive-gc-intervals}, our algorithm's cumulative loss within the interval can be upper bounded as
\begin{align}
    \sum_{t = s_{i_{k}}}^{s_{i_{k+1}} - 1} f_t(\x_t) &= \sum_{a = i_k}^{i_{k+1}-1}\left(\Bigg(\sum_{t \in [s_{a}, s_{a+1}-2]}f_t(\x_t)\Bigg) + f_{s_{a+1}-1}(\x_{s_{a+1}-1})\right) \notag \\
    &\leq \bigg(\sum_{a=i_{k}}^{i_{k+1} - 1} C_a \bigg) + GD |i_{k+1} - i_k| \leq (GD+C_{i_{k+1} - 1}) |i_{k+1} - i_{k}|. \label{eq:pd-adaptive-relation}
\end{align}
where the first inequality is by the threshold mechanism and the fact that $f_t(\x) \in [0, GD]$.

We then split a given interval $[r,s]$ into three parts to analyze, namely, the consecutive $v-1$ intervals $I_1$ to $I_{v-1}$, interval $[r, s_p-1]$, and $[s_{i_v}, s]$, where notably the last two intervals are not fully covered by any interval in geometric covering intervals,
\begin{align}
    \sum_{t=r}^s f_t(\x_t) -  \sum_{t=r}^s f_t(\u) &= \sum_{t=r}^{s_p-1} f_t(\x_t) - f_t(\u) + \sum_{t=s_p}^{s_{i_v}-1} f_t(\x_t) -  f_t(\u) + \sum_{t=s_{i_v}}^s f_t(\x_t) -  f_t(\u) \notag \\
    &=\underbrace{\sum_{t=r}^{s_p-1} f_t(\x_t) - f_t(\u)}_{\texttt{term-(a)}} + \underbrace{\sum_{k=1}^{v-1}\sum_{t\in I_k} f_t(\x_t) -  f_t(\u)}_{\texttt{term-(b)}} + \underbrace{\sum_{t=s_{i_v}}^s f_t(\x_t) -  f_t(\u)}_{\texttt{term-(c)}}. \label{eq:worst-case-term}
\end{align}
We analyze \texttt{term-(b)} first, since it is the most intricate part in this proof. From interval $I_1$ to $I_{v-1}$, beginning with~\pref{eq:worst-case-Fx} we have
\begin{align}
    \texttt{term-(b)} &\leq \sum_{k=1}^{v-1} \O\left(\sqrt{\log T \cdot F^{\x}_{I_k}} + \log T\right) \notag \\
    &\leq \sum_{k=1}^{v-1} \O\left(\sqrt{\log T \cdot C_{i_{v} - 1} \cdot \abs{i_{k+1} - i_k}} + \log T\right)\notag\\
    &\leq \sum_{k=1}^{v-1}\O\left(\sqrt{\log T \cdot C_{i_{v} - 1} \cdot \frac{|i_{v} - i_{v-1}|}{2^{v-1 - k}}} + \log T\right)\notag\\
    &\leq \O\left(v \log T + \sqrt{\log T \cdot C_{i_v - 1} } \cdot \sum_{b=0}^{+\infty} \sqrt{\frac{\abs{i_{v} - i_{v-1}}}{2^b}} \right)\notag\\
    &\leq \O\left(v\log T + \sqrt{\log T\cdot C_{i_v - 1} \cdot \abs{i_{v} - i_{v-1}}}\right) \label{eq:worst-case-C-i},
\end{align}
where the second inequality is by~\eqref{eq:pd-adaptive-relation} and together with the monotonically increasing property of thresholds, the third inequality is by the second inequality listed in~\eqref{eq:geometric-series}, and the last inequality is by the summation of geometric sequence. We emphasize that the second inequality is upheld due to the newly-designed problem-dependent schedule mechanism. This mechanism, which monitors the cumulative loss of final decisions $\{f_t(\x_t)\}$, enables us to associate the $F_{I_k}^\x$ factor with the number of markers $\abs{i_{k+1} - i_k}$ and further to apply the summation of geometric series.

In the subsequent analysis, our objective is to demonstrate that $C_{i_v - 1} \cdot |i_v - i_{v-1}| = \O(\abs{I})$. Employing the mechanism of the time-varying threshold as defined in~\pref{eq:generate-function-adaptive}, we have
\begin{align*}
    C_{i_v-1} = \mathcal{G}(i_v - 1) \leq \O(\log(i_v)).
\end{align*}
Moreover, since $|i_v - i_{v-1}|$ denotes the number of markers generated by our algorithm during the interval $I_{v-1}$, it can be bounded above by
\begin{equation*}
|i_v - i_{v-1}| \leq \frac{GD|I|}{C_{i_{v-1}}} = \O\left(\frac{GD|I|}{\log (i_{v-1})}\right).
\end{equation*}
This stems from the fact that the cumulative loss of the algorithm over $I$ does not exceed $GD|I|$, and we leverage the smallest threshold $C_{i_{v-1}}$ during the interval $I_{v-1}$ to determine the upper limit on the number of markers.

Plugging the upper bounds for $C_{i_v-1}$ and $|i_v - i_{v-1}|$ into~\pref{eq:worst-case-C-i}, we have
\begin{align*}
    \texttt{term-(b)} &\leq \O\left(v\log T + \sqrt{\log T\cdot C_{i_v - 1} \cdot |i_v - i_{v-1}|}\right) \\
    &\leq \O\left(v\log T + \sqrt{\log T\cdot \log(i_v) \cdot \frac{|I|}{\log (i_{v-1})}}\right) \\
    &\leq \O\left(v\log T + \sqrt{\log T \cdot |I| (1 + \frac{1}{\log i_{v-1}})}\right)\\
    &\leq \O\left(\log |I| \log T + \sqrt{|I| \log T}\right).
\end{align*}
The third inequality follows from the first inequality presented in~\eqref{eq:geometric-series}. The concluding inequality is by the fact that $v \leq \O(\log F_{[r,s]}) \leq \O(\log |I|)$ as proved in~\eqref{eq:v-upper-bound}. This is true because the variable $v$ is introduced in our analysis by Lemma~\ref{lemma:pcgc-number}, which is independent of the worst-case analysis.

As shown in~\pref{eq:pd-adaptive-r-s_p}, we can upper bound the \texttt{term-(a)} as,
\begin{equation*}
    \texttt{term-(a)} \leq C_{p-1} \leq \O(\log T).
\end{equation*}

Using again Lemma~\ref{lemma:pd-m-upper-bound}, the \texttt{term-(c)} is bounded as,
\begin{equation*}
    \texttt{term-(c)} \leq \O\left(\log T + \sqrt{F_{[s_{i_v}, s]} \log T }\right) \leq \O\left(\log T + \sqrt{|I| \log T}\right).
\end{equation*}

Now we are ready to derive the worst-case adaptive regret by plugging the upper bounds from \texttt{term-(a)} to \texttt{term-(c)} into~\pref{eq:worst-case-term},
\begin{align*}
    & \sum_{t = r}^{s} f_t(\x_t) - \sum_{t = r}^{s} f_t(\u) \\
    &=\sum_{t=r}^{s_p-1} f_t(\x_t) - f_t(\u) + \sum_{k=1}^{v-1}\sum_{t\in I_k} f_t(\x_t) -  f_t(\u) + \sum_{t=s_{i_v}}^s f_t(\x_t) -  f_t(\u)\\
    &\leq \O\left(\log T\right) + \O\left(\log |I| \log T + \sqrt{|I| \log T}\right)+\O\left(\log T + \sqrt{|I| \log T}\right)\\
    &= \O\left(\sqrt{|I|\log T} + \log |I| \log T\right) = \O\left(\sqrt{\left(|I| + \log T \cdot \log^2 |I|\right) \log T} \right) = \O(\sqrt{|I| \log T}).
\end{align*}
The last step holds by considering the following cases.
\begin{itemize}
    \item When the interval length is $|I| = \Theta(T^\alpha)$ with $\alpha \in (0, 1]$. Then,
    \begin{align*}
        & \O\left(\sqrt{\left(|I| + \log T \cdot \log^2 |I|\right) \log T} \right) \\
        & = \O\left(\sqrt{\left(T^\alpha + \alpha^2\log^3 T \right) \log T} \right) \\
        & = \O\left(\sqrt{T^\alpha \log T} \right)= \O\big(\sqrt{|I| \log T} \big).
    \end{align*}
    \item When the interval length is $|I| = \Theta(\log^\beta T)$, and note that $\beta \in [1, +\infty)$ as $\abs{I} = \Omega(\log T)$ is the minimum order to ensure the adaptive regret to be non-trivial. Then,
    \begin{align*}
        &\O\left(\sqrt{\left(|I| + \log T \cdot \log^2 |I|\right) \log T} \right)\\
        &= \O\left(\sqrt{\left(\log^\beta T + \beta^2 \log T \cdot (\log\log T)^2 \right) \log T} \right)\\
        &=  \O\left(\sqrt{(\log^\beta T + \beta^2 \log T) \log T} \right)\\
        &= \O\left(\sqrt{\log^\beta T \log T} \right) = \O\big(\sqrt{|I| \log T}\big).
    \end{align*}
\end{itemize}
Hence we finish the proof for the worst-case adaptive regret bound. Combining both small-loss bound and the worst-case safety guarantee, we complete the proof of Theorem~\ref{thm:small-loss-adaptive}.
\end{proof}

\subsection{Proof of Lemma~\ref{lemma:pd-adaptive-surrogate-meta-regret}}
\label{sec:appendix-adaptive-lemma1}
\begin{proof}
First we introduce some useful variables to help us prove the adaptivity of Adapt-ML-Prod under sleeping-expert setting. Similar to the proof technique of~\citet{ICML'15:Daniely-adaptive}, for any interval $[i, j] \in \tilde{\mathcal{C}}$ in the geometric covers defined in~\eqref{eq:covering}, on which we suppose the $m$-th base-learner is active, we define the following pseudo-weight for the $m$-th base-learner,
\[ \tilde{w}_{\tau,m} = \left\{
    \begin{array}{cl}
    0 & \tau < i,\\
    1 & \tau = i,\\
    \big(\tilde{w}_{\tau-1, m}(1 + \eta_{\tau-1, m}(\hat{\ell}_{\tau-1} - \ell_{\tau-1,m}))\big)^{\frac{\eta_{\tau,m}}{\eta_{\tau-1, m}}} & i < \tau \leq j + 1,\\
	\tilde{w}_{j + 1, m} & \tau > j + 1.
\end{array} \right. \]

In addition, we use $\tilde{W}_t = \sum_{k \in [T]} \tilde{w}_{t, k}$ to denote the summation of pseudo-weights for all possible base-learners up to time $t$. As for the problem-dependent geometric covers, in the worst case there are at most $T$ base-learners generated, we use $[T]$ to denote the indexes for all the base-learners. Notice that the pseudo-weight $\tilde{w}_t$ is defined as $0$ for asleep base-learners till time $t$, so we can include all possible ones safely in the definition even though they are not generated in practical implementations of the algorithm.

Below, we use the potential argument~\citep{COLT'14:second-order-Hedge} to prove the desired result. Specifically, we establish the regret bound by lower and upper bounding the quantity $\ln \tilde{W}_{t+1}$.

\paragraph{Lower bound of $\ln \tilde{W}_{t+1}$.}
We claim that for $t \in [i, j]$ it holds that
\begin{equation*}
    \ln \tilde{w}_{t+1, m} \geq \eta_{t + 1, m} \sum_{\tau=i}^{t} (r_{\tau, m} - \eta_{\tau,m} r_{\tau,m}^2).
\end{equation*}

We prove the above inequality by induction on $t$. When $t = i$, by definition,
\begin{equation*}
    \ln\tilde{w}_{i+1, m} = \frac{\eta_{i+1, m}}{\eta_{i, m}}\ln \left(1 + \eta_m r_{i, m}\right) \geq \frac{\eta_{i+1, m}}{\eta_{i, m}}\left(\eta_m r_{i, m} - \eta_m^2r_{i,m}^2\right) = \eta_{i+1, m}(r_{i, m} - \eta_{m}r_{i,m}^2),
\end{equation*}
where the inequality is because of $\ln(1 + x) \geq x - x^2, \forall x \geq -1/2$.

Suppose the statement holds for $\ln\tilde{w}_{t, m}$, then we proceed to check the situation for $t+1$ round as follows. Indeed,
\begin{align}
    \ln \tilde{w}_{t+1, m} &{}=\frac{\eta_{t + 1, m}}{\eta_{t, m}} \left(\ln \tilde{w}_{t, m} + \ln \left(1 + \eta_{t, m} r_{t, m}\right) \right)\notag\\
	&{}\geq \frac{\eta_{t + 1, m}}{\eta_{t, m}} \left(\ln \tilde{w}_{t, m} + \eta_{t, m} r_{t, m} - \eta_{t, m}^2 r_{t, m}^2 \right)\notag\\
	&{}= \frac{\eta_{t + 1, m}}{\eta_{t, m}} \ln \tilde{w}_{t, m} + \eta_{t + 1, m}\left(r_{t,m} - \eta_{t, m} r_{t, m}^2\right)\notag\\
	&{}\geq \frac{\eta_{t + 1, m}}{\eta_{t, m}} \left(\eta_{t, m} \sum_{\tau=i}^{t - 1} (r_{\tau, m} - \eta_{\tau,m} r_{\tau,m}^2) \right)+ \eta_{t + 1, m}\left(r_{t,m} - \eta_{t, m} r_{t, m}^2\right)\notag\\
	&{}= \eta_{t + 1, m} \sum_{\tau=i}^{t} (r_{\tau, m} - \eta_{\tau,m} r_{\tau,m}^2). \label{eq:pd-AdaMLProd-w-low-bound}
\end{align}
Then, as $\tilde{w}_{t+1, m}$ is positive for any $m$-th base-learner, we have $\ln \tilde{W}_{t+1} \geq \ln \tilde{w}_{t+1, m} $. Thus by~\eqref{eq:pd-AdaMLProd-w-low-bound} we obtain the desired lower bound of $\ln \tilde{W}_{t+1}$.

\paragraph{Upper bound of $\ln \tilde{W}_{t+1}$.}
By the construction of the geometric covers as specified in~\pref{eq:covering}, we know that there will be at most $2m$ base-learners initialized for the $m$-th base-learner active on interval $[i, j]$ till her end. This is because $m$-th base-learner is initialized when $m$-th marker is recorded, and she will expire before the moment when $2m$-th marker is recorded, as demonstrated by the construction of cover defined in~\pref{eq:covering}. Owing to this property, we have $\tilde{W}_{t+1} = \sum_{k \in [2m]} \tilde{w}_{t+1, k}$ as others' pseudo-weight equals to $0$ by definition. So we can upper bound $\tilde{W}_{t+1}$ as,
\begin{align}
    \label{eq:pd-AdaMLProd-W-before-bound}
    \tilde{W}_{t+1} = \sum_{k \in [2m]} \tilde{w}_{t+1, k} &{} = \sum_{k \in [2m]:i_k = t+1} \tilde{w}_{t+1, k} + \sum_{k\in [2m]: i_k \leq t} \tilde{w}_{t+1, k} \notag\\
    &{} = \mathds{1}\{\text{new alg. at }t+1\} + \sum_{k\in [2m]:i_k \leq t} \tilde{w}_{t+1, k},
\end{align}
where we denote by $[i_k, j_k] \in \tilde{\mathcal{C}}$ the active time for $k$-th base-learner.

For the second term in~\eqref{eq:pd-AdaMLProd-W-before-bound}, we have
\begin{align}
    \sum_{k:i_k \leq t} \tilde{w}_{t+1, k} &= \sum_{k\in[2m]: t \in [i_k, j_k]} \tilde{w}_{t+1, k} +  \sum_{k\in[2m]:t > j_k} \tilde{w}_{t+1, k}\notag \\
    &= \sum_{k\in[2m]: t \in [i_k, j_k]} \tilde{w}_{t+1, k} +  \sum_{k\in[2m]:t > j_k} \tilde{w}_{t, k}\notag \\
    &{}\leq \sum_{k\in[2m]: t \in [i_k, j_k]} \tilde{w}_{t,k}(1 +\eta_{t,k} r_{t,k}) + \frac{1}{e} \left(\frac{\eta_{t,k}}{\eta_{t+1, k}} - 1 \right) + \sum_{k\in[2m]:t > j_k} \tilde{w}_{t, k}\notag\\
    &{}= \tilde{W}_t + \underbrace{\sum_{k\in[2m]: t \in [i_k, j_k]} \eta_{t, k} \tilde{w}_{t,k}r_{t,k}}_{= 0} + \sum_{k\in[2m]: t \in [i_k, j_k]} \frac{1}{e} \left(\frac{\eta_{t,k}}{\eta_{t+1, k}} - 1 \right), \label{eq:pd-AdaMLProd-W-w-bound}
\end{align}
where the first equality holds by the definition of $\tilde{w}_{t+1, k}$, the second inequality is by the updating rule of $\tilde{w}_{t+1, k}$ and Lemma~\eqref{lemma:second-order-hedge-weight-relation}, and the second term in the last equality equals to $0$ due to the weight update rule in~\eqref{eq:meta-AdaMLProd-weight} and the fact of $\tilde{w}_{t, k} = w_{t, k}$ for any $t \in [i_k, j_k]$.

Combining~\eqref{eq:pd-AdaMLProd-W-before-bound},~\eqref{eq:pd-AdaMLProd-W-w-bound} and by induction, we obtain the following upper bound:
\begin{align}
    \label{eq:pd-AdaMLProd-W-upp-bound}
    \tilde{W}_{t+1} &{}\leq 1 + 2m +  \frac{1}{e} \sum_{k \in [2m]} \sum_{\tau = i_k}^{\min \{t, j_k\}} \left(\frac{\eta_{\tau, k}}{\eta_{\tau + 1, k}} - 1\right).
\end{align}

We now turn to analyze the third term in~\eqref{eq:pd-AdaMLProd-W-w-bound}. Indeed,~\citet{COLT'14:second-order-Hedge} have analyzed it under the static regret measure. For the sake of completeness, we present the proof with our notations. For any $k \in [2m]$, for any $\tau \in [i_k, \min \{t, j_k\}]$, the relationship between $\eta_{\tau, k}$ and $\eta_{\tau + 1, k}$ can be considered as three cases,
\begin{itemize}
    \item $\eta_{\tau, k} = \eta_{\tau + 1, k} = 1/2,$
    \item
    $
        \eta_{\tau + 1, k} =  \sqrt{\gamma_k/(1 + \sum_{u=i_k}^{\tau} r_{u, k}^2)} < \eta_{\tau, k} = \frac{1}{2},
    $
    \item $\eta_{\tau + 1, k} \leq \eta_{\tau, k} < 1/2.$
\end{itemize}

In all cases, the ratio $\eta_{\tau, k} / \eta_{\tau + 1, k} - 1$ is at most
\begin{align}
    \sum_{\tau = i_k}^{\min \{t, j_k\}} \left( \frac{\eta_{\tau, k}}{\eta_{\tau + 1, k}} - 1\right) &{}\leq  \sum_{\tau = i_k}^{\min \{t, j_k\}} \left(\sqrt{\frac{1 + \sum_{u=i_k}^{\tau} r_{u, k}^2}{1 + \sum_{u=i_k}^{\tau - 1} r_{u, k}^2}} - 1\right) \notag\\
    &{}=\sum_{\tau = i_k}^{\min \{t, j_k\}} \left( \sqrt{\frac{r_{\tau, k}^2}{1 + \sum_{u=i_k}^{\tau - 1} r_{u, k}^2} + 1} - 1 \right)\notag\\
    &{}\leq \frac{1}{2} \sum_{\tau = i_k}^{\min \{t, j_k\}} \frac{r_{\tau, k}^2}{1 + \sum_{u=i_k}^{\tau - 1} r_{u, k}^2}\notag\\
    &{}\leq \frac{1}{2} \left( 1 + \ln\left(1 + \sum_{u=i_k}^{\min \{t, j_k\}} r_{u, k}^2\right)\right) - \ln(1)\notag\\
    &{}\leq \frac{1}{2} \left( 1 + \ln(1 + t)\right), \label{eq:pd-AdaMLProd-tune-ratio}
\end{align}
where the second inequality uses $\sqrt{1 + x} \leq 1 + x/2$ and the third inequality follows from Lemma~\ref{lemma:adamlprod-self-confident} with the choice of $f(x) = 1/x$. Substituting~\eqref{eq:pd-AdaMLProd-tune-ratio} into~\eqref{eq:pd-AdaMLProd-W-upp-bound}, we get
\begin{align}
    \tilde{W}_{t+1} \leq 1 + 2m + \frac{m}{e}\left(1 + \ln \left(1 + t\right)\right)\leq (1+2m) \left(1 + \frac{1}{2e}\left(1 + \ln(1+t)\right)\right). \label{eq:pd-adamlprod-W-upp-bound}
\end{align}
Further taking the logarithm over the above inequality gives the upper bound of $\ln \tilde{W}_{t+1}$.

\paragraph{Upper bound of meta-regret.}
Now, we can lower bound and upper bound $\ln\tilde{W}_{t+1}$ by~\eqref{eq:pd-AdaMLProd-w-low-bound} and~\eqref{eq:pd-adamlprod-W-upp-bound}. Then, rearranging the terms yields the upper bound of scaled meta-regret,
\begin{align}
    \sum_{\tau=i}^{t}r_{\tau, m} &{}\leq \sum_{\tau=i}^t \eta_{\tau,m}r_{\tau,m}^2 + \frac{\ln(1+2m) + \mu_t}{\eta_{t+1, m}} \notag\\
    &{} \leq 2\sqrt{\gamma_i} \sqrt{1 + \sum_{\tau = i}^t  r_{\tau, i}^2}   + \frac{\ln(1+2m) + \mu_t}{\eta_{t+1, m}}   \label{eq:pd-adamlprod-gamma-tuning}\\
    &{} \leq  \frac{\ln(1+2m) + \mu_t + 2\gamma_m}{\sqrt{\gamma_m}} \sqrt{1 + \sum_{\tau = i}^t  r_{\tau, m}^2}  + 2\ln(1 + 2m) +  4\gamma_m + 2\mu_t \notag \\
    &{} \leq \left(3\sqrt{\ln(1+2m)} + \mu_t\right)\sqrt{1 + \sum_{\tau = i}^t  r_{\tau, m}^2} + 6\ln(1+2m) + 2\mu_t, \notag
\end{align}
where we denote $\mu_t = \ln(1 + (1 + \ln(1+t))/(2e))$.The second inequality is by Lemma~\ref{lemma:adamlprod-self-confident} and choose $f(x) = 1/\sqrt{x}$. The last inequality is by the choice of $\sqrt{\gamma_m} = \sqrt{\ln(1+2m)} \geq \sqrt{\ln(3)} \geq 1$. As for the third inequality, there are two cases to be considered:
\begin{itemize}
    \item when $\sqrt{1 + \sum_{\tau = i}^t  r_{\tau, m}^2 } > 2\sqrt{\gamma_m}$, we have that~\eqref{eq:pd-adamlprod-gamma-tuning} is at most $2\sqrt{\gamma_m} \sqrt{1 + \sum_{\tau = i}^t  r_{\tau, m}^2} + \frac{\ln(1+2m) + \mu_t}{\sqrt{\gamma_m}} \sqrt{1 + \sum_{\tau = i}^t  r_{\tau, m}^2}$.
    \item when $\sqrt{1 + \sum_{\tau = i}^t  r_{\tau, m}^2 } \leq 2\sqrt{\gamma_m}$, we have that $\eta_{t+1, m} = 1/2$ and~\eqref{eq:pd-adamlprod-gamma-tuning} is at most $2\ln(1+2m)  + 4\gamma_m + 2\mu_t$.
\end{itemize}
Taking both cases into account implies the desired inequality.

Finally, we end the proof by evaluating the meta-regret in terms of the surrogate loss.
\begin{align}
    &\sum_{\tau = i}^t \inner{\gradg_t(\y_\tau)}{\y_\tau - \y_{\tau, m}} \notag\\
    &= 2GD\cdot\sum_{\tau=i}^{t}r_{\tau, m}\notag \\
    &\leq 2GD\bigg(3\sqrt{\ln(1+2m)} + \mu_t\bigg)\sqrt{1 + \sum_{\tau = i}^t  r_{\tau, m}^2} + 12GD\ln(1+2m) + 4GD\mu_t \notag\\
    &\leq \bigg(3\sqrt{\ln(1+2m)} + \mu_t\bigg)\sqrt{\sum_{\tau = i}^t  \inner{\gradg_\tau(\y_\tau)}{\y_\tau - \y_{\tau, m}}^2} + 18GD\ln(1+2m) + 6GD\mu_t \notag\\
    &\leq  \bigg(3\sqrt{\ln(1+2m)} + \mu_t\bigg)\sqrt{\sum_{\tau = i}^t 4D^2\norm{\gradg_\tau (\y_\tau)}_2^2 } + 18GD\ln(1+2m) + 6GD\mu_t \notag\\
    &\leq 2D\bigg(3\sqrt{\ln(1+2m)} + \mu_t\bigg)\sqrt{\sum_{\tau = i}^t \norm{\gradf_\tau (\x_\tau)}_2^2 } + 18GD\ln(1+2m) + 6GD\mu_t \notag\\
    &\leq 4D\bigg(3\sqrt{\ln(1+2m)} + \mu_t\bigg)\sqrt{L\sum_{\tau = i}^t f_t(\x_t) } + 18GD\ln(1+2m) + 6GD\mu_t, \notag
\end{align}
where the second inequality is true because $1 \leq \sqrt{\ln(1+2m)} \leq \ln(1+2m)$ holds for any $m \geq 1$, the third inequality is by Cauchy-Schwarz inequality, the forth inequality is by Theorem~\ref{thm:surrogate-loss} and the last inequality is due to the self-bounding property of smooth and non-negative functions (see Lemma~\ref{lemma:self-bounded}).
\end{proof}

\subsection{Proof of Lemma~\ref{lemma:pd-adaptive-covering-regret}}
\label{sec:appendix-adaptive-lemma2}
\begin{proof}
We start the proof by decomposing the adaptive regret into meta-regret and base-regret in terms of the surrogate loss by Theorem~\ref{thm:surrogate-loss},
\begin{align}
    \sum_{\tau=i}^t f_\tau(\x_\tau) - \sum_{\tau=i}^t f_\tau(\u) &{}\leq \sum_{\tau=i}^t g_\tau(\x_\tau) - \sum_{t=i}^t g_\tau(\u) \leq \sum_{\tau=i}^t \inner{\gradg_\tau(\y_\tau)}{\y_{\tau} - \u} \notag \\
    &{}=\underbrace{ \sum_{\tau=i}^t \inner{\nabla g_\tau(\y_\tau)}{\y_\tau - \y_{\tau, m}}}_{\meta} + \underbrace{ \sum_{\tau=i}^t \inner{\nabla g_\tau(\y_{\tau})}{\y_{\tau, m} - \u}}_{\base}, \label{eq:pd-adaptive-before-plug}
    \end{align}
where our analysis will be performed by tracking the $m$-th base-learner, whose corresponding active interval is exactly the analyzed one. Our analysis is satisfied to any interval since there is always a base-learner active on it ought to our algorithm design.

    \paragraph{Upper bound of base-regret.}
    Since the base-algorithm (SOGD) guarantees anytime regret, direct application of Lemma~\ref{lemma:sogd-T-regret} with the assumption of surrogate domain $\Y$ can upper bound the base-regret,
    \begin{align}
        \sum_{\tau=i}^t \inner{\nabla g_\tau(\y_\tau)}{\y_{\tau, m} - \u} &\leq 4D\sqrt{\delta + \sum_{\tau=i}^t \norm{\gradg_\tau(\y_\tau)}_2^2} \leq 8D\sqrt{L\sum_{\tau=i}^t f_\tau(\x_\tau)} + 4D\sqrt{\delta} \label{eq:pd-adaptive-base-regret},
    \end{align}
    where we skip some steps for transforming $\norm{\gradg_\tau(\y_\tau)}_2^2$ into $4Lf_\tau(\x_\tau)$. The similar arguments can be found in the proof of Theorem~\ref{thm:dynamic-regret-project-smooth}.

\paragraph{Upper bound of meta-regret.}
By Lemma~\ref{lemma:pd-adaptive-surrogate-meta-regret}, we can upper bound the meta-regret as
\begin{equation}
    \label{eq:pd-adaptive-meta-regret}
    \begin{split}
    & \sum_{\tau = i}^t \inner{\gradg_\tau(\y_\tau)}{\y_\tau - \y_{\tau, m}}\\
    & \leq  4D\left(3\sqrt{\ln(1+2m)} + \mu_t\right)\sqrt{L\sum_{\tau = i}^t f_\tau(\x_\tau)}  + 18GD\ln(1+2m) + 6GD\mu_t,
    \end{split}
\end{equation}
where we denote $\mu_t = \ln(1 + (1 + \ln(1+t))/(2e))$.

\paragraph{Upper bound of adaptive regret.}
Substituting~\eqref{eq:pd-adaptive-base-regret},~\eqref{eq:pd-adaptive-meta-regret} into~\eqref{eq:pd-adaptive-before-plug} obtains
    \begin{align}
        \label{eq:pd-adaptive-f_t-x-interval-regret}
        &\sum_{\tau=i}^t f_\tau(\x_\tau) - \sum_{\tau=i}^t f_\tau(\u)\notag\\
        &\leq 4D\left(3\sqrt{\ln(1+2m)} + \mu_t + 2\right)\sqrt{L\sum_{\tau = i}^t f_\tau(\x_\tau)} + 18GD\ln(1+2m) + 6GD\mu_t + 4D\sqrt{\delta}\notag\\
        &= \O\left(\sqrt{\log (m) \cdot \sum_{\tau=i}^t f_\tau(\x_\tau)} + \log (m)\right),
    \end{align}
    which proves the first part of the results.

    Furthermore, by applying the standard technical lemma presented in Lemma~\ref{lemma:substitute-F_T}, we can convert the cumulative loss of final decisions in the above regret bound, $\sum_{\tau=i}^t f_\tau(\x_\tau)$, into the comparator's cumulative loss, $\sum_{\tau=i}^t f_\tau(\u)$,
    \begin{align}
        &{}\sum_{\tau=i}^t f_\tau(\x_\tau) - \sum_{\tau=i}^t f_\tau(\u) \notag \\
        &{}\leq 4D\sqrt{L}\bigg(\sqrt{\ln(1+2m)} +\mu_t + 2\bigg)\sqrt{\sum_{\tau=i}^t f_\tau(\u) + 18GD\ln(1+2m) + 6GD\mu_t +4D\sqrt{\delta} }\notag\\
        &{} \qquad + 18GD\ln(1+2m) + 6GD\mu_t +4D\sqrt{\delta} + 16D^2L\bigg(\sqrt{\ln(1+2m)} +\mu_t + 2\bigg)^2\notag\\
        &{}\leq 4D\sqrt{L}\bigg(\sqrt{\ln(1+2m)} +\mu_t + 2\bigg)\sqrt{\sum_{\tau=i}^t f_\tau(\u)}  \notag\\
        &{}\qquad + 27GD\ln(1+2m) + 9GD\mu_t +6D\sqrt{\delta} + 24D^2L\bigg(\sqrt{\ln(1+2m)} +\mu_t + 2\bigg)^2 \notag\\
        &{} \leq 4D\sqrt{L}\bigg(\sqrt{\ln(1+2m)} +\mu_t + 2\bigg)\sqrt{\sum_{\tau=i}^t f_\tau(\u)} \notag\\
        &{} \qquad +(27GD+72D^2L)\ln(1+2m) + 72D^2L\mu^2_t + 9GD\mu_t + 6D\sqrt{\delta} + 288D^2L. \notag
    \end{align}
    The second inequality makes use of $\sqrt{a + b} \leq \sqrt{a} + \sqrt{b}$ and $\sqrt{ab} \leq (a + b)/2$. The last inequality holds by $(a + b + c)^2 \leq 3(a^2 + b^2 + c^2)$.
\end{proof}

\subsection{Proof of Lemma~\ref{lemma:threshold-lower-bound}}
\begin{proof}
    For interval $[s_m, s_{m+1}-1]$, there must exist an interval $[i, j] \in \tilde{\mathcal{C}}$ such that $[s_m, s_{m+1}-1] \subseteq [i, j]$ with $i = s_m$. Therefore, we can apply~\pref{eq:f_tu-bound} presented in Lemma~\ref{lemma:pd-adaptive-covering-regret} to upper bound the regret during $[s_m, s_{m+1}-1]$ for any comparator $\u \in \X$,
    \begin{align}
        \sum_{t=s_m}^{s_{m+1}-1} f_t(\x_t) - \sum_{t=s_m}^{s_{m+1}-1} f_t(\u)
        &{}\leq 4D\bigg(\sqrt{\ln(1+2m)} +\mu_T + 2\bigg)\sqrt{L\sum_{t=i}^j f_t(\u)} \notag \\
        &{} \qquad +(27GD+72D^2L)\ln(1+2m) + 72D^2L\mu^2_T + 9GD\mu_T \notag \\
        &{} \qquad + 6D\sqrt{\delta} + 288D^2L, \label{eq:lemma6-sqrt-ftu}
    \end{align}
    where we use the monotonically increasing property that $\mu_{s_{m+1}-1} \leq \mu_T$.

    Incorporating some basic inequalities, specifically $ab \leq a^2/4 + b^2$ and $(a+b+c)^2 \leq 3(a^2 + b^2 + c^2)$, we can isolate $\sum_{t=s_m}^{s_{m+1}-1} f_t(\u) $ from the square root in~\eqref{eq:lemma6-sqrt-ftu}:
    \begin{align}
        \sum_{t=s_m}^{s_{m+1}-1} f_t(\x_t) &\leq 2\sum_{t=s_m}^{s_{m+1}-1} f_t(\u) +  (27GD+84D^2L)\ln(1+2m) \notag \\
        &{} \qquad + 84D^2L\mu^2_T + 9GD\mu_T + 3D\sqrt{\delta} + 336D^2L \notag \\
        &=2\sum_{t=s_m}^{s_{m+1}-1} f_t(\u) + \frac{1}{2}\mathcal{G}(m), \label{eq:lemma-6-ftu}
    \end{align}
    where the equality is by the definition in~\pref{eq:generate-function-adaptive}.

    By the problem-dependent schedule mechanism, as stated in~{\color{blue}Lines}~\ref{line:condition}, the cumulative loss between $[i, j]$ exceeds the threshold $C_m = \mathcal{G}(m)$, i.e., $\sum_{t=s_m}^{s_{m+1}-1} f_t(\x_t) \geq C_m$. Therefore, together with~\pref{eq:lemma-6-ftu}, we can lower bound the cumulative loss of comparator $\u$ as,
    \begin{align*}
        \sum_{t=s_m}^{s_{m+1}-1} f_t(\u) \geq \frac{1}{2}\left(\sum_{t=s_m}^{s_{m+1}-1} f_t(\x_t) - \frac{1}{2}\mathcal{G}(m)\right) \geq \frac{1}{2}\left(C_m - \frac{1}{2}\mathcal{G}(m)\right) = \frac{1}{4}C_m,
    \end{align*}
    where the last equality is by the definition of threshold.
\end{proof}

\subsection{Proof of Lemma~\ref{lemma:pd-m-upper-bound}}
\label{sec:appendix-adaptive-lemma3}
\begin{proof}
We assume that the $m$-th base-learner is initialized at the beginning of interval $[i, j] \in \tilde{\mathcal{C}}$, in other words, $i = s_m$, where $s_m$ denotes the $m$-th marker. Before time stamp $i$, the schedule has registered $m-1$ markers, i.e., from $s_1$ to $s_{m-1}$. By Lemma~\ref{lemma:threshold-lower-bound}, we can lower bound the cumulative loss of any comparator $\u\in \X$ as,
\begin{align*}
    \sum_{\tau= s_1}^{s_m - 1} f_\tau(\u) \geq \frac{1}{4}\sum_{k=1}^{m-1} C_k \geq \frac{C_1}{4}(m-1).
\end{align*}
The second inequality holds since $C_k$ is monotonically increasing with respect to its index, see the threshold generating function in~\pref{eq:generate-function-adaptive}.

Therefore, rearranging the above inequalities gives upper bound of quantity $m$,
\begin{equation}
    m \leq 1 + \frac{4}{C_1} \sum_{\tau=s_1}^{s_m - 1} f_\tau(\u) \leq 1 + \frac{4}{C_1} \sum_{\tau=1}^{t} f_\tau(\u), \notag
\end{equation}
where the last inequality makes use of the non-negative assumption on loss functions. Choosing $\u \in \argmin_{\a \in \X}\sum_{\tau=1}^{t} f_\tau(\a)$ as the minimizer for the given interval $[i, j]$ gives
\begin{equation}
    \label{eq:lemma-7-proof-upper-m}
    m \leq 1 + \frac{4}{C_1} \min_{\u \in \X} \sum_{\tau=1}^{t} f_\tau(\u).
\end{equation}

Plugging \pref{eq:lemma-7-proof-upper-m} into~\pref{eq:f_tu-bound} in Lemma~\ref{lemma:pd-adaptive-covering-regret}, we prove the second result in this lemma,
\begin{equation*}
        \sum_{\tau=i}^t f_\tau(\x_\tau) - \mathop{\min}_{\u \in \X}\sum_{\tau = i}^t f_\tau(\u) \leq \alpha(t) + \beta(t)\sqrt{\mathop{\min}_{\u \in \X} \sum_{\tau=i}^t f_\tau(\u)} = \O\left(\sqrt{F_{[i, t]}\log F_{[1,t]}}\right),
\end{equation*}
    where
\begin{align*}
    \alpha(t) &{}= (27GD+72D^2L)\ln \left(3+ \frac{8}{C_1}\sum_{\tau=1}^t f_\tau(\u)\right) + 72D^2L\mu_t^2 + 9GD\mu_t + 6D\sqrt{\delta} + 288D^2L,\\
    \beta(t) &{} = 4D\sqrt{L}\left(\sqrt{\ln\left(3+ \frac{8}{C_1}\sum_{\tau=1}^t f_\tau(\u)\right)} +\mu_t + 2\right),
\end{align*}
with $\mu_t = \ln(1 + (1 + \ln(1+t))/(2e))$.
\end{proof}

\section{Omitted Details for Interval Dynamic Regret Minimization}
\label{sec:appendix-proof-simultaneous}
In this section, we first discuss the method of~\citet{ICML'20:Ashok} in Appendix~\ref{subsec:discussion-ashok-method}, another algorithm designed for optimizing interval dynamic regret in the literature, and present a projection-efficient variant. Then, we provide the proof of Theorem~\ref{thm:small-loss-interval-dynamic-regret}, the small-loss interval dynamic regret bound with $1$ projection complexity. Key lemmas supporting this proof are detailed in Appendix~\ref{sec:appendix-proof-simultaneous-thm-lemma}, with further key proofs found in Appendix~\ref{sec:appendix-small-loss-interval-dynamic-regret}.

\subsection{Projection-Efficient Variant for the Method of~\citet{ICML'20:Ashok}}
\label{subsec:discussion-ashok-method}
As discussed in the main text,~\citet{ICML'20:Ashok} proposes an algorithm to optimize interval dynamic regret, which significantly differs in its algorithmic framework from the meta-base structure of the AOA algorithm\citep{AISTATS'20:Zhang}. In this appendix, we further demonstrate how our reduction mechanism can be applied to achieve a projection-efficient variant of this algorithm. Additionally, we provide a computational comparison with our algorithm designed for optimizing interval dynamic regret (i.e., Algorithm~\ref{alg:problem-dependent-interval-dynamic}).

\paragraph{Reviewing the method of~\citet{ICML'20:Ashok}.} Unlike the meta-base type algorithm (such as AOA~\citep{AISTATS'20:Zhang}) which combines decisions using a weighted vector $\p_t \in \Delta_N$ from the meta-algorithm as $\x_t = \sum_{i=1}^N p_{t,i}\x_{t,i}$,~\citet{ICML'20:Ashok} aggregates the decisions from base-learners \emph{sequentially}, submitting the decision as $\x_t = \sum_{i=1}^N \x_{t,i}$ with the number of base-learners $N = \O(\log T)$. Their method relies on the specific beneficial properties of parameter-free online learners~\citep{COLT'18:black-box-reduction}. Specifically, with $\x_{t}^{(k-1)} = \sum_{i=1}^{k-1} \x_{t,i}$ representing the aggregated decisions of the first $k-1$ base-learners, the $k$-th base-learner then makes a decision $(\y_{t, k}, z_{t, k})$ on a $(d+1)$-dimensional lifted domain $V_t^{(k)}$ defined as
\begin{align}
    \label{eq:Ashok-lifted-domain}
    V_t^{(k)} = \big\{(\y, z) \in \R^d \times [0,1] \mid \x_{t}^{(k-1)} + \y - z\cdot\x_{t}^{(k-1)} \in \X  \big\} \subseteq \R^{d+1}.
\end{align}
The decision of the $k$-th base-learner is $\x_{t, k} = \y_{t, k} - z_{t, k}\cdot\x_{t}^{(k-1)}$. The projection onto $V_t^{(k)}$ ensures that the updated decision $\x_{t}^{(k)} = \x_{t}^{(k-1)} + \x_{t,k}$ remains within the feasible domain $\X$. Generally, projecting onto $V_t^{(k)}$ is not simpler than projecting onto the original feasible domain $\X$. Moreover, each round requires $\O(\log T)$ projections onto this lifted domain, significantly increasing computational costs.

\paragraph{Projection-Efficient Variant.} By applying our reduction method to the approach of~\citet{ICML'20:Ashok}, we can reduce the computation of the projection. Specifically, as discussed in Section~\ref{sec:reduction-apply-dynamic}, we can construct a surrogate domain $\Y \supseteq \X$ with a simpler structure, e.g., $\Y$ is often chosen as an Euclidean ball. Then, for $k \in [N-1]$, we allow the base-learners to project onto the lifted domain constructed over $\Y$:
\begin{align}
    \label{eq:Ashok-lifted-domain-new}
    \hat{V}_t^{(k)} = \big\{(\y, z) \in \R^d \times [0,1] \mid \x_{t}^{(k-1)} + \y - z\cdot\x_{t}^{(k-1)} \in \Y  \big\}.
\end{align}
If we choose $\Y$ to be an Euclidean ball, the corresponding lifted domain $\hat{V}_t^{(k)}$ becomes a cylinder, which simplifies the projection for the base-learners.
In this approach, we only require the $N$-th base-learner to project her decision onto the lifted domain constructed over $\X$, to ensure the final submitted decision $\x_t = \x_t^{(N)} = \sum_{i=1}^N \x_{t,i} \in \X$ to lie within $\X$. After submitting $\x_t$, we send the gradient $\g_t = \nabla f_t(\x_t)$ to the method of~\citet{ICML'20:Ashok} with the aforementioned revision.
Thus, our reduction mechanism can improve efficiency by requiring only $1$ projection onto the complex lifted domain associated with $\X$ (as defined in~\eqref{eq:Ashok-lifted-domain} for $k=N$), while performing $\O(\log T)$ projections onto simpler lifted domains (as defined in~\eqref{eq:Ashok-lifted-domain-new} for $k=1,\cdots, N-1$).

\paragraph{Comparisons.} Even with our reduction mechanism applied, the variant of ~\citet{ICML'20:Ashok}'s method still necessitates at least $1$ projection onto the complex lifted domain associated with $\X$, which is more complicated than the simple projection onto $\X$ required by our method. Furthermore, although performing $\O(\log T)$ projections onto simpler lifted domains is easy and feasible, it is still more complex than the simple rescaling operation our method uses for projections onto the (surrogate) Euclidean ball.

\subsection{Key Lemmas}
\label{sec:appendix-proof-simultaneous-thm-lemma}
This part collects several key lemmas for proving small-loss interval dynamic regret (Theorem~\ref{thm:small-loss-interval-dynamic-regret}). The first lemma shows that our base-algorithm enjoys an anytime dynamic regret.
\begin{myLemma}
    \label{lemma:anytime-small-loss-dynamic-regret}
    Under Assumptions~\ref{assumption:bounded-gradient},~\ref{assumption:bounded-domain}, and~\ref{assumption:smoothness}, setting the step size pool as $\H = \big\{\eta_j = 2^{j-1} \sqrt{5D^2 / (1+8LGDT)} \mid j \in [N]\big\}$ with $N = \lceil 2^{-1} \log_2((5D^2 + 2D^2T)(1+8LGDT)/(5D^2) ) \rceil + 1$ and $\epsilon_{i,\tau} =  \sqrt{(\ln N) /(1 + D^2 \sum_{s=i}^{\tau-1} \|\nabla g_s(\y_s)\|_2^2)}$, assume that the $m$-th base-learner described in Algorithm~\ref{alg:problem-dependent-interval-dynamic} is initialized at the beginning of an interval $[i, j] \in \tilde{\mathcal{C}}$, then this base-learner ensures the following dynamic regret
\begin{align*}
    & \sum_{\tau = i}^t \inner{\gradg_\tau(\y_\tau)}{\y_{\tau, m} - \u_\tau}\\
    &\leq  6D\sqrt{L(\ln N)  F_{[i,t]}^\x} + 3\sqrt{2(5D^2 + 2DP_{[i,t]})(1 + LF_{[i,t]}^\x)} +  \frac{(6+G^2D^2)\sqrt{\ln N}}{2} \\
    &=\O\left(\sqrt{F_{[i,t]}^\x(1 + P_{[i,t]})} \right),
\end{align*}
which holds for any time stamp $t \in [i,j]$ and any comparators sequence $\u_i, \dots, \u_t \in \X$. In above, $\y_{\tau, m}$ denotes the decision of the $m$-th base-learner at time stamp $\tau$, $F^\x_{[i,t]} = \sum_{\tau=i}^t f_\tau(\x_\tau)$ denotes the cumulative loss of the decisions and $P_{[i,t]} = \sum_{\tau = i+1}^t \norm{\u_\tau - \u_{\tau - 1}}_2$ is the path-length of comparators within the interval $[i,t]$.
\end{myLemma}

Combining the above lemma with the analysis of the meta-algorithm  Adapt-ML-Prod (see Lemma~\ref{lemma:pd-adaptive-surrogate-meta-regret}), we can obtain the interval dynamic regret for intervals in the problem-dependent geometric covering intervals $\tilde{\mathcal{C}}$.
\begin{myLemma}
\label{lemma:fpd-simul-cover-regret}
Under Assumptions~\ref{assumption:bounded-gradient},~\ref{assumption:bounded-domain}, and~\ref{assumption:smoothness}, for any interval $I=[i, j] \in \tilde{\mathcal{C}}$ in the geometric covers defined in~\pref{eq:covering}, at the beginning of which we assume the $m$-th base-learner is initialized, Algorithm~\ref{alg:problem-dependent-interval-dynamic} ensures
\begin{align}
    &\sum_{\tau=i}^{t} f_\tau(\x_\tau) - f_\tau(\u_\tau) \notag \\
     &\leq \left(12D\sqrt{\ln(1+2m)} + 4D\mu_t + 6D\sqrt{\ln N} + 3\sqrt{2(5D^2 + 2DP_{[i, t]})} \right) \sqrt{LF^{\u}_{[i, t]}} \notag\\
        &\qquad +\frac{3L}{2}\left(12D\sqrt{\ln(1+2m)} + 4D\mu_t + 6D\sqrt{\ln N} + 3\sqrt{2(5D^2 + 2DP_{[i, t]})} \right)^2 \notag \\
        &\qquad + 27GD\ln(1+2m) + 9GD\mu_t + \frac{9}{2}\sqrt{2(5D^2 + 2DP_{[i, t]})} + \frac{3(6+G^2D^2)\sqrt{\ln N}}{4} \label{eq:lemma9-ftu} \\
    &=\O\left(\sqrt{F^{\u}_{[i, t]}\left(P_{[i, t]} + \log m\right) }  + P_{[i, t]}\right), \notag
\end{align}
which holds for any time stamp $t \in [i,j]$ and any comparators sequence $\u_i, \dots, \u_t \in \X$. In above, $F^{\u}_{[i,t]} = \sum_{\tau=i}^t f_\tau(\u_\tau)$ denotes the cumulative loss of the comparators and $P_{[i,t]} = \sum_{\tau = i+1}^t \norm{\u_\tau - \u_{\tau - 1}}_2$ is the path-length of comparators within the interval $[i,t]$.
\end{myLemma}

As an analog to Lemma~\ref{lemma:threshold-lower-bound}, with the components at hand, we would like to estimate the cumulative loss $\sum_{t\in I} f_t(\u_t)$ in terms of the thresholds. However, our analysis shows that, we cannot estimate the cumulative loss alone, instead with the path length together, $F^{\u}_I + P_I = \sum_{t\in I} f_t(\u_t) + \sum_{t\in I} \norm{\u_t - \u_{t-1}}_2$ when comparing with a sequence of time-varying comparators. Indeed, a similar phenomenon appears also in the generalization of small-loss static regret to dynamic regret~\citep{JMLR:sword++}.

\begin{myLemma}
    \label{lemma:threshold-lower-bound-ada-dynamic}
    Under Assumptions~\ref{assumption:bounded-gradient},~\ref{assumption:bounded-domain}, and~\ref{assumption:smoothness}, for any interval $[s_m, s_{m+1} -1]$ determined by two consecutive intervals $s_m$ and $s_{m+1}$, where we denote by $s_m$ the $m$-th marker, Algorithm~\ref{alg:problem-dependent-interval-dynamic} ensures that for any comparators $\u_{s_m}, \dots, \u_{s_{m+1} -1} \in \X $, we have
    \begin{align*}
        2\sum_{t=s_m}^{s_{m+1}-1} f_t(\u_t) + (126L+5)\sum_{t=s_{m}+1}^{s_{m+1}-1}  \norm{\u_t - \u_{t-1}}_2 \geq \frac{1}{2}C_m,
    \end{align*}
    where $C_m = \mathcal{G}(m)$ is the $m$-th threshold with the threshold function $\mathcal{G}(\cdot)$ defined at~\pref{eq:threshold-generator-interval-dynamic}.
\end{myLemma}

\begin{myLemma}
    \label{lemma:simul-upper-bound-m}
    Under Assumptions~\ref{assumption:bounded-gradient},~\ref{assumption:bounded-domain}, and~\ref{assumption:smoothness}, for any interval $I=[i, j] \in \tilde{\mathcal{C}}$ in the geometric covers and any time stamp $t \in [i, j]$, the variable $m$ specified in Lemma~\ref{lemma:fpd-simul-cover-regret} can be bounded by
    \begin{equation*}
        m \leq \O\left(F^{\u}_{[1,t]} + P_{[1,t]}\right),
    \end{equation*}
    for any comparators $\u_1, \dots, \u_t \in \X$.
    This result can further imply that Algorithm~\ref{alg:problem-dependent-interval-dynamic} satisfies
    \begin{equation*}
        \sum_{\tau=i}^t f_\tau(\x_\tau) - \sum_{\tau = i}^t f_\tau(\u_\tau) \leq  \O\left(\sqrt{F^{\u}_{[i, t]}\left(P_{[i, t]} + \log \big( \min_{\u_{1:t}}\{ F^{\u}_{[1, t]} + P_{[1,t]}\} \big)\right) }  + P_{[i, t]} \right),
    \end{equation*}
    where $F^{\u}_{[i,t]} = \sum_{\tau=i}^t f_\tau(\u_\tau)$ denotes the cumulative loss of the comparators, $P_{[i,t]} = \sum_{\tau = i+1}^t \norm{\u_\tau - \u_{\tau - 1}}_2$ is the path-length of comparators within the interval $[i,t]$, and $\min_{\u_{1:t}}\{ F^{\u}_I + P_I\} = \min_{\u_{1}, \dots, \u_t}\left\{\sum_{\tau=1}^{t} f_t(\u_\tau) +  \sum_{\tau=2}^t \norm{\u_\tau - \u_{\tau-1}}_2\right\}$.
\end{myLemma}

\subsection{Proof of Theorem~\ref{thm:small-loss-interval-dynamic-regret}}
\label{sec:appendix-small-loss-interval-dynamic-regret}
\begin{proof}
    The proof closely follows that of Theorem~\ref{thm:small-loss-adaptive}, with only the key steps outlined here. For a comprehensive understanding of the techniques, readers may check the proofs in Appendix~\ref{sec:proof-small-loss-adaptive}, where the problem setting is simpler with only a fixed comparator considered.

    By a standard argument, we can identify markers \(s_p\) and \(s_q\) such that \(s_{p-1} \leq r < s_p\) and \(s_q \leq s < s_{q+1}\). Drawing upon Lemma~\ref{lemma:pcgc-number}, a sequence of intervals $I_1, \dots, I_v$ exists where $i_1 = p$, $i_v\leq q < i_{v+1}$, and $v \leq \lceil \log_2(q-p+2) \rceil$.

    For the interval $[r, s_p -1]$, we can directly bound the regret by the threshold of the cumulative loss and obtain that $\sum_{t=r}^{s_p -1} f_t(\x_t) - \sum_{t = r}^{s_p -1} f_t(\u_t) \leq \sum_{t=s_{p-1}}^{s_p - 1} f_t(\x_t) \leq C_{p-1} + GD$, due to the threshold mechanism and the non-negative property of loss functions. The threshold generating function specified in~\pref{eq:threshold-generator-interval-dynamic} further indicates $C_{p-1} \leq \O\left(\log (F_{T} + P_{T})\right)$.

    By Lemma~\ref{lemma:simul-upper-bound-m}, since $I_1$ to $I_v$ belong to the covering intervals, then our algorithm can enjoy $\O(P_{I_k} + \sqrt{F_{I_k}(P_{I_k} +\log(P_{I_k} + F_{I_k}))})$ regret bound for $k \in [v]$. Summing up the interval dynamic regret from $I_1$ to $I_v$ gives
    \begin{align*}
        \sum_{t = s_p}^{s} f_t(\x_t) - \sum_{t = s_p}^{s} f_t(\u_t) &= \sum_{k=1}^{v-1} \sum_{t \in I_k}\left(f_t(\x_t) - f_t(\u_t)\right) + \sum_{t \in [s_{i_v}, s]}\left(f_t(\x_t) - f_t(\u_t)\right)\\ &{}\leq \sum_{k=1}^{v-1} \O\left(P_{I_k} + \sqrt{F^{\u}_{I_k}\left(P_{I_k} + \log \big(\min_{\u_{1:T}}\{ F^{\u}_T + P_T\}\big)\right)} \right) \\
        &{} \qquad + \O\left(P_{[i_v, s]} + \sqrt{F^{\u}_{[i_v, s]}\left(P_{[i_v, s]} + \log \big(\min_{\u_{1:T}}\{ F^{\u}_T + P_T\}\big)\right)} \right)\\
        &{}\leq  \O\left( P_{I} + \sqrt{F^{\u}_{I}\left(P_{I} + \log \big(\min_{\u_{1:T}}\{ F^{\u}_T + P_T\}\big)\right)\cdot v }\right),
    \end{align*}
    where the last inequality makes use of Cauchy-Schwarz inequality and $v$ denotes the number of combined intervals. With Lemma~\ref{lemma:threshold-lower-bound-ada-dynamic}, we can upper bound $v$ through following steps,
    \begin{align*}
        &\min_{\u_{s_p:s_q-1}}\{F^{\u}_{[s_p, s_q-1]} +P_{[s_p, s_q-1]} \} \geq \frac{1}{4}\sum_{i=p}^{q-1} C_i \geq \frac{C_1}{4}(q - p)\\
        &\Rightarrow v \leq \lceil \log_2(q-p+2) \rceil \leq \O\left(\log \big(\min_{\u_{r:s}} \{F^{\u}_I + P_I\} \big)\right).
    \end{align*}

    Combining the upper bounds of $C_{p-1}$ and $v$ yields the following  interval dynamic regret,
    \begin{align*}
        &\sum_{t=r}^s f_t(\x_t) - \sum_{t = q}^s f_t(\u_t)\\ &{}\leq \O\left( P_{I} + \sqrt{F^{\u}_{I}\left(P_{I} + \log \big(\min_{\u_{1:T}}\{ F^{\u}_T + P_T\}\big) \right)\cdot v }\right) + C_{p-1} + GD\\
        &{}\leq \O\left( P_{I} + \sqrt{F^{\u}_{I}\left(P_{I} + \log \big(\min_{\u_{1:T}}\{ F^{\u}_T + P_T\}\big) \right)\cdot \log \big(\min_{\u_{r:s}}\{ F^{\u}_I + P_I\}\big) }\right) +\O\left(\log \big(\min_{\u_{1:T}}\{ F^{\u}_T + P_T\}\big)\right)\\
        &= \O\left( \sqrt{\big(F^{\u}_{I}+P_{I}\big)\left(P_{I} + \log \big(\min_{\u_{1:T}}\{ F^{\u}_T + P_T\}\big)\right) \cdot \log \big(\min_{\u_{r:s}}\{ F^{\u}_I + P_I\}\big) }\right).
    \end{align*}

    We mention that using a similar worst-case regret analysis to that in Appendix~\ref{sec:proof-small-loss-adaptive} can ensure a safety guarantee, which can strictly match the worst-case interval dynamic regret bound in~\citep{AISTATS'20:Zhang}. Details are omitted here.
\end{proof}

\subsection{Proof of Lemma~\ref{lemma:anytime-small-loss-dynamic-regret}}
\begin{proof}
    The proof is closely analogous to that of Theorem~\ref{thm:dynamic-regret-project-smooth}, except that this result can enjoy the anytime dynamic regret. For any time stamp $t \in [i, j]$, we can decompose the dynamic regret into meta-regret and base-regret, and bound them respectively. Notice that, we add the prefix ``\texttt{base:}'' to indicate that the regret analysis is over the base-algorithm level, as Algorithm~\ref{alg:problem-dependent-interval-dynamic} actually has three layers.
    \begin{align}
        \sum_{\tau = i}^t \inner{\gradg_\tau(\y_\tau)}{\y_{\tau, m} - \u_\tau} = \underbrace{\sum_{\tau=i}^t \inner{\nabla g_\tau(\y_\tau)}{\y_{\tau, m} - \y_{\tau, m, k}}}_{\texttt{base:meta-regret}} + \underbrace{\sum_{\tau=i}^t \inner{\nabla g_\tau(\y_\tau)}{\y_{\tau, m, k} - \u_\tau}}_{\texttt{base:base-regret}}, \label{eq:simul-before-plugging}
    \end{align}
    where $\y_{\tau, m, k}$ denotes the $k$-th decision maintained by the $m$-th base-learner in Algorithm~\ref{alg:problem-dependent-interval-dynamic}.
    \paragraph{Upper bound of base:meta-regret.} Notice that the meta-algorithm used for our dynamic algorithm is Hedge with self-confident tuning learning rates, so we have
    \begin{align}
        \sum_{\tau=i}^t \inner{\nabla g_\tau(\y_\tau)}{\y_{\tau, m} - \y_{\tau, m, k}} &{}\leq 3D\sqrt{\ln N   \sum_{\tau=i}^t \norm{\gradg_\tau(\y_\tau)}_2^2 }  + \frac{(6+G^2D^2)\sqrt{\ln N}}{2} \notag \\
        &{}\leq  6D\sqrt{L\ln N  \sum_{\tau=i}^t f_\tau(\x_\tau)} + \frac{(6+G^2D^2)\sqrt{\ln N}}{2}. \label{eq:simul-dynamic-meta-regret}
    \end{align}
    for any base:base-learner $k \in [N]$. The above reasoning is similar to~\pref{eq:dynamic-meta-regret}.

    \paragraph{Upper bound of base:base-regret.} By Theorem~\ref{thm:omd-dynamic-regret} and Lemma~\ref{thm:ogd-dynamic-regret}, it is easy to verify once the learning rate is set, OGD ensures the following dynamic bound before tuning,
    \begin{align}
        \sum_{\tau=i}^t \inner{\gradg_\tau(\y_\tau)}{\y_{\tau, m, k} - \u_\tau} & {} \leq \frac{5D^2}{2\eta_k} + \frac{D}{\eta_k}\sum_{\tau=i+1}^t \norm{\u_{\tau-1} - \u_\tau}_2 + \eta_k \sum_{\tau=i}^t \norm{\gradg_\tau(\y_\tau)}_2^2 \notag \\
        & {} \leq \frac{5D^2}{2\eta_k} + \frac{D}{\eta_k}\sum_{\tau=i+1}^t \norm{\u_{\tau-1} - \u_\tau}_2  + 4\eta_kL\sum_{\tau=i}^t f_\tau(\x_\tau) \label{eq:simul-dynamic-base-regret},
    \end{align}
    for any base:base-learner $k \in [N]$.
    \paragraph{Upper bound of anytime dynamic regret.}
    Plugging~\eqref{eq:simul-dynamic-meta-regret} and~\eqref{eq:simul-dynamic-base-regret} into~\eqref{eq:simul-before-plugging}, we can obtain the dynamic regret by tracking $k$-th base-learner,
    \begin{align}
        & \sum_{\tau = i}^t \inner{\gradg_\tau(\y_\tau)}{\y_{\tau, m} - \u_\tau} \notag \\
        &\leq 6D\sqrt{L\ln N  \sum_{\tau=i}^t f_\tau(\x_\tau)} + \frac{5D^2 + 2DP_{[i,t]}}{2\eta_k} +  4\eta_kL\sum_{\tau=i}^t f_\tau(\x_\tau) + \frac{(6+G^2D^2)\sqrt{\ln N}}{2}. \label{eq:simul-before-tuning-base}
    \end{align}

    Next, we specific the base-learner to compare with. We are aiming to choose the one with a step size closest to the (near-)optimal step size till time $t$, $\eta_t^{\star} =\sqrt{\frac{5D^2 + 2DP_t}{1 + 8LF_{[i, t]}^\x}}$, where we denote $F_{[i, t]}^\x = \sum_{\tau = i}^t f_\tau(\x_\tau)$. With the same argument as the proof of Theorem~\ref{thm:dynamic-regret-project-smooth}, we can identify the base-learner $k$ satisfying $ \eta_k \leq \eta_t^{\star} \leq  2\eta_k$. Tuning \pref{eq:simul-before-tuning-base} with learning rate $\eta_k$ specified above demonstrates that the dynamic regret can be upper bounded by
    \begin{align}
        6D\sqrt{L\ln N  \sum_{\tau=1}^t f_\tau(\x_\tau)} + 3\sqrt{2(5D^2 + 2DP_{[i,t]})(1 + LF_{[i,t]}^\x)} +  \frac{(6+G^2D^2)\sqrt{\ln N}}{2}. \label{eq:simul-dynamic-detailed-form}
    \end{align}
    This ends the proof.
\end{proof}

\subsection{Proof of Lemma~\ref{lemma:fpd-simul-cover-regret}}
\begin{proof}
    By Theorem~\ref{thm:surrogate-loss} and the combination of our algorithm, we can upper bound the interval dynamic on interval $I$ into two terms as before,
    \begin{align}
		\sum_{\tau=i}^{t} f_\tau(\x_\tau) -\sum_{\tau=i}^{t} f_\tau(\u_\tau) &\leq \sum_{\tau=i}^{t} \inner{\gradg_\tau(\y_\tau)}{\y_\tau - \u_\tau}\notag\\
        &= \sum_{\tau=i}^{t} \underbrace{\inner{\nabla g_\tau(\y_\tau)}{\y_\tau - \y_{\tau,m}}}_{\meta} + \underbrace{\inner{\nabla g_\tau(\y_\tau)}{\y_{\tau,m} - \u_\tau}}_{\base}, \notag
	\end{align}
    where the base-learner's decision $\y_{\tau, m}$ comes from $m$-th base-learner, namely the efficient dynamic algorithm, which is also a combination of several OGD algorithms and active on the considered interval. Our algorithm is a three-layer structure indeed, but we hide the details of the efficient dynamic algorithm by Lemma~\ref{lemma:anytime-small-loss-dynamic-regret}.

    \paragraph{Upper bound of meta-regret.} Our interval dynamic algorithm uses essentially the same meta-algorithm and cover as efficient adaptive algorithm (see Algorithm~\ref{alg:problem-dependent-adaptive}), so we can directly use Lemma~\ref{lemma:pd-adaptive-surrogate-meta-regret} to upper bound the meta-regret,
    \begin{align*}
        \sum_{\tau = i}^t \inner{\gradg_\tau(\y_\tau)}{\y_\tau - \y_{\tau, m}} \leq 4D\left(3\sqrt{\ln(1+2m)} + \mu_t \right)\sqrt{LF_{[i, t]}^\x } + 18GD\ln(1+2m) + 6GD\mu_t,
    \end{align*}
    where we denote $\mu_t = \ln(1 + (1 + \ln(1+t))/(2e))$ and $F_{[a,b]}^\x = \sum_{\tau = a}^b f_\tau(\x_\tau)$.

    \paragraph{Upper bound of base-regret.} By the step size pool setting and Lemma~\ref{lemma:anytime-small-loss-dynamic-regret}, we know that for any $t \in [i, j]$, our base-algorithm ensures anytime dynamic regret,
    \begin{align*}
        \sum_{\tau = i}^t\inner{\nabla g_\tau(\y_\tau)}{\y_{\tau,m} - \u_\tau} &\leq 6D\sqrt{L\ln N  F_{[i, t]}^\x} + 3\sqrt{2L(5D^2 + 2DP_t)F_{[i, t]}^\x}\\
        & \qquad + 3\sqrt{2(5D^2 + 2DP_{[i,t]})}+  \frac{(6+G^2D^2)\sqrt{\ln N}}{2}.
    \end{align*}
    \paragraph{Upper bound of interval dynamic regret.} Combining the meta-regret and base-regret discussed above, applying Lemma~\ref{lemma:substitute-F_T} and omitting tedious calculations, we have
    \begin{align}
        &\sum_{\tau=i}^{t} f_\tau(\x_\tau) - f_\tau(\u_\tau)\notag\\
        &\leq \left(12D\sqrt{\ln(1+2m)} + 4D\mu_t + 6D\sqrt{\ln N} + 3\sqrt{2(5D^2 + 2DP_{[i,t]})} \right) \sqrt{LF^{\u}_{[i, t]}} \notag\\
        &\qquad + 27GD\ln(1+2m) + 9GD\mu_t + \frac{9}{2}\sqrt{2(5D^2 + 2DP_{[i,t]})} + \frac{3(6+G^2D^2)\sqrt{\ln N}}{4} \notag \\
        &\qquad +\frac{3L}{2}\left(12D\sqrt{\ln(1+2m)} + 4D\mu_t + 6D\sqrt{\ln N} + 3\sqrt{2(5D^2 + 2DP_{[i,t]})} \right)^2 \notag \\
        &\leq \O\left(\sqrt{F^{\u}_{[i, t]}\left(P_{[i, t]} + \log m\right) }  + P_{[i, t]}+ \log m \right). \notag
    \end{align}
\end{proof}

\subsection{Proof of Lemma~\ref{lemma:threshold-lower-bound-ada-dynamic}}
\begin{proof}
    We introduce the notation $F_{[a, b]} = \sum_{t=a}^b f_t(\u_t)$ to denote the cumulative loss of comparators and $P_{[a, b]} = \sum_{t=a+1}^b \norm{\u_t - \u_{t-1}}_2$ to denote the path length of comparators. By Lemma~\ref{lemma:fpd-simul-cover-regret} and derivations, we can isolate $F_I$ and $P_I$ from the square root in~\eqref{eq:lemma9-ftu}, where we choose $I = [s_m, s_{m+1} - 1]$ as there always exists an interval in the schedule to cover it:
    \begin{align*}
        \sum_{t \in I} f_t(\x_t) \leq  2\sum_{t \in I}f_t(\u_t) + (126L+5)DP_{I} +  \frac{1}{2} \mathcal{G}(m),
    \end{align*}
    where $\mathcal{G}(\cdot)$ is the threshold function defined at~\pref{eq:threshold-generator-interval-dynamic}. By the threshold mechanism, we know that the cumulative loss within $I$ exceeds $C_m$, which implies
    \begin{align*}
        2F^{\u}_{I} + (126L+5)DP_{I} \geq \sum_{t\in I} f_t(\x_t) - \frac{1}{2}\mathcal{G}(m) \geq C_m - \frac{1}{2}\mathcal{G}(m) = \frac{1}{2}C_m,
    \end{align*}
    where we denote by $F^{\u}_{I} = \sum_{t \in I} f_t(\u_t)$.
\end{proof}
\subsection{Proof of Lemma~\ref{lemma:simul-upper-bound-m}}
\label{sec:appendix-interval-dynamic-m-bound}
\begin{proof}
    We assume that the $m$-th base-learner is initialized at the beginning of interval $[i, j] \in \tilde{\mathcal{C}}$ (in other words, $i = s_m$), where $s_m$ denotes the $m$-th marker. Before time stamp $i$, the schedule has registered $m-1$ markers, i.e., from $s_1$ to $s_{m-1}$. By Lemma~\ref{lemma:threshold-lower-bound-ada-dynamic}, for any comparators, we have
    \begin{align*}
        \sum_{k=1}^{m-1} \left(2F^{\u}_{[s_k, s_{k+1}-1]} + (126L+5)DP_{[s_k, s_{k+1}-1]}\right) \geq \frac{1}{2}\sum_{u=1}^{m-1} C_u \geq \frac{C_1}{2}(m-1).
    \end{align*}
    Rearranging the above inequality provides the upper bound of $m$ as
    \begin{align*}
        m &{}\leq 1 + \frac{2}{C_1} \left(\sum_{u=1}^{m-1} 2F^{\u}_{[s_u, s_{u+1}-1]} + (126L+5)DP_{[s_u, s_{u+1}-1]}\right)\\
        &{} \leq  1 + \frac{2}{C_1}\left( 2F^{\u}_{[s_1, s_{m} - 1]} + (126L+5)DP_{[s_1, s_{m} - 1]} \right)\\
        &{} \leq  1 + \frac{2}{C_1} \left( 2F^{\u}_{[1,t]} +  (126L+5)DP_{[1, t]}\right).
    \end{align*}
    Substituting the upper bound of $m$ into Lemma~\ref{lemma:fpd-simul-cover-regret} finishes the proof.
\end{proof}

\section{Omitted Details for Efficient Projection Examples}
In this section, we present the omitted details for the two applications of our proposed efficient projection scheme.

\subsection{Online Non-stochastic Control}
\label{sec:appendix-control}
In this subsection, we provide the required assumptions for online non-stochastic control and then provide the   proof of Theorem~\ref{thm:control}.
\subsubsection{Assumptions}
The following assumptions are required by Theorem~\ref{thm:control}, which are commonly used in the non-stochastic control analysis~\citep{hazan2022introduction,JMLR'23:memory}.
\begin{myAssumption}
  \label{assump:control-1}
The system matrices are bounded, i.e., $\norm{A}_{\operatorname{op}} \leq \kappa_A$ and $\norm{B}_{\operatorname{op}} \leq \kappa_B$. Besides, the disturbance $\norm{w_t} \leq W$ holds for any $t \in [T]$.
\end{myAssumption}
\begin{myAssumption}
  \label{assump:control-2}
The cost function $c_t(x, u)$ is convex. Further, when $\norm{x}, \norm{u} \leq D$, it holds that $\abs{c_t(x, u)}\leq \beta D^2$ and $\norm{\nabla_x c_t(x, u)}, \norm{\nabla_u c_t(x, u)} \leq G_cD$.
\end{myAssumption}
\begin{myAssumption}
  \label{assump:control-3}
  DAC controller $\pi(K, M)$ satisfies
  \begin{enumerate}
    \item $K$ is $(\kappa, \gamma)$-strongly stable, i.e., there exist matrices $L, H$ satisfying $A-BK = HLH^{-1}$, such that,
    \begin{enumerate}
      \item The spectral norm of $L$ satisfies $\norm{L} \leq 1-\gamma$.
      \item The controller and transforming matrices are bounded, i.e., $\norm{K} \leq \kappa$ and $\norm{H}, \norm{H^{-1}} \leq \kappa$.
    \end{enumerate}
    \item $M \in \M$ such that $\left\{M = (M^{[1]},\dots, M^{[H]}) \in \left(\R^{d_u \times d_x}\right)^H \mid \norm{M^{[i]}}_{\operatorname{op}} \leq \kappa_B \kappa^3(1-\gamma)^i \right\}$.
  \end{enumerate}
\end{myAssumption}

\subsubsection{Proof of Theorem~\ref{thm:control}}
\label{sec:appendix-proof-control}
The challenge in proving Theorem~\ref{thm:control} lies in accounting for the switching-cost while improving the efficiency. The crucial observation is given by $\norm{M_{t-1}-M_t}_{\operatorname{F}} \leq \norm{ M_{t-1}^\prime - M_t^\prime }_{\operatorname{F}}$. This relationship is derived by the nonexpanding property of projection operator~\citep{SIAM'09:nemirovski-robust-stochastic}. This implies that the switching-cost within the original domain can be constrained by that in the surrogate domain, which the algorithm is designed to minimize.
\begin{proof}
The proof mainly follows the one of Scream.Control. We present the essential steps to demonstrate the application of efficient reduction here and refer interested readers to Appendix C.2.3 in~\citet{AISTATS'22:scream} for comprehensive proof.

We denote by $f_t(\cdot):\M^{H+2} \mapsto \R$ the truncated loss and the dynamic regret enjoys the following decomposition:
\begin{align*}
  &\sum_{t=1}^T c_t(x_t, u_t) -\sum_{t=1}^T c_t(x_t^{\pi_t}, u_t^{\pi_t})\\
  &= \sum_{t=1}^T c_t(x_t^K(M_{0: t-1}), u_t^K(M_{0: t}))- \sum_{t=1}^Tc_t(x_t^K(M_{0: t-1}^*), u_t^K(M_{0: t}^*)) \\
  &=  \underbrace{\sum_{t=1}^T c_t(x_t^K(M_{0: t-1}), u_t^K(M_{0: t}))-\sum_{t=1}^T f_t(M_{t-1-H: t})}_{A_T} +\underbrace{\sum_{t=1}^T f_t(M_{t-1-H: t})- \sum_{t=1}^T f_t(M_{t-1-H: t}^*)}_{B_T} \\
  &\qquad +\underbrace{
    \sum_{t=1}^T f_t(M_{t-1-H: t}^*)-\sum_{t=1}^T c_t(x_t^K(M_{0: t-1}^*), u_t^K(M_{0: t}^*))}_{C_T} .
\end{align*}

Notice that $A_T$ and $C_T$ represent the approximation error induced by the truncated loss, which does not involve the efficient reduction and can be bounded effectively. As for $B_T$:
\begin{align*}
  B_T &\leq \sum_{t=1}^T \tilde{f}_t\left(M_t\right)-\sum_{t=1}^T \tilde{f}_t\left(M_t^*\right)+\lambda \sum_{t=2}^T\left\|M_{t-1}-M_t\right\|_{\operatorname{F}}+\lambda \sum_{t=2}^T\left\|M_{t-1}^*-M_t^*\right\|_{\operatorname{F}} \\
& \leq \sum_{t=1}^T\left\langle\nabla\tilde{f}_t\left(M_t\right), M_t-M_t^*\right\rangle+\lambda \sum_{t=2}^T\left\|M_{t-1}-M_t\right\|_{\operatorname{F}}+\lambda \sum_{t=2}^T\left\|M_{t-1}^*-M_t^*\right\|_{\operatorname{F}} \\
& \leq \sum_{t=1}^T\left\langle\nabla g_t\left(M_t^\prime\right), M_t^\prime-M_t^*\right\rangle+\lambda \sum_{t=2}^T\left\|M_{t-1}-M_t\right\|_{\operatorname{F}}+\lambda \sum_{t=2}^T\left\|M_{t-1}^*-M_t^*\right\|_{\operatorname{F}} \\
& \leq \sum_{t=1}^T\left\langle\nabla g_t\left(M_t^\prime\right), M_t^\prime-M_t^*\right\rangle+\lambda \sum_{t=2}^T\left\|M_{t-1}^\prime-M_t^\prime\right\|_{\operatorname{F}}+\lambda \sum_{t=2}^T\left\|M_{t-1}^*-M_t^*\right\|_{\operatorname{F}},
\end{align*}
where $\lambda = (H+2)^2L_f$ is a constant. The first inequality is by the coordinate-Lipschitz continuity of truncated function $f_t(\cdot)$. The third inequality is by the reduction mechanism and the final inequality is by $\norm{M_{t-1}-M_t}_{\operatorname{F}} \leq \norm{ M_{t-1}^\prime - M_t^\prime }_{\operatorname{F}}$ derived from the nonexpanding property of projection~\citep{SIAM'09:nemirovski-robust-stochastic} and that one can verify  in general this property holds for nearest point projection in Hilbert space.

Remind that Scream.Control employed in Algorithm~\ref{alg:control} aims at minimizing the dynamic regret with switching-cost in domain $\M^\prime$, which can guarantee $\Ot(\sqrt{T(1+P_T)})$ regret bound by Theorem~4 in~\citet{AISTATS'22:scream}. Thus, by taking into account that $\Fnorm{\nabla g_t\left(M_t^\prime\right)} \leq \Fnorm{\nabla\tilde{f}_t\left(M_t\right)}$ by the efficient reduction mechanism, which is true under non-stochastic control setting since the algorithm optimizes the linearized loss and employs the Frobenius norm as the projection distance metric, we can derive our result.
\end{proof}

\subsection{Online Principal Component Analysis}
\label{sec:appendix-pca}
This section provides omitted details for the online PCA problem. In Appendix~\ref{sec:appendix-pca-lemma} we provide the guarantee for base-algorithm and the lemma justifying the projection complexity. Appendix~\ref{sec:appendix-pca-proof} presents the overall proof of Theorem~\ref{thm:pca-adaptive}.
\subsubsection{Key Lemmas}
\label{sec:appendix-pca-lemma}
The following lemma presents the base-regret for the employed gradient-based algorithm.
\begin{myLemma}
  \label{lemma:pca-base-regret}
  Assuming $\Fnorm{\Xb_t} \leq 1$ and $k \leq \frac{d}{2}$, then any base-algorithm employed in Algorithm~\ref{alg:PCA} specified as,
  \begin{align*}
    \hat{\P}^{s, \prime}_{t+1,i}= \hat{\P}^s_{t,i} - \eta_i \nabla g_t(\hat{\P}_t^s),~~ \hat{\P}^s_{t+1,i} = \hat{\P}^{s, \prime}_{t+1,i} \Big(\ind{\Fnorm{\hat{\P}^{s, \prime}_{t+1,i}} \leq \sqrt{k}} + \frac{\sqrt{k}}{\Fnorm{\widehat{\P}_{t+1}}}\ind{\Fnorm{\hat{\P}^{s, \prime}_{t+1,i}} > \sqrt{k}}\Big),
  \end{align*}
  which is active during time span $I =[r, s] \subseteq [T]$ and is indexed by number $i \in [T]$, ensures the following regret bound for any comparator $\P \in \Pcal_k$ by
  tuning learning rate as $\eta = \frac{k(d-k)}{d\abs{I}}$,
  \begin{align*}
    \sum_{t=r}^s \operatorname{tr}\left(\nabla g_t(\hat{\P}^s_t) \cdot  \hat{\P}^s_{t,i}\right) - \sum_{t=r}^s  \operatorname{tr}\left(\nabla g_t(\hat{\P}^s_t) \cdot \P\right) \leq \O\left(\sqrt{k\cdot \abs{I}}\right),
  \end{align*}
\end{myLemma}
The above claim can be verified by the proof in Appendix. B of \citet{jmlr'16:Warmuth-pca} together with $\norm{\nabla g_t(\hat{\P}^s_t)}_{\operatorname{F}} \leq \norm{\nabla f_t(\hat{\P}_t)}_{\operatorname{F}}$, which can be verified by noticing that the loss function $f_t(\P)$ is coordinate-wise linear with $\P$ and we use Frobenius norm as the distance metric.

The following lemma provides the details to project decision onto domain $\hat{\Pcal}_k$.
\begin{myLemma}[Lemma 3.2 of~\citet{NIPS'13:PCA-projection}]
  \label{lemma:pca-projection}
  Let $\mathbf{P}^\prime \in \mathbb{R}^{d\times d}$ be a symmetric matrix, with eigenvalues $\sigma_1^\prime, \dots, \sigma_d^\prime$ and associated eigenvectors $\v_1^\prime, \dots, \v_d^\prime$ . Its projection $\mathbf{P} = \Pi_{\hat{\Pcal}_k}[\mathbf{P}^\prime]$ onto the domain $\hat{\Pcal}_k$ with respect to the Frobenius norm, is the unique feasible matrix which has the same eigenvectors as $\mathbf{P}^\prime$, with the associated eigenvalues $\sigma_1, \dots, \sigma_d$ satisfying:
  \begin{align*}
    \sigma_i = \max\left(0, \min (1, \sigma_i^\prime + S)\right), i \in [d],
  \end{align*}
  with $S\in \R$ being chosen in such a way that $\sum_{i=1}^d \sigma_i = k$. Moreover, there exists an algorithm to find $S$ in an $\O(d\log d)$ running time complexity.
\end{myLemma}

\subsubsection{Proof of Theorem~\ref{thm:pca-adaptive}}
\label{sec:appendix-pca-proof}
The proof of Theorem~\ref{thm:pca-adaptive} enjoys much similarity as the one of efficient adaptive algorithm. We refer the readers to Appendix~\ref{sec:adaptive-regret} for more details.
\begin{proof}
  We mainly present the key steps for applying our reduction. For any interval $I = [r,s] \subseteq [T]$ and any comparator $\P \in \Pcal_k$, starting with the linearity of expectation, we have,
\begin{align*}
  \sum_{t=r}^s \E\left[ f_t\left(\P_t\right) \right] -  f_t(\P) &=  \sum_{t=r}^s \E\left[ \tr\left((\I - \P_t)\Xb_t\right)\right] - \tr\left((\I - \P)\Xb_t\right)\\
  &= \sum_{t=r}^s \tr\left((\I - \hat{\P}_t)\Xb_t\right)- \tr\left((\I - \P)\Xb_t\right)\\
  &= \sum_{t=r}^s \tr\left(\nabla f_t(\hat{\P}_t)\cdot\hat{\P}_t\right)- \tr\left(\nabla f_t(\hat{\P}_t) \cdot \P\right)\\
  &\leq \sum_{t=r}^s \tr\left(\nabla g_t(\hat{\P}^s_t)\cdot\hat{\P}_t^s\right)- \tr\left(\nabla g_t(\hat{\P}^s_t) \cdot \P\right)\\
  &=\underbrace{ \sum_{t=r}^s \tr\left(\nabla g_t(\hat{\P}^s_t)\cdot\hat{\P}_t^s\right) -  \sum_{t=r}^s \tr\left(\nabla g_t(\hat{\P}^s_t)\cdot\hat{\P}_{t,i}^s\right)}_{\meta}\\
  &{}\qquad +  \underbrace{\sum_{t=r}^s \tr\left(\nabla g_t(\hat{\P}^s_{t})\cdot\hat{\P}_{t,i}^s\right) -  \sum_{t=r}^s \tr\left(\nabla g_t(\hat{\P}^s_t)\cdot\P\right)}_{\base},
\end{align*}
where the first inequality is by Theorem~\ref{thm:surrogate-loss}, which is true under online PCA setting, since the optimization operates within the Hilbert space.

Since we employ Adapt-ML-Prod and standard geometric covering schedule to ensemble the base-learners, then one can expect that $\meta \leq \O\left(\sqrt{k\cdot\abs{I}\cdot\log T}\right)$. By Lemma~\ref{lemma:pca-base-regret}, the base-regret is of order $\base\leq \O(\sqrt{k\cdot\abs{I}})$. Combining these two bounds together, we finish the proof.
\end{proof}

\section{Useful Lemmas}
\label{sec:appendix-tech-lemmas}
This section collects some lemmas useful for the proofs.
\subsection{OGD and Dynamic Regret}

This part provides the dynamic regret of online gradient descent (OGD)~\citep{ICML'03:zinkvich} and scale-free online gradient descent (SOGD)~\citep{TCS'18:SOGD} from the view of online mirror descent (OMD), which is a common and powerful online learning framework. Following the analysis in~\citep{JMLR:sword++}, we can obtain dynamic regret of OGD and SOGD in a unified view owing to the versatility of OMD. Specifically, OMD updates by
\begin{equation}
    \label{eq:OMD-update-rule}
    \x_{t+1} = {}  \argmin_{\x \in \X}~\eta_t \inner{\nabla f_{t}(\x_{t})}{\x} + \D_{\Rcal}(\x,\x_{t}),
\end{equation}
where $\eta_t > 0$ is the time-varying step size, $f_t(\cdot):\x \mapsto \mathbb{R}$ is the convex loss function, and $\D_{\Rcal}(\cdot, \cdot)$ is the Bregman divergence induced by $\Rcal(\cdot)$ defined as $\D_{\Rcal}(\x, \y) = \Rcal(\x) - \Rcal(\y) - \inner{\nabla \Rcal(\y)}{\x - \y}$. OMD enjoys the following dynamic regret guarantee~\citep{JMLR:sword++}.
\begin{myThm}[Theorem 1 of~\citet{JMLR:sword++}]
    \label{thm:omd-dynamic-regret}
    Suppose that the regularizer $\Rcal: \X \mapsto \R$ is  1-strongly convex with respect to the norm $\|\cdot\|$. The dynamic regret of Optimistic Mirror Descent (OMD) whose update rule specified in~\eqref{eq:OMD-update-rule} is bounded as follows:
\begin{align*}
& \sum_{t=1}^{T} f_t(\x_t) - \sum_{t=1}^{T} f_t(\u_t) \\
& \leq \sum_{t=1}^{T}\eta_t \norm{\nabla f_t(\x_t)}_*^2 +  \sum_{t=1}^{T} \frac{1}{\eta_t} \Big( \Div{\u_t}{\x_t} - \Div{\u_t}{\x_{t+1}}\Big)  -  \sum_{t=1}^{T} \frac{1}{\eta_t}\Div{\x_{t+1}}{\x_t},
\end{align*}
which holds for any comparator sequence $\u_1,\ldots,\u_T \in \X$.
\end{myThm}

Choosing $\Rcal(\x) = \frac{1}{2}\norm{\x}_2^2$ will lead to the update form of online gradient descent used as base-learners in our algorithm:
\begin{equation}
    \label{eq:OGD-update-rule}
    \x_{t+1} = {}  \argmin_{\x \in \X}~\eta_t \inner{\nabla f_{t}(\x_{t})}{\x} + \frac{1}{2} \norm{\x - \x_t}^2_2,
\end{equation}
where the Bregman divergence becomes $ \D_{\Rcal}(\x,\x_{t}) = \frac{1}{2}\norm{\x-\x_t}_2^2$. We proceed to show the dynamic regret of online gradient descent (OGD),
\begin{myLemma}
    \label{thm:ogd-dynamic-regret}
    Under Assumption~\ref{assumption:bounded-domain}, by choosing static step size $\eta_t = \eta > 0$, Online Gradient Descent defined in~\pref{eq:OGD-update-rule} satisfies:
    \begin{equation*}
        \sum_{t=1}^{T} f_t(\x_t) - \sum_{t=1}^{T} f_t(\u_t) \leq \frac{7D^2}{4\eta} + \frac{D}{\eta}\sum_{t=2}^T \norm{\u_{t-1} - \u_t}_2 + \eta \sum_{t=1}^T \norm{\nabla f_t(\x_t)}_2^2
    \end{equation*}
    for any comparator sequence $\u_1,\ldots,\u_T \in \X$.
\end{myLemma}

\begin{proof}
    Applying Theorem~\ref{thm:omd-dynamic-regret} with the $\Rcal(\x) = \frac{1}{2} \norm{\x}_2^2$ and fixed step size $\eta_t = \eta > 0$ gives
    \begin{align*}
   & \sum_{t=1}^{T} f_t(\x_t) - \sum_{t=1}^{T} f_t(\u_t)\\
   & \leq \frac{1}{2\eta} \sum_{t=1}^{T} \Big( \norm{\u_t - \x_t}_2^2 - \norm{\u_t - \x_{t+1}}_2^2\Big) + \eta \sum_{t=1}^{T} \norm{\nabla f_t(\x_t)}_2^2 - \frac{1}{2\eta} \sum_{t=1}^T \norm{\x_t - \x_{t+1}}_2^2\\
   & \leq \frac{1}{2\eta} \sum_{t=1}^{T} \left(\norm{\x_t}_2^2 - \norm{\x_{t+1}}_2^2\right) + \frac{1}{\eta} \sum_{t=1}^T \left(\x_{t+1} - \x_{t}\right)^\top \u_t + \eta \sum_{t=1}^{T} \norm{\nabla f_t(\x_t)}_2^2 \\
   & \leq \frac{1}{2\eta}\norm{\x_1}_2^2 + \frac{1}{\eta}\left(\x_{T+1}^\top \u_T - \x_1^\top \u_1 \right) + \frac{1}{\eta}\sum_{t=2}^T(\u_{t-1} - \u_t)^\top \x_t + \eta \sum_{t=1}^{T} \norm{\nabla f_t(\x_t)}_2^2\\
   & \leq \frac{7D^2}{4\eta} + \frac{D}{\eta}\sum_{t=2}^T \norm{\u_{t-1} - \u_2}_2 + \eta \sum_{t=1}^{T} \norm{\nabla f_t(\x_t)}_2^2,
    \end{align*}
    where the last inequality is due to the domain boundedness. This ends the proof.
\end{proof}

\begin{myLemma}[Stability Lemma]
    \label{lemma:stability-lemma}
    Suppose the regularizer $\psi:\X \mapsto \R$ is $1$-strongly convex with respect to the norm $\norm{\cdot}$. The subsequent decisions $\x_{t+1}, \x_t$ specialized in the OMD update rule~\eqref{eq:OMD-update-rule} satisfy $\norm{\x_{t+1} - \x_t} \leq \eta_t \norm{\nabla f_t(\x_t)}_*$.
\end{myLemma}

\subsection{Self-confident Tuning}
\label{sec:appendix-self-condident-tuning}
\citet{TCS'18:SOGD} proved the regret bound of SOGD, and for completeness, we here provide the regret analysis under the OMD framework. Indeed, SOGD can be treated as OMD with a self-confident learning rate. Thus, we have the following lemma.
\begin{myLemma}
    \label{lemma:sogd-T-regret}
    Under Assumptions~\ref{assumption:bounded-gradient} and~\ref{assumption:bounded-domain}, the OMD algorithm defined in equation~\eqref{eq:OMD-update-rule} with the choices of regularizer $\Rcal(\x) = \frac{1}{2}\norm{\x}_2^2$ and time-varying learning rates  $\eta_{t} = \frac{D}{2\sqrt{\delta + \sum_{s=1}^{t-1} \|\nabla f_s(\x_s)\|_2^2}}$ with $\delta > 0 $ enjoys the following guarantee:
    \begin{equation*}
        \sum_{t=1}^T f_t(\x_t) - \sum_{t=1}^T f_t(\u) \leq 2D\cdot \sqrt{\delta + \sum_{t=1}^T \norm{\nabla f_t(\x_t)}_2^2}.
    \end{equation*}
\end{myLemma}

\begin{proof}
    Applying Theorem~\ref{thm:omd-dynamic-regret}, we have that for any comparator $\u \in \X$,
    \begin{align}
    & \sum_{t=1}^{T} f_t(\x_t) - \sum_{t=1}^{T} f_t(\u_t) \notag \\
    & \leq \sum_{t=1}^{T} \frac{1}{2\eta_t} \Big( \norm{\u - \x_t}_2^2 - \norm{\u - \x_{t+1}}_2^2\Big) +  \sum_{t=1}^{T} \eta_t \norm{\nabla f_t(\x_t)}_2^2 -  \sum_{t=1}^T\frac{1}{2\eta_t} \norm{\x_t - \x_{t+1}}_2^2\notag\\
    & \leq  \frac{1}{2\eta_1}\norm{\u - \x_1}_2^2 + \sum_{t=2}^T \left(\frac{1}{\eta_t} - \frac{1}{\eta_{t-1}}\right) \frac{\norm{\u - \x_t}_2^2}{2} +  \sum_{t=1}^{T} \eta_t \norm{\nabla f_t(\x_t)}_2^2 \notag \\
    & \leq \frac{D^2}{2\eta_1} + \frac{D^2}{2} \sum_{t=2}^T \left(\frac{1}{\eta_t} - \frac{1}{\eta_{t-1}}\right) +  \sum_{t=1}^{T} \eta_t \norm{\nabla f_t(\x_t)}_2^2  \notag\\
    & = \frac{D^2}{2\eta_T} +  \sum_{t=1}^{T} \eta_t \norm{\nabla f_t(\x_t)}_2^2.\notag
    \end{align}
    Then, applying Lemma~\ref{lemma:self-confident} to the second term and using the definition of $\eta_T$, we obtain the following regret bound:
    \begin{align*}
        \sum_{t=1}^{T} f_t(\x_t) - \sum_{t=1}^{T} f_t(\u_t) &{} \leq D\cdot\sqrt{\delta + \sum_{t=1}^T \norm{\nabla f_t(\x_t)}_2^2} + D\left(\sqrt{\delta + \sum_{t=1}^T \norm{\nabla f_t(\x_t)}_2^2} - \sqrt{\delta} \right)\\
    &{}\leq  2D\cdot \sqrt{\delta + \sum_{t=1}^T \norm{\nabla f_t(\x_t)}_2^2},
    \end{align*}
    which completes the proof.
\end{proof}

To bound the meta-regret of our dynamic methods, we introduce the FTRL lemma~\cite[Corollary 7.8]{book'12:FO-book} under the time-varying learning rates.
\begin{myLemma}[FTRL Lemma]
    \label{lemma:FTRL-Lemma}
    Suppose that the regularizer function $\Rcal: \X \mapsto \R$ is $\alpha$-strongly convex with respect to the norm $\|\cdot\|$. Let $f_t$ be a sequence of convex loss functions and $\psi_t(\x) = \frac{1}{\eta_t}(\psi(\x) - \min_{\x^\prime \in \X} \psi(\x^\prime))$, where $\eta_{t+1} \leq \eta_{t}$ holds for $t \in [T]$. Then the decision sequence $\x_t$ generated by the following FTRL update rule
    \begin{equation*}
        \x_t =  \argmin_{\x \in \X} \bigg\{\psi_t(\x) + \sum_{\tau=1}^{t-1} f_\tau(\x) \bigg\}
    \end{equation*}
    satisfies the following regret upper bound for any $\u \in \X$,
    \begin{equation*}
        \sum_{t=1}^{T} f_t(\x_t) - \sum_{t=1}^{T} f_t(\u) \leq \frac{\psi(\u)-\min _{\x\in \X} \psi(\x)}{\eta_{T+1}}+\frac{1}{2 \alpha} \sum_{t=1}^{T} \eta_{t}\norm{\gradf_t(\x_t)}_{*}^{2}.
    \end{equation*}
\end{myLemma}

Based on it, we can derive the regret upper bound for the Hedge algorithm with self-confident learning rates.
\begin{myLemma}
    \label{lemma:self-tuning-hedge}
    Consider the prediction with expert advice setting with $N$ experts and the linear loss $f_t(\x) = \inner{\ellb_t}{\x}$, where $\ell_t \in \mathbb{R}^d$. Then the self-confident tuning Hedge, whose initial decision is $\p_1 = 1/N \cdot \boldsymbol{1}$ and update rules are
    \begin{align*}
        p_{t+1, i} \propto \exp\left( \epsilon_{t+1}\sum_{\tau=1}^t \ell_{\tau ,i} \right) \mbox{ with } \epsilon_{t+1} = \sqrt{\frac{\ln N}{1 + \sum_{\tau=1}^t \norm{\ell_\tau}_{\infty}^2 }}
    \end{align*}
    ensures the following regret guarantee: for any $i \in [N]$
    \begin{equation*}
        \sum_{t=1}^T \inner{\p_t}{ \ellb_t} -  \sum_{t=1}^T \ell_{t, i} \leq 3\sqrt{\ln N \cdot \left(1 + \sum_{t=1}^T\norm{\ellb_t}_{\infty}^2 \right) } +\frac{\sqrt{\ln N}}{2} \cdot \max_{t \in [T]} \norm{\ellb_t}_{\infty}^2.
    \end{equation*}
\end{myLemma}
\begin{proof}
    It is easy to verify that this Hedge update can be treated as a special case of the time-varying FTRL algorithm by choosing $\psi(\p) = \sum_{s=1}^N p_s \ln p_s$, which is $1$-strongly convex with respect to $\|\cdot\|_1$, and $\psi_t(\p) = \frac{1}{\epsilon_t}\psi(\p)$. Thus, by Lemma~\ref{lemma:FTRL-Lemma}, we have
    \begin{align*}
        \sum_{t=1}^T \inner{\p_t}{ \ellb_t} -  \sum_{t=1}^T \ell_{t, i} &{} \leq \frac{\ln N}{\epsilon_{T+1}} + \frac{1}{2} \sum_{t=1}^T \epsilon_t \norm{\ellb_t}_{\infty}^2\\
        &\leq  \frac{\ln N}{\epsilon_{T+1}} + \frac{\sqrt{\ln N}}{2} \cdot \left(4\sqrt{1 + \sum_{t=1}^T \norm{\ellb_t}_{\infty}^2} + \max_{t\in [T]} \norm{\ellb_t}_{\infty}^2\right)\\
        &= 3\sqrt{\ln N \cdot \left(1 + \sum_{t=1}^T\norm{\ellb_t}_{\infty}^2 \right) } +\frac{\sqrt{\ln N}}{2} \cdot \max_{t \in [T]} \norm{\ellb_t}_{\infty}^2,
    \end{align*}
    where the first inequality chooses $\u$ as the one-hot vector with all entries being $0$ except the $i$-th one as $1$, and second inequality is by Lemma~\ref{lemma:UAI-first-order}.
\end{proof}

\subsection{Facts on Geometric Covers}
\begin{myLemma}[Lemma 11 of~\citet{ICML19:Zhang-Adaptive-Smooth}]
    \label{lemma:pcgc-number}
    Let $[s_p, s_q] \subseteq [T]$ be an arbitrary interval that starts from a marker $s_p$ and ends at another marker $s_q$. Then, we can find a sequence of consecutive intervals $I_1=[s_{i_1}, s_{i_2} - 1]$, $I_2 = [s_{i_2}, s_{i_3} - 1]$, $\ldots$, $I_v = [s_{i_v}, s_{i_{v+1}} - 1] \in \tilde{\mathcal{C}}$    such that $i_1 = p$, $i_v\leq q < i_{v+1}$, and $v \leq \lceil \log_2(q-p+2) \rceil$.
\end{myLemma}

\subsection{Technical Lemmas}
\label{sec:tech-lemma}
In this part, we present several technical lemmas used in the proofs.
\begin{myLemma}[Lemma 3.1 of~\citet{NIPS'10:smooth}]
    \label{lemma:self-bounded}
    Suppose $f: \R^d \mapsto \R$ is an $L$-smooth function, then for any $\x \in \R^d$, we have
    \begin{equation}
        \label{eq:self-bounding-full}
        \norm{\gradf(\x)}_2^2 \leq 2L \cdot \left( f(\x) - \min_{\x \in \R^d} f(\x) \right).
    \end{equation}
    Furthermore, when the function is non-negative, we have $\norm{\gradf(\x)}_2 \leq \sqrt{2L \cdot f(\x)}$.
\end{myLemma}

\begin{myLemma}[Lemma 3.5 of~\citet{JCSS'02:Auer-self-confident}]
    \label{lemma:self-confident}
    Let $l_1, \dots, l_T$ be non-negative real numbers. Then
    \begin{equation*}
        \sum_{t=1}^T \frac{l_t}{\sqrt{\delta + \sum_{i=1}^t l_i}} \leq 2\left(\sqrt{\delta + \sum_{t=1}^T l_t} -\sqrt{\delta}\right).
    \end{equation*}
\end{myLemma}
\begin{myLemma}[Lemma 14 of~\citet{COLT'14:second-order-Hedge}]
    \label{lemma:adamlprod-self-confident}
    Let $a_0>0$ and $a_1, \dots, a_m \in [0, 1]$ be real numbers and let $f:(0, +\infty)\mapsto[0, +\infty)$ be a non-increasing function. Then
    \begin{equation*}
        \sum_{i=1}^m a_i f(a_0 + \cdots + a_{i-1}) \leq f(a_0) + \int^{a_0+a_1+\dots+a_m}_{a_0} f(u) ~\mathrm{d}u.
    \end{equation*}
\end{myLemma}
\begin{myLemma}[Lemma 4.8 of~\citet{UAI'19:FIRST-ORDER}]
    \label{lemma:UAI-first-order}
    Let $a_1, a_2,\dots, a_T$ be non-negative real numbers. Then
    \begin{equation*}
        \sum_{t=1}^{T} \frac{a_{t}}{\sqrt{1+\sum_{s=1}^{t-1} a_{s}}} \leq 4 \sqrt{1+\sum_{t=1}^{T} a_{t}}+\max_{t \in [T]} a_{t}.
    \end{equation*}
\end{myLemma}
\begin{myLemma}[Lemma 5 of~\citet{thesis:shai2007}]
    \label{lemma:shai2007}
    For any $x, y, a \in \mathbb{R}_+$ satisfying $x - y \leq \sqrt{ax}$, we have $x - y \leq \sqrt{ay} + a$.
\end{myLemma}
Based on Lemma~\ref{lemma:shai2007}, we can achieve the following variant.
\begin{myLemma}
    \label{lemma:substitute-F_T}
    For any $x, y, a, b \in \mathbb{R}_+$ satisfying $x - y \leq \sqrt{ax} + b$, $x - y \leq \sqrt{ay + ab} + a + b$.
\end{myLemma}

\begin{myLemma}[Lemma 13 of~\citet{COLT'14:second-order-Hedge}]
    \label{lemma:second-order-hedge-weight-relation}
    For all $x > 0$ and $\alpha \geq 1$, $x \leq x^\alpha + \frac{\alpha-1}{e}$.
\end{myLemma}

\end{document}